\documentclass{article}
\usepackage[utf8]{inputenc}
\usepackage{fullpage}
\usepackage[normalem]{ulem}
\usepackage{graphicx}
\usepackage{grffile}
\usepackage{xcolor}
\usepackage{subfig}
\usepackage{hyperref}

\usepackage{amsmath}
\usepackage{amssymb}
\usepackage{mathtools}
\usepackage{amsthm}

\usepackage{nicefrac} 
\usepackage{amsfonts} 
\usepackage{enumitem}

\usepackage[style=numeric,citestyle=numeric,maxbibnames=99,backend=biber,sorting=nyt,natbib=true,backref=true]{biblatex}
\usepackage[utf8]{inputenc} 
\usepackage[T1]{fontenc}    
\usepackage{hyperref}       
\usepackage{url}            
\usepackage{booktabs}       
\usepackage{amsfonts}       
\usepackage{nicefrac}       
\usepackage{microtype}      
\usepackage{xcolor}         
\usepackage{soul}
\usepackage{color}
\usepackage[ruled]{algorithm2e}
\usepackage{algorithmic}
\usepackage{amssymb}
\usepackage{amsthm}
\usepackage{wrapfig}
\usepackage{comment}
\usepackage{graphicx}
\usepackage{bbm}
\usepackage{url}
\usepackage{xr}
\usepackage{multirow}
\usepackage{listings}
\usepackage{float}
\usepackage{subfloat}
\usepackage{cancel}

\usepackage{subfiles}

\definecolor{mypink1}{rgb}{0.858, 0.188, 0.478}

\makeatletter
\newtheorem*{rep@theorem}{\rep@title}
\newcommand{\newreptheorem}[2]{%
\newenvironment{rep#1}[1]{%
 \def\rep@title{#2 \ref{##1}}%
 \begin{rep@theorem}}%
 {\end{rep@theorem}}}
\makeatother

\newtheorem{theorem}{Theorem}
\newtheorem{lemma}{Lemma}
\newtheorem{proposition}{Proposition} 
\newtheorem{remark}{Remark}
\newtheorem{corollary}{Corollary}
\newtheorem{definition}{Definition}

\newreptheorem{theorem}{Theorem}
\newreptheorem{corollary}{Corollary}
\newreptheorem{proposition}{Proposition}
\newtheorem{assumption}{Assumption}

\newcommand{\SA}{\boldsymbol{\Sigma}_{0:k}}
\newcommand{\SB}{\boldsymbol{\Sigma}_{k:\infty}}

\newcommand{\Th}{\hat{\boldsymbol{\theta}}}
\newcommand{\T}{\boldsymbol{\theta}}
\newcommand{\Tb}{\bar{\boldsymbol{\theta}}}

\newcommand{\TAS}{{\boldsymbol{\theta}}^*_{0:k}}
\newcommand{\TBS}{{\boldsymbol{\theta}}^*_{k:\infty}}

\newcommand{\A}{\mathbf{A}}

\newcommand{\X}{\mathbf{X}}
\newcommand{\I}{\mathbf{I}}

\newcommand{\Cb}{\bar{\mathbf{C}}}
\newcommand{\D}{\mathbf{D}}
\newcommand{\Q}{\mathbf{Q}}
\newcommand{\U}{\mathbf{U}}
\newcommand{\V}{\mathbf{V}}
\newcommand{\x}{\mathbf{x}}
\newcommand{\y}{\mathbf{y}}
\newcommand{\z}{\mathbf{z}}
\newcommand{\bb}{\mathbf{b}}
\newcommand{\Z}{\mathbf{Z}}
\newcommand{\HH}{\mathbf{H}}
\newcommand{\M}{\mathbf{M}}
\newcommand{\SIG}{{\boldsymbol{\Sigma}}}

\newcommand{\BT}{\bar{\T}_{\text{aug}}}
\newcommand{\HT}{\hat{\T}_{\text{aug}}}
\newcommand{\mg}{\mu_{\mathcal{G}}}
\newcommand{\cg}{\mathrm{{Cov}}_{\mathcal{G}}}

\def\E{\mathbb{E}}

\def\reals{\mathbb{R}}

\def\aug{{\text{aug}}}

\DeclareMathOperator{\sgn}{sgn}

\newcommand{\e}{\mathbf{e}}
\newcommand{\uu}{\mathbf{u}}
\newcommand{\vv}{\mathbf{v}}

\SetKwComment{Comment}{\color{green!50!black}// }{}

\SetKwProg{Function}{function}{}{}

\SetKwInOut{Input}{input}
\SetKwProg{Init}{init}{}{}

\newcommand{\BlackBox}{\rule{1.5ex}{1.5ex}}  
\ifdefined\proof
    \renewenvironment{proof}{\par\noindent{\bf Proof\ }}{\hfill\BlackBox\\[2mm]}
\else
    \newenvironment{proof}{\par\noindent{\bf Proof\ }}{\hfill\BlackBox\\[2mm]}
\fi

\title{The good, the bad and the ugly sides of data augmentation:\\ An implicit spectral regularization perspective}

\author{%
Chi-Heng Lin$^{1,4}$ \and Chiraag Kaushik$^{1}$ \and Eva L Dyer$^{*,1, 2}$ \and
Vidya Muthukumar$^{*, 1, 3}$
}

\date{}
\begin{document}

\maketitle

\begin{abstract}
\noindent Data augmentation (DA) is a powerful workhorse for bolstering performance in modern machine learning.
Specific augmentations like translations and scaling in computer vision are traditionally believed to improve generalization by generating new (artificial) data from the same distribution.
However, this traditional viewpoint does not explain the success of prevalent augmentations in modern machine learning (e.g. randomized masking, cutout, mixup), that greatly alter the training data distribution.
In this work, we develop a new theoretical framework to characterize the impact of a general class of DA on underparameterized and overparameterized linear model generalization.
Our framework reveals that DA induces \textit{implicit spectral regularization} through a combination of two distinct effects: a) manipulating the relative proportion of eigenvalues of the data covariance matrix in a training-data-dependent manner, and b) uniformly boosting the entire spectrum of the data covariance matrix through ridge regression. These effects, when applied to popular augmentations, give rise to a wide variety of phenomena, including discrepancies in generalization between over-parameterized and under-parameterized regimes and differences between regression and classification tasks.
Our framework highlights the nuanced and sometimes surprising impacts of DA on generalization, and serves as a testbed for novel augmentation design.
\end{abstract}

\begin{NoHyper}
\let\thefootnote\relax
\footnotetext{$^*$Both senior authors contributed equally.}
\footnotetext{$^1$School of Electrical and Computer Engineering, Georgia Institute of Technology, GA, USA}
\footnotetext{$^2$Department of Biomedical Engineering, Georgia Institute of Technology, GA, USA}
\footnotetext{$^3$School of Industrial and Systems Engineering, Georgia Institute of Technology, GA, USA}
\footnotetext{$^4$Samsung Research America}
\end{NoHyper}


\let\clearpage\relax
\section{Introduction}
Data augmentation (DA), or the transformation of data samples before or during learning, is a workhorse of both supervised~\citep{shorten2019survey, iosifidis2018dealing, liu2021adaptive} and self-supervised approaches~\citep{gidaris2018unsupervised, chen2020simple, grill2020bootstrap,azabou2021mine, zbontar2021barlow} for machine learning (ML).
It is critical to the success of modern ML in multiple domains, e.g., computer vision \citep{shorten2019survey}, natural language processing \citep{ feng2021survey}, time series data~\citep{wen2020time}, and neuroscience \citep{lashgari2020data,azabou2021mine,liu2021drop}. This is especially true in settings where data and/or labels are scarce or in other cases where algorithms are prone to overfitting~\citep{zhang2021understanding}. While DA is perhaps one of the most widely used tools for regularization, most augmentations are applied in an ad hoc manner, and it is often unclear exactly how, why, and when a DA strategy will work for a given dataset \citep{cubuk2019practical, ratner2017learning, balestriero2022effects}.

Recent theoretical studies have provided insights into the effect of DA on learning and generalization when augmented samples lie close to the original data distribution~\citep{dao2019kernel, chen2020group}.
However, state-of-the-art augmentations that are used in practice (e.g. data masking~\citep{he2022masked}, cutout \citep{devries2017improved}, mixup \citep{zhang2017mixup})  are stochastic and can significantly alter the distribution of the data \citep{gontijo2020affinity,he2022masked, yuan2021simple}.
Despite many efforts to explain the success of DA in the literature~\citep{bishop1995training,chapelle2001vicinal,chen2020group,dao2019kernel,wu2020generalization}, 
there is still a lack of a comprehensive platform to compare different types of augmentations at a quantitative level.

In this paper, we address this challenge by proposing a simple yet flexible theoretical framework that precisely characterizes the impact of DA on generalization.
Our framework enables generalization analysis for: 1. \textit{general stochastic augmentations}, 2. the classical \textit{ underparameterized regime} \citep{hastie2009elements} and the modern \textit{overparameterized regime}, 3. \textit{regression} and \textit{classification tasks}, and 4. \textit{strong} and \textit{weak distributional-shift augmentations}.
To do this, we borrow and build on finite-sample analysis techniques that simultaneously operate in the underparameterized and overparameterized regime for linear and kernel models~\citep{bartlett2020benign,tsigler2020benign,muthukumar2020class,muthukumar2020harmless}.

We find that DA induces two types of implicit, training-data-dependent regularization: manipulation of the spectrum (i.e. eigenvalues) of the data covariance matrix, and the addition of explicit $\ell_2$-type regularization to avoid noise overfitting.
The first effect of spectral manipulation can either make or break generalization by introducing helpful or harmful biases.
In contrast, the explicit $\ell_2$ regularization effect always improves generalization by preventing possibly harmful overfitting of noise.

Our theory reveals \emph{good, bad, and ugly} sides to DA depending on the setting, nature of task and type of augmentation. We find that on one hand, DA \emph{improves} generalization when it is designed in a targeted manner to reduce variance while preserving bias (for any setting/task) or if the reduction in variance outweighs increase in bias (for classification or underparameterized regression). On the other hand, DA is more \emph{unforgiving} for overparameterized regression; here, we find that popular augmentations frequently induce a large increase in both bias and distribution shift between training and test data. We also identify several \emph{ugly} (i.e.~subtle/nuanced) features to DA depending on whether the task is regression or classification, the model is underparameterized or overparameterized, and the augmentations are pre-computed or applied on-the-fly.



\subsection{Main contributions}

Below, we outline and provide a roadmap of the main contributions of this work.
\begin{itemize}
    \item We propose a new framework for studying non-asymptotic generalization with data augmentation for linear models by building on the recent literature on the theory of overparameterized learning~\citep{bartlett2020benign,tsigler2020benign,muthukumar2020class}. 
    We provide natural definitions of the augmentation mean and covariance operators that capture the impact of change in data distribution on model generalization in Section \ref{sec:erm}, and sharply characterize the ensuing performance for both regression and classification tasks in Sections \ref{regress-results} and \ref{class-results}, respectively.
    
    \item In Section \ref{case_study}, we apply our theory to provide new interpretations of a broad class of randomized DA strategies used in practice; e.g., random-masking~\citep{he2022masked}, cutout~\citep{devries2017improved}, noise injection~\citep{bishop1995training}, and group-invariant augmentations~\citep{chen2020group}.
    An example is as follows: while the classical noise injection augmentation~\citep{bishop1995training} causes only a constant shift in the spectrum, data masking~\citep{he2022masked, assran2022masked}, cutout~\citep{devries2017improved} and distribution-preserving augmentations~\citep{chen2020group} tend to \emph{isotropize} the equivalent data spectrum.
    We also use our framework as a testbed for new approaches by designing a new augmentation method, inspired by isometries in random feature rotation (Section~\ref{the good}). We show that this augmentation achieves smaller bias than the least-squared estimator and variance reduction on the order of the ridge estimator.


    \item In Section \ref{sec:exp} we empirically examine the influence of DA in conjunction with data and model family on generalization. We compare our closed-form expression with augmented stochastic gradient descent (SGD)~\citep{dao2019kernel,chen2020group, chen2020simple} and pre-computed augmentations~\citep{wu2020generalization, shen2022data}.
    In addition to verifying our theoretical insights, our experiments reveal phenomena of independent interest, including surprising distinctions between pre-computed DA and augmented SGD and varying robustness to augmentation hyperparameter tuning between regression and classification tasks.
    
    \item We conclude in Section~\ref{good_bad} with an extended discussion of the ``good, bad and ugly" ideas of DA. In Section~\ref{the good} we discuss how cleverly designed \emph{data-adaptive} covariance modification can reduce both bias and variance, and how a broader class of DA leads to variance reduction that outweighs bias increase for classification and \emph{underparameterized} regression tasks. In Section~\ref{the bad} we unpack the suboptimalities of the ``isotropizing" effect of DA, particularly in overparameterized regression where bias is especially harmful~\citep{muthukumar2020harmless,hastie2019surprises}. Finally, in Section \ref{the ugly}, we identify strikingly divergent impacts of DA depending on whether the task is regression or classification, the model is under or overparameterized, and the augmentation is pre-computed or applied to SGD.
    Our findings here corroborate the empirically observed benefits of DA being primarily applied ``on-the-fly", on moderate-dimensional data and classification tasks~\citep{yuan2021simple, dai2022boosting}.
\end{itemize}


\subsection{Notation}\label{notation}
We use $n$ to denote the number of training examples and $p$ to denote the data dimension.
Given a training data matrix $\X\in\reals^{n\times p}$ where each row (representing a training example) is independently and identically distributed (i.i.d.) and has covariance $\SIG := \mathbb{E}[\x\x^\top]$, we denote $\mathbf{P}^{\SIG}_{1:k-1}$ and $\mathbf{P}^{\SIG}_{k:\infty}$ as the projection matrices to the top $k-1$ and the bottom $p - k + 1$ eigen-subspaces of $\SIG$, respectively. For convenience, we denote the residual Gram matrix by $\mathcal{A}_k(\mathbf{\X};\lambda) = \lambda\I_n + \mathbf{\X}\boldsymbol{P}^{\SIG}_{k:\infty}\mathbf{\X}^T$, where $\lambda$ is some regularization constant.
Subscripts denote the subsets of column vectors when applied to a matrix; e.g. for a matrix $\V$ we have $\V_{a:b}:= [\mathbf{v}_{a},\mathbf{v}_{a+1},\dots,\mathbf{v}_b]$.
A similar definition applies to vectors; e.g. for a vector $\mathbf{x}$ we have $\mathbf{\x}_{a:b}=[\mathbf{x}_a,\mathbf{x}_{a+1},\dots,\mathbf{x}_b]$. 
The Mahalanobis norm of a vector is defined by $\|\x\|_{\mathbf{H}}=\sqrt{\x^\top\mathbf{H}\x}$. For a matrix $\A$, $\mathrm{diag}({\A})$ denotes the diagonal matrix with a diagonal equal to that of $\A$, $\mathrm{Tr}(\A)$ denotes its trace and $\mu_i(\A)$ its $i$-th largest eigenvalue. The symbols $\gtrsim$ and $\lesssim$ are used to denote inequality relations that hold up to universal constants which may depend only on $\sigma_x$ or $\sigma_\varepsilon$ and not on $n$ or $p$.
All asymptotic convergence results are stated in probability.

More specific notation corresponding to our signal model is given in Section~\ref{sec:gen_bound_reg}, and some additional notation that is convenient to define for our analysis is postponed to Section~\ref{deter_strat}.

\section{Related Work}
We organize our discussion of related work into two verticals: a) historical and recent perspectives on the role of data augmentation, and b) recent analyses of minimum-norm and ridge estimators in the over-parameterized regime.

\subsection{Data augmentation}

\paragraph{Classical links between DA and regularization:}
Early analysis of DA showed that adding random Gaussian noise to data points is equivalent to Tikhonov regularization~\citep{bishop1995training} and \emph{vicinal risk minimization}~\citep{zhang2017mixup, chapelle2001vicinal}; in the latter, a local distribution is defined in the neighborhood of each training sample, and new samples are drawn from these local distributions to be used during training.
These results established an early link between augmentation and explicit regularization. However, the impact of such approaches on generalization has been mostly studied in the underparameterized regime of ML, where the primary concern is reducing variance and avoiding overfitting of noise. Modern ML practices, by contrast, have achieved great empirical success in overparameterized settings and with a broader range of augmentation strategies~\citep{shorten2019survey, iosifidis2018dealing, liu2021adaptive}. 
The type of regularization that is induced by these more general augmentation strategies is not well understood.
Our work provides a systematic point of view to study this general connection without assuming any additional explicit regularization, or specific operating regime.

\vspace{-2mm}
\paragraph{In-distribution versus out-of-distribution augmentations:} 
Intuitively, if we could design an augmentation that would produce more virtual but identically distributed samples of our data, we would expect an improvement in generalization.
Based on this insight and the inherent structure of many augmentations used in vision (that have symmetries), another set of works explores the common intuition that data augmentation helps insert beneficial group-invariances into the learning process \citep{cohen2016group, raj2017local, mroueh2015learning, bruna2013invariant,yang2019invariance}. 
These studies generally consider cases in which the group structure is explicitly present in the model design via convolutional architectures \citep{cohen2016group, bruna2013invariant} or feature maps approximating group-invariant kernels \citep{raj2017local, mroueh2015learning}. 
The authors of \cite{chen2020group} propose a general group-theoretic framework for DA and explain that an averaging effect helps the model generalize through variance reduction.
However, they only consider augmentations that do not alter (or alter by minimal amounts) the original data distribution; consequently, they identify variance reduction as a sole positive effect of DA.
Moreover, their analysis applies primarily to underparameterized or explicitly regularized models\footnote{More recent studies of invariant kernel methods, trained to interpolation, suggest that invariance could either improve~\citep{mei2021learning} or worsen~\citep{donhauser2021rotational} generalization depending on the precise setting. Our results for the overparameterized linear model (in particular, Corollary~\ref{gv}) also support this message.}.

Recent empirical studies have highlighted the importance of diverse stochastic augmentations \citep{gontijo2020affinity}. They argue that in many cases, it is important to   introduce samples which are \emph{out-of-distribution} (OOD) \citep{sinha2021negative,peng2022open} (in the sense that they do not resemble the original data). In our framework, we allow for cases in which augmentation leads to significant changes in distribution and provide a path to analysis for such OOD augmentations that encompass empirically popular approaches for DA~\citep{he2022masked,devries2017improved}.
We also consider the modern overparameterized regime~\citep{belkin2019reconciling,dar2021farewell}.
We show that the effects of OOD augmentations go far beyond variance reduction, and the spectral manipulation effect introduces interesting biases that can either improve or worsen generalization for overparameterized models.

\vspace{-2mm}
\paragraph{Analysis of specific types of DA in linear and kernel methods:} 
\cite{dao2019kernel} propose a Markov process-based framework to model compositional DA and demonstrate an asymptotic connection between a Bayes-optimal classifier and a kernel classifier dependent on DA. Furthermore, they study the \emph{augmented empirical risk minimization} procedure and show that some types of DA, implemented in this way, induce approximate data-dependent regularization. 
However, unlike our work, they do not quantitatively study the generalization of these classifiers.
\cite{li2019enhanced} also propose a kernel classifier based on a notion of invariance to local translations, which produces competitive empirical performance.
In another recent analysis,~\cite{wu2020generalization} study the generalization of linear models with DA that constitutes \emph{linear transformations} on the data for regression in the overparameterized regime (but still considering additional explicit regularization). 
They find that data augmentation can enlarge the span of training data and induce regularization.
There are several key differences between their framework and ours.
First, they analyze deterministic DA, while we analyze stochastic augmentations used in practice~\citep{grill2020bootstrap,chen2020group}. 
Second, they assume that the augmentations would not change the labels generated by the ground-truth model, thereby only identifying beneficial scenarios for DA (while we identify scenarios that are both helpful and harmful).
Third, they study empirical risk minimization with pre-computed augmentations, in contrast to our study of augmentations applied \emph{on-the-fly} during the optimization process~\citep{dao2019kernel,chen2020group}, which are arguably more commonly used in practice. Our experiments in Section~\ref{sec:pre_post} identify sizably different impacts of these methods of application of DA even in simple linear models.
Finally, the role of DA in linear model optimization, rather than generalization, has also been recently studied; in particular,~\cite{hanin2021data} characterize  how DA affects the convergence rate of optimization.

\vspace{-2mm}
\paragraph{The impact of DA on nonlinear models:} Recent works aim to to understand the role of DA in nonlinear models such as neural networks. \cite{lejeune2019implicit} show that certain local augmentations induce regularization in deep networks via a ``rugosity'', or ``roughness'' complexity measure. While they show empirically that DA reduces rugosity, they leave open the question of whether this alone is an appropriate measure of a model's generalization capability. Very recently, \cite{shen2022data} showed that training a two-layer convolutional neural network with a specific permutation-style augmentation can have a novel \emph{feature manipulation} effect.
Assuming the recently posited ``multi-view" signal model~\citep{allen2020towards}, they show that this permutation-style DA enables the model to better learn the essential feature for a classification task. They also observe that the benefit becomes more pronounced for nonlinear models. 
Our work provides a similar message, as we also identify the DA-induced data manipulation effect as key to generalization.
However, we provide a comprehensive general-purpose framework for DA by which we can compare and contrast different augmentations that can either help or hurt generalization, while~\cite{shen2022data} only analyze a permutation-style augmentation.
We believe that combining our general-purpose framework for DA with a more complex nonlinear model is a promising future direction, and we discuss possible analysis paths for this in Section~\ref{sec:discussion}.

\subsection{Interpolation and regularization in overparameterized models}

\paragraph{Minimum-norm-interpolation analysis:}
Our technical approach leverages recent results in overparameterized linear regression, where models are allowed to interpolate the training data. Following the definition of~\cite{dar2021farewell}, we characterize such works by their explicit focus on models that achieve close to zero training loss and which have a high complexity relative to the number of training samples.
Specifically, many of these works provide finite sample analysis of the risk of the least squared estimator (LSE) and the ridge estimator~\citep{bartlett2020benign, tsigler2020benign, hastie2019surprises, belkin2020two,muthukumar2020harmless}. This line of research (most notably,~\cite{bartlett2020benign,tsigler2020benign}) finds that the 
mean squared error (MSE), comprising the bias and variance, can be characterized in terms of the effective ranks of the spectrum of the data distribution. The main insight is that, contrary to traditional wisdom, perfect interpolation of the data may not have a harmful effect on the generalization error in highly overparameterized models.
In the context of these advances, we identify the principal impact of DA as \emph{spectral manipulation} which directly modifies the effective ranks, thus either improving or worsening generalization.
We build in particular on the work of~\cite{tsigler2020benign}, who provide non-asymptotic characterizations of generalization error for general sub-Gaussian design, with some additional technical assumptions that also carry over to our framework\footnote{As remarked on at various points throughout the paper, we believe that the subsequent and very recent work of~\cite{mcrae2022harmless}, which weakens these assumptions further, can also be plugged with our analysis framework; we will explore this in the sequel.}.

Subsequently, this type of ``harmless interpolation'' was shown to occur for classification tasks~\citep{muthukumar2020class, cao2021risk, wang2021binary, chatterji2021finite,shamir2022implicit,deng2022model,montanari2019generalization}.
In particular,~\cite{muthukumar2020class,shamir2022implicit} showed that classification can be significantly easier than regression due to the relative benignness of the 0-1 test loss.
Our analysis also compares classification and regression and shows that the potentially harmful biases generated by DA are frequently nullified with the 0-1 metric.
As a result, we identify several beneficial scenarios for DA in classification tasks.
At a technical level, we generalize the analysis of~\cite{muthukumar2020class} to sub-Gaussian design. 
We also believe that our framework can be combined with the alternative mixture model (where covariates are generated from discrete labels~\citep{chatterji2021finite,wang2021binary,cao2021risk}), but we do not formally explore this path in this paper.

\vspace{-2mm}
\paragraph{Generalized $\ell_2$ regularizer analysis:}
Our framework extends the analyses of least squares and ridge regression to estimators with general Tikhonov regularization, i.e., a penalty of the form $\theta^\top \M \theta$ for arbitrary positive definite matrix $\M$. A closely related work is \cite{wu2020optimal}, which analyzes the regression generalization error of general Tikhonov regularization. 
However, our work differs from theirs in three key respects. 
First, the analysis of~\cite{wu2020optimal} is based on the proportional asymptotic limit (where the sample size $n$ and data dimension $p$ increase proportionally with a fixed ratio) and provides sharp asymptotic formulas for regression error that are exact, but not closed-form and not easily interpretable.
On the other hand, our framework is non-asymptotic, and we generally consider $p \gg n$ or $p \ll n$; our expressions are closed-form, match up to universal constants and are easily interpretable.
Second, our analysis allows for a more general class of \emph{random} regularizers that themselves depend on the training data; a key technical innovation involves showing that the additional effect of this randomness is, in fact, minimal.
Third, we do not explicitly consider the problem of determining an optimal regularizer; instead, we compare and contrast the generalization characteristics of various types of practical augmentations and discuss which characteristics lead to favorable performance. 

In addition to explicitly regularized estimators,~\cite{wu2020optimal} also analyze the ridgeless limit for these regularizers, which can be interpreted as the minimum-Mahalanobis-norm interpolator. 
In Section~\ref{sec:expts-aSGD} we show that such estimators can also be realized in the limit of minimal DA.

\vspace{-2mm}
\paragraph{The role of explicit regularization and hyperparameter tuning:}
Research on harmless interpolation and double descent~\citep{belkin2019reconciling} has challenged conventional thinking about regularization and overfitting for overparameterized models; in particular, good performance can be achieved with weak (or even negative) explicit regularization \citep{kobak2020optimal, tsigler2020benign}, and gradient descent trained to interpolation can sometimes beat ridge regression~\citep{richards2021comparing}.
These results show that the scale of the ridge regularization significantly affects model generalization; consequently, recent work strives to estimate the optimal scale of ridge regularization using cross-validation techniques~\citep{patil2021uniform,patil2022mitigating}.

As shown in classical work~\citep{bishop1995training}, ridge regularization is equivalent to augmentation with (isotropic) Gaussian noise, and the scale of regularization naturally maps to the variance of Gaussian noise augmentation.
Our work links DA to a much more flexible class of regularizers and shows that some types of DA induce an implicit regularization that yields much more robust performance across the hyperparameter(s) dictating the ``strength'' of the augmentation.
In particular, our experiments in Section~\ref{sec:expts-compareaugs} show that random mask~\citep{he2022masked}, cutout~\citep{devries2017improved} and our new random rotation augmentation yield comparable generalization error for a wide range of hyperparameters (masking probability, cutout width and rotation angle respectively); the random rotation is a new augmentation proposed in this work and frequently beats ridge regularization as well as interpolation.
Thus, our flexible framework enables the discovery of DA with appealing robustness properties not present in the more basic methodology of ridge regularization.

\vspace{-2mm}
\paragraph{Other types of indirect regularization:} 
We also mention peripherally related but important work on other types of indirect regularization involving \emph{creating fake ``knockoff'' features}~\citep{candes2018panning,romano2020deep} and \emph{dropout in parameter space}~\citep{cavazza2018dropout,mianjy2018implicit}. 
The knockoff methodology creates copies of \emph{features} (rather than augmenting data points) that are uncorrelated with the target to perform variable selection. Dropout also induces implicit regularization by randomly dropping out intermediate neurons (rather than covariates, as does the random mask~\citep{he2022masked} augmentation) during the learning process, and has been shown to have a close connection with sparsity regularization~\citep{mianjy2018implicit}. 
Overall, these constitute methods of indirect regularization that are applied to model parameters rather than data.
An intriguing question for future work is whether these effects can also be achieved through DA.


\section{Problem Setup}


In this section, we introduce the notation and setup for our analysis of generalization with data augmentation (DA). 
We review the fundamentals of empirical risk minimization (ERM) without DA and discuss how augmentations affect the ERM procedure.
Then, we derive a reduction to ridge regression that paves the way for our analysis in Section \ref{analysis-results}.

\subsection{Empirical risk minimization with data augmentation}\label{sec:erm}


Traditionally, high-dimensional ML models are commonly trained to minimize a combination of prediction error on training data and some measure of model complexity.
This is encapsulated in the \emph{regularized empirical risk minimization objective}, expressed for linear models $f_{\T}(\x) = \langle \x, \T \rangle$ as $\hat{\T} =\arg\min_{{{\T}}}  \left[\ell(\X\T, \y) + R({\T})\right]$
where $\ell$ is a loss function, $\X = \begin{bmatrix} \x_1 &  \ldots & \x_n \end{bmatrix}^\top \in \reals^{n \times p}$ is the training data matrix that stacks the $n$ covariates, $\y \in \reals^n$ is the vector of observations/responses, $\T \in \reals^p$ is the linear model parameter that we want to optimize, and $R(\T)$ is an explicit regularizer applied to the model. 
We will adopt the choice of \emph{squared loss function} $\ell(\X\T,\y)=\|\X\T-\y\|_2^2$ throughout this work owing to its mathematical tractability.

%

Modern machine learning relies heavily on data augmentation (DA) to achieve state-of-the-art performance~\citep{shorten2019survey,zhang2021understanding}.
Augmentations are typically applied \emph{on-the-fly} and stochastically to different examples during training \citep{chen2020group,grill2020bootstrap,chen2020simple}. This procedure, known as \emph{augmented stochastic gradient descent (aSGD)}, is widely used in practice.~\cite{chen2020group} showsed that aSGD
converges to the solution of the following \emph{augmented empirical risk minimization (aERM)} problem: 
\begin{align}\label{DAobj}
    \hat{\T}=\arg\min_{\T} \mathbb{E}_G\left[\|G(\X)\T -\y\|_2^2\right].
\end{align}
Above, $G$ denotes a stacked data augmentation function applied to each row of the matrix, i.e., $G(\X) = [g_1(\x_1)\dots,g_n(\x_n)]^T$; we assume that the transformations $g_i$ are stochastic and  
are drawn i.i.d. from an augmentation function distribution $\mathcal{G}$. 
For example, the classical \emph{Gaussian noise injection} augmentation~\citep{bishop1995training} is stochastic and takes the form $g(\x)=\x + \mathbf{n}$, where $\mathbf{n}$ is an isotropic Gaussian random variable. 


To characterize the impact of different augmentations on our solution, we begin by defining the first and second-order statistics of an augmentation distribution. 
We will show that these quantities appear in the closed-form characterization to the least-squares aERM problem for {\em any} augmentation function.
\begin{definition}[\textbf{Augmentation Mean and Covariance Operator}]\label{cov_def}
Consider a stochastic augmentation $\x\mapsto g(\x)$, where $g$ is drawn randomly from an augmentation distribution $\mathcal{G}$.
We then define the augmentation mean and the covariance for a given data point $\x$ as
\begin{align}
    &\mg(\x) := \E_{g\sim \mathcal{G}}[g(\x)],~~\cg(\x) := \E_{g\sim \mathcal{G}}\left[\left(g(\x)-\mu_{\mathcal{G}}(\x)\right)\left(g(\x)-\mu_{\mathcal{G}}(\x)\right)^\top\right],
\end{align}
where we use the subscript $\mathcal{G}$ to emphasize that the expectation is only over the randomness of the augmentation function $g$.
Similarly, we define the augmentation mean and covariance operators with respect to the training data set $\X = \begin{bmatrix} \x_1 &  \ldots & \x_n \end{bmatrix}^\top$ as:
   \begin{align*}
    &\mg(\X) := [\mg(\x_1), \mg(\x_2),\dots,\mg(\x_n)]^\top,~~\cg(\X) := \frac{1}{n}\sum_{i=1}^n\cg(\x_i).
\end{align*}
Finally, we call an augmentation distribution \textbf{unbiased on average}\footnote{Note that this definition of bias is completely different from the bias-variance decomposition that manifests in regression analysis, i.e.,~\eqref{decomppp}.} if $\mg(\x) = \x$.
\end{definition}

\subsection{Implications of a DA-induced regularizer and connections to ridge regression}\label{impli}

With this notation introduced, we now explain why DA gives rise to implicit regularization.
For simplicity, we consider augmentation distributions that are unbiased on average here\footnote{We handle biased augmentation distributions in Section~\ref{bias_regress}.}.
Then, we can simplify the objective~\eqref{DAobj} as:
\begin{align}\label{obj_1}
    \E_G[\|G(\X)\T-\y)\|_2^2] &= \E_G[\|\left(G(\X)-\mu(\X)\right)\T+\mu(\X)\T -\y)\|_2^2] \nonumber\\
    =~&\|\mu(\X)\T - \y\|_2^2 + \|\T\|_{n\cg(\X)}^2 \nonumber \\
    =~& \|\X\T - \y\|_2^2 + \|\T\|_{n\cg(\X)}^2,
\end{align}
where the last two steps used the assumption that the augmentation distribution is unbiased on average.
From this expression, it is clear that DA produces an \textit{implicit}, \textit{data-dependent} regularization $\|\T\|_{n\cg(\X)}^2$, defined by the augmentation covariance we just introduced. 

The heart of our analysis is a detailed investigation of the implications of this data-dependent regularization on generalization.
As a first step, we unpack the effects of the DA-induced regularizer $\|\T\|_{n\cg(\X)}^2$. Note that the objective~\eqref{obj_1} can be viewed as a general Tikhonov regularization problem with a possibly data-dependent regularizer matrix. Using this observation, we will show that this creates the effects of (i) $\ell_2$ \textit{regularization} (i.e. Tikhonov regularization with an identity regularizer matrix) and (ii) \textit{data spectrum modification}. 

The first step is to explicitly connect the solution to a ridge regression estimator.
Since our focus is on stochastic augmentations, we assume that $\cg(\X)\succ 0$. 
Then, the objective~\eqref{obj_1} admits a closed-form solution given by $\hat{\T}_{\text{aug}} = (\X^\top\X + n\cg(\X))^{-1}\X^\top \y$.
We use this solution to link the estimator $\hat{\T}_{\text{aug}}$ to a ridge estimator by derivation below. 
For ease of exposition, we suppress the dependency of $\cg$ on the training data matrix $\X$.
\begin{align}\label{simt}
    \hat{\T}_{\text{aug}}&= (\X^\top\X + n\cg(\X))^{-1}\X)^\top\y\nonumber\\
    &={\cg}^{-1/2}(n\I_p + {\cg}^{-1/2}\X^\top\X{\cg}^{-1/2})^{-1}
    {\cg}^{-1/2}\X^\top\y\nonumber\\
    &= {\cg}^{-1/2}(n\I_p + \widetilde{\mathbf{\X}}^\top\widetilde{\mathbf{\X}})^{-1}
    \widetilde{\mathbf{\X}}^\top\y ~~~~~(\text{where }\widetilde{\mathbf{\X}}:=\X{\cg}^{-1/2})\nonumber\\
    &={\cg}^{-1/2}\hat{\T}_{\text{ridge}}, ~~~~~\text{where}~ \hat{\T}_{\text{ridge}}:=(n\I_p + \widetilde{\X}^\top\widetilde{\mathbf{\X}})^{-1}\widetilde{\mathbf{\X}}^\top\y.
\end{align}
Recall that $\SIG:=\E_\x[\x\x^\top]$ denotes the original data covariance.
Then, it is easy to see that the MSE $\|\hat{\T}_{\text{aug}}-\T^*\|^2_{\SIG}$ is equivalent to $\|\hat{\T}_{\text{ridge}}-{\cg}^{1/2}\T^*\|^2_{{\cg}^{-1/2}\SIG{\cg}^{-1/2}}$.
Suppose, for a moment, that $\cg$ were fixed (or independent of $\X$). Then,~\eqref{simt} demonstrates an equivalence between the solution of aERM and a ridge estimator with data matrix $\widetilde{\X}$, data covariance $ {\cg}^{-1/2}\SIG{\cg}^{-1/2}$, ridge parameter\footnote{This demonstrates that \emph{negative} regularization, which is studied in some recent work~\citep{tsigler2020benign,kobak2020optimal} is not possible to achieve through the DA framework.} $\lambda=n$, and true model ${\cg}^{1/2}\T^*$ (in the sense that both solutions achieve the same MSE). 

Therefore, in terms of generalization, we can view DA as inducing a two-fold effect: 
\begin{itemize}
    \item[a)] $\ell_2$ regularization at a scale that is proportional to the number of training samples ($\lambda_{\text{reg}}=n$), 
    \item[b)]  a modification of the original data covariance from $\SIG$ to ${\cg}^{-1/2}\SIG{\cg}^{-1/2}$, which can make a sizable impact on the original spectrum.
\end{itemize}
It is important to note that this equivalence between solutions is only approximate since $\cg$ itself depends on $\X$.
We will justify and formalize this approximation in Section~\ref{deter_strat}.

\subsection{Application to different augmentations used in practice}
\label{app}
Understanding the closed-form expression of the aERM estimator above requires exactly characterizing the augmentation covariance operator $\cg(\X)$.
In Table~\ref{common-augs}, we list several common augmentations for which the augmentation covariance can be easily characterized and interpreted,  including: White and Correlated Gaussian noise, Unbiased Random Mask, Pepper noise, and Random Cutout. 
The derivations for these expressions are in Appendix~\ref{common_est}.

The expressions for $\cg(\X)$ in Table~\ref{common-augs} reveal that, in many cases,  the augmentation covariance is characterized by an interesting interplay between properties of the training data matrix $\X$ and parameters of the augmentation distribution. 
For example, in the case of the unbiased random mask augmentation, $\cg(\X)$ is a diagonal matrix whose entries depend on the covariance matrix of the training data $\X^\top \X$ and the masking probability $\beta$. The salt-and-pepper augmentation has a similar term appear in its augmentation covariance (corresponding to the ``salt" part of the augmentation), along with a data-independent term that has the same form as Gaussian noise injection (corresponding to the ``pepper" part of the augmentation). We note that, in general, any regularization of the form $\|\T\|_{A(\X)}^2$, where $A(\X)$ is some positive semi-definite matrix dependent on $\X$, can be achieved by a simple additive correlated Gaussian noise augmentation where $g(\X) = \X + \mathbf{N}$, $\mathbf{N}\sim \mathcal{N}(\mathbf{0},\boldsymbol{A(\X)})$.

\begin{table}[h]
\centering
\caption{{\em Examples of common augmentations} for which we can compute a closed-form solution to the aERM objective. Here, $\M$ is a circulant matrix defined in Appendix \ref{common_est}.}
\setlength{\tabcolsep}{3.5pt} 
\renewcommand{\arraystretch}{1.5}
\begin{tabular}[t]{|c|c|c|}
\hline
 &Augmentation function: $g(\x)$ & Covariance operator: $ \cg(\X)$ \\
\hline
Gaussian noise injection &$\x +\mathbf{n}$, $
\mathbf{n}\sim\mathcal{N}(0,\sigma^2\I)$& $\sigma^2\I$\\\hline
Correlated noise injection  &$\x +\mathbf{n}$, $
\mathbf{n}\sim\mathcal{N}(0,\mathbf{W})$&$ \mathbf{W}$
\\\hline
Unbiased random mask&$\mathbf{b}\odot \x$, $\mathbf{b}_i\sim \mathrm{Bernoulli}(1-\beta)$&$\frac{\beta}{1-\beta}\frac{1}{n}\mathrm{diag}(\X^\top\X)$\\\hline
Pepper noise injection &$\mathbf{b}\odot \x + (\mathbf{1}-\mathbf{b})\odot\mathcal{N}(0,\sigma^2)$&$\frac{\beta}{1-\beta}~\frac{1}{n}\rm{diag}\left(\X^\top\X\right)  + \frac{\beta\sigma^2}{(1-\beta)^2}\I$\\\hline
Random Cutout &zero-out $k$ consecutive features &$\frac{p}{p-k}\frac{1}{n}\M\odot\X^\top\X$\\
\hline
\end{tabular}
\label{common-augs}
\end{table}%

\section{Main Results}\label{analysis-results}
This section presents our meta theorems for the generalization performance of regression and classification tasks.
We consider estimators for augmentations which are unbiased-in-average and biased-in-average separately, as they exhibit significant differences in terms of generalization. 
The applications of the general theorem will be discussed in detail in Section \ref{good_bad}.
Table \ref{roadmap} provides the road map of our main results and their applications in this and the next sections.

\subsection{Preliminaries}\label{sec:gen_bound_reg}


Recall that $\X \in\mathbb{R}^{n\times p}$ denotes the training data matrix with $n$ i.i.d. rows comprising of the training data. Each data point $\x\in\mathbb{R}^p$ can be written as $\x = \SIG^{1/2}\z$, where we assume, without loss of generality, that $\SIG$ is a diagonal matrix with non-negative diagonal elements $\lambda_1 \geq \lambda_2, \dots \geq \lambda_p$, and $\z$ is a latent vector which is zero-mean, isotropic (i.e., $\E[\z]=0$, $\mathbb{E}\left[\z\z^T\right] = \mathbf{I}$), and sub-Gaussian with sub-Gaussian norm $\sigma_z$. 
(Note that the assumption of diagonal covariance $\SIG$ is without loss of generality because sub-Gaussianity is preserved under any unitary transformation; however, the covariance induced by DA will frequently not remain diagonal).

Our analysis applies across the classical underparameterized regime ($n \geq p$) and the modern overparameterized regime ($p > n$); however, much of our discussion of consequences of DA will be centered on the latter regime.
We assume the true data generating model to be $y = \x^T\T^* +\varepsilon$, where $\varepsilon$ denotes the noise, which is also isotropic and sub-Gaussian with sub-Gaussian norm $\sigma_{\varepsilon}$ and variance $\sigma^2$. We believe that our non-asymptotic framework can be extended to more general kernel settings as in the recent work of~\cite{mcrae2022harmless}, where features are not assumed to be sub-Gaussian, but we leave this extension to future work.

\subsubsection{Error Metrics}
In this work, we will focus on the squared loss training objective~\eqref{DAobj} for both regression and classification tasks. 
While we make this choice for relative mathematical tractability, we note that it is well-justified in practice as recent work~\citep{hui2020evaluation,muthukumar2020class, wang2021binary, chatterji2021finite} has shown that the squared loss can achieve competitive results when compared with the cross-entropy loss in classification tasks\footnote{We also believe that our analysis of the modified spectrum induced by DA suggests that such equivalences could also be shown for aSGD applied on the cross-entropy v.s. squared loss, but do not pursue this path in this paper.}.
For the regression task, we use the \emph{mean squared error (MSE)}, defined for an estimator $\hat{\T}$ as:
\begin{align*}
    \mathrm{MSE}(\Th) = \mathbb{E}_\x[(\x^T(\hat{\T}-\T^*))^2|\X, \varepsilon],
\end{align*}
Recall in the above that $\T^*$ denotes the true coefficient vector, $\varepsilon$ denotes noise in the observed data, and $\x$ denotes a test example that is independent of the training examples $\X$. 
For classification, we will use the \emph{probability of classification 0-1 error (POE)} as the testing metric:
\begin{align*}
\mathrm{POE}(\Th)=\mathbb{E}_\x[\mathbb{I}\{\sgn(\x^\top\Th)\neq \sgn(\x^\top\T^*)\}].
\end{align*}

\setlength{\tabcolsep}{5.5pt} 
\renewcommand{\arraystretch}{1.5}
\begin{table}[]
\centering
\caption{{\em Road map of main results.}\label{roadmap}}
\begin{tabular}{|c|cc|}
\hline                                                                                        & \multicolumn{1}{c|}{\bf Regression}                                                                                                                                                                                 & {\bf Classification}                                                                                                                                                                                                                    \\ 
\hline
\multirow{2}{*}{\begin{tabular}[c]{@{}c@{}}{\bf Meta-Theorem:} \\ Unbiased Estimator\end{tabular}}                         & \multicolumn{1}{c|}{\multirow{2}{*}{\textit{Theorem} \ref{gen_bound}}}                                                                                                                                                                                  & \multirow{2}{*}{\textit{Theorem} \ref{thm2}}                                                                                                                                                                                                          \\
                                                                                                                     & \multicolumn{1}{c|}{}                                                                                                                                                                                                                                            &                                                                                                                                                                                                                                                               \\ \hline
\multirow{2}{*}{\begin{tabular}[c]{@{}c@{}}{\bf Meta-Theorem:}\\ Biased Estimator\end{tabular}}                            & \multicolumn{1}{c|}{\multirow{2}{*}{\textit{Theorem} \ref{bias-thm1}}}                                                                                                                                                                                   & \multirow{2}{*}{\textit{Theorem} \ref{bias_cls}}                                                                                                                                                                                                     \\
                                                                                                                     & \multicolumn{1}{c|}{}                                                                                                                                                                                                                                            &                                                                                                                                                                                                                                                               \\ \hline
\multirow{2}{*}{\ Augmentation Case Studies}                                                                           & \multicolumn{1}{l|}{\multirow{2}{*}{\begin{tabular}[c]{@{}l@{}}Cutout: \textit{Cor. }\ref{cor_rm}, \ref{cor:cutout}, \ref{het_mask}, \ref{bias-rm} \\ Compositions: \textit{Cor. }\ref{salt_cor}\end{tabular}}} & \multicolumn{1}{l|}{\multirow{2}{*}{\begin{tabular}[c]{@{}l@{}}Cutout: \textit{Cor. } \ref{cor:cutout}, \ref{rm-class}, \ref{hrm-class}\\ Group Invariant: \textit{Cor. }\ref{gv}\end{tabular}}} \\
                                                                                                                     & \multicolumn{1}{l|}{}                                                                                                                                                                                                                                            & \multicolumn{1}{l|}{}                                                                                                                                                                                                                                         \\ \hline
{ Interplay with Signal Model }                                                                                         & \multicolumn{1}{c|}{\textit{Corollary} \ref{het_mask}}                                                                                                                                                                                                  & \textit{Corollary} \ref{hrm-class}                                                                                                                                                                                                                    \\ \hline
\multirow{2}{*}{\begin{tabular}[c]{@{}c@{}}{ Comparisons Between Under-}\\ \& { Over-parameterized Regimes}\end{tabular}} & \multicolumn{2}{c|}{\multirow{2}{*}{\textit{Corollary} \ref{cor_rm}, \ref{rm-class}, \ref{gv}}}                                                                                                                                                                                                                                                                                                                                                                                       \\
                                                                                                                     & \multicolumn{2}{c|}{}                                                                                                                                                                                                                                                                                                                                                                                                                                                                                                            \\ \hline
\multirow{2}{*}{\begin{tabular}[c]{@{}c@{}}{ Comparisons between} \\ { Regression} \& { Classification}\end{tabular}}        & \multicolumn{2}{c|}{\multirow{2}{*}{\textit{Proposition} \ref{gen_compare}, \ref{prop:non-uniformmask}}}                                                                                                                                                                                                                                                                                                                                                                                                      \\
                                                                                                                     & \multicolumn{2}{c|}{}                                                                                                                                                                                                                                                                                                                                                                                                                                                                                                            \\ \hline
\end{tabular}
\end{table}


\subsubsection{Spectral quantities of interest}
Recent works studying overparameterized regression and classification tasks \citep{bartlett2020benign,tsigler2020benign, muthukumar2020class, zou2021benign} have discovered that the \emph{spectrum}, i.e. eigenvalues, of the data covariance play a central role in characterizing the generalization error. 
In particular, two \emph{effective ranks}, which are functionals of the data spectrum and act as types of effective dimension, dictate the generalization error of both underparameterized and overparameterized models.
These are defined below.
\begin{definition}[\textbf{Effective Ranks,~\citep{bartlett2020benign}}]\label{eff_ranks}
    For any covariance matrix (spectrum) $\SIG$, ridge regularization scale given by $c$, and index $k \in \{0,\ldots, p - 1\}$, two notions of effective ranks are given as below: $$\rho_{k}({\SIG};c):=\frac{c+\sum_{i>k} \lambda_{i}}{n \lambda_{k+1}} ,~~R_k({\SIG};c):=\frac{(c + \sum_{i>k}\lambda_i)^2}{\sum_{i>k}\lambda_i^2}.$$
\end{definition}
Using this notation, the risk for the minimum-norm least squares estimate from \cite{bartlett2020benign,tsigler2020benign} can be sharply characterized as
\begin{align*}
    &\mathrm{MSE} \asymp \underbrace{\|\T^*-\E_{\varepsilon}[\Th|\X]\|^2_{\SIG}}_{\text{Bias}} + \underbrace{\|\Th-\E_{\varepsilon}[\Th|\X]\|^2_{\SIG}}_{\text{Variance}}, \text{ where } \\
    &{\rm Bias} \lesssim \|\TBS\|_{\SB}^2 +  \|\TAS\|_{\SA^{-1}}^2\lambda_{k+1}^2\rho_k(\SIG ; 0)^2,~~ {\rm Variance} \asymp \frac{k}{n} + \frac{n}{R_k(\SIG;0)},
\end{align*}
where $k \leq \min(n, p)$ is an index that partitions the spectrum of the data covariance $\SIG$ into ``spiked'' and residual components and can be chosen in the analysis to minimize the above upper bounds. 
We note that the expression for the bias is matched by a lower bound up to universal constant factors for certain types of signal: either random~\citep{tsigler2020benign} or sparse~\citep{muthukumar2020class}.

Intuitively, this characterization implies a two-fold requirement on the data spectrum for good generalization (in the sense of statistical consistency: $\mathrm{MSE} \to 0$ as $n \to \infty$): it must a) decay quickly enough to preserve ground-truth signal recovery (i.e. ensure that $\rho_k$ is small, resulting in low bias), but also b) retain a long enough tail to reduce the noise-overfitting effect (i.e. ensure that $R_k$ is large, resulting in low variance).


\subsection{A deterministic approximation strategy for DA analysis}\label{deter_strat}

Our main results show that the DA framework naturally inherits the above principle.
In other words, the impact of DA on generalization (in both underparameterized and overparameterized regimes) boils down to understanding the effective ranks of a \emph{modified, augmentation-induced spectrum}.
Our starting point is the approximate connection between the aERM estimator and ridge estimator that was established in Section~\ref{impli}.
Out of the box, this \emph{does not} establish a direct equivalence between the MSE of the two estimators.
This is because the implicit regularizer $\text{Cov}_G$ that is induced by DA intricately depends on the data matrix $\X$, which creates strong dependencies amongst the training examples in the equivalent ridge estimator.
A key technical contribution of our work is to show that, in essence, this dependency turns out to be quite weak for a large class of augmentations that are used in practice.
Our strategy is to approximate the aERM estimator $\Th_{\text{aug}}$ with an idealized estimator $\Tb_{\text{aug}}$ that uses the \emph{expected} augmentation covariance (over the original data distribution).
The two estimators are formally defined below:
\begin{equation}
\begin{aligned}\label{est_2}
    &\Th_{\text{aug}}= (\mg(\X)^\top\mg(\X) + n\cg(\X))^{-1}\mg(\X)^\top\y,\\
    &\Tb_{\text{aug}}= (\mg(\X)^\top\mg(\X) + n\mathbb{E}_{\x}[\cg(\x)])^{-1}\mg(\X)^\top\y,
\end{aligned}
\end{equation}
where $\x$ denotes a fresh data point.
This admits a decomposition of the MSE into three error terms, given by
\begin{align}\label{decomppp}
    \text{MSE}\lesssim 
    \underbrace{\|\T^*-\E_{\varepsilon}[\bar{\T}_{\text{aug}}|\X]\|^2_{\SIG}}_{\text{Bias}} + \underbrace{\|\bar{\T}_{\text{aug}}-\E_{\varepsilon}[\bar{\T}_{\text{aug}}|\X]\|^2_{\SIG}}_{\text{Variance}} + \underbrace{\|\Th_{\text{aug}}-\bar{\T}_{\text{aug}}\|^2_{\SIG}}_{\text{Approximation Error}}.
\end{align}
The bias and variance terms can be analyzed with relative ease through an extension of the techniques of~\cite{bartlett2020benign,tsigler2020benign} to general positive-semidefinite regularizers that are not dependent on the training data\footnote{For this case, a related contribution lies in the work of~\cite{wu2020optimal}. Note that~\cite{wu2020optimal} provided precise asymptotics for general regularizers in the proportional regime $p \propto n$ and focused on the question of the optimal Tikhonov regularizer, while our focus is on more interpretable non-asymptotic bounds for the general regularizers that are induced by popular augmentations. We believe that our framework could also yield identical proportional asymptotics for DA under an equivalent version of Assumption~\ref{asm:assump1} for the proportional regime $p \propto n$, but do not pursue this path in this paper.} $\X$, as we outlined in Section \ref{impli}.
We provide a novel analysis of the approximation error term in Section \ref{regress-results} and show, for an arbitrary data covariance $\SIG$ and several popular augmentations, that this approximation error is often dominated by either the bias or variance.
As described in more detail in Section~\ref{unbias_regress}, this domination implies that we can tightly characterize the MSE with upper and lower bounds that match up to constant factors for these augmentations.
Figure~\ref{bias_var} confirms that the approximation error is indeed negligible. In this plot, we show the decomposition corresponding to the terms in (\ref{decomppp}) for random mask augmentation with different masking probabilities denoted by $\beta$. We can see that the approximation error is small compared with  the other error components. 

\begin{figure}[htbp!]
    \centering
    \setcounter{subfigure}{0}
    {\includegraphics[width=1\textwidth]{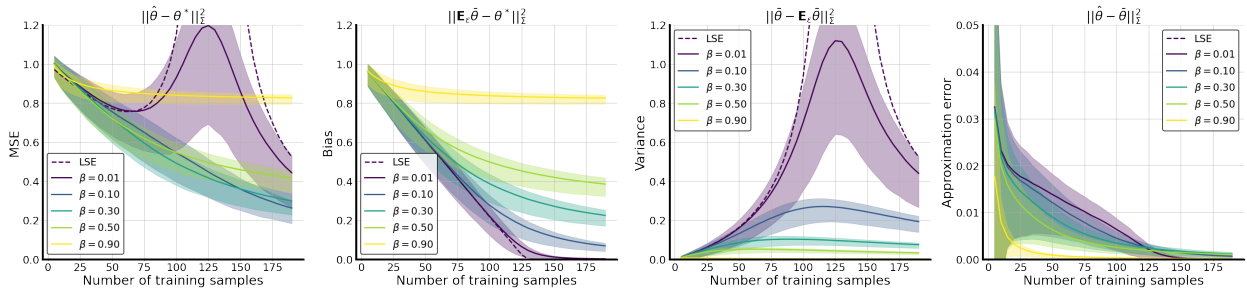}\vspace{-0mm}}
  \caption{\footnotesize{{\em Decomposition of MSE into the bias, variance, and approximation error as in Theorem \ref{gen_bound}.} A random masking augmentation is applied  with different dropout probability $\beta$ and the bias, variance, and approximation error are computed as a function of the number of training samples. The approximation error is small compared to the bias and variance and goes to zero quickly with more training data.}}\label{bias_var}
  \vspace{-3mm}
\end{figure}

That the approximation error is negligible is a surprising observation in the high-dimensional regime, as the sample data augmentation covariance $\cg(\X)$ and its expectation $\mathbb{E}_{\x}[\text{Cov}_g(\x)]$ are $p$-dimensional square matrices and $p \gg n$.
We critically use the special structure of the augmentations we study to show that despite this high-dimensional structure, it is common for $\cg(\X)$ to converge to its expectation at a rate that depends mostly on $n$ and minimally on $p$.

To show that our deterministic approximation is validated, i.e., the approximation error term is negligible, we require the following technical assumption, which shows that a normalized version of the empirical augmentation-induced covariance matrix converges as $n, p \to \infty$.

\begin{assumption}
 \label{asm:assump1} 
Let the data dimension $p$ grows with $n$ at a polynomial rate $p \asymp n^{\alpha}$ for some $\alpha > 0$.
Then, we assume that for any sequence of data covariance matrices $\{\SIG_p\}_{p \geq 1}$, the normalized empirical covariance induced by the augmentation distribution converges to its expectation as $n \to \infty$.
More formally, we assume that $\Delta_G\to 0 \text{ as } n \to \infty \text{ almost surely},$ where
    \begin{align}
    \label{eq:deltag}
        \Delta_G := \left\|\frac{1}{n}\mathbb{E}_{\x}[\cg(\x)]^{-\frac{1}{2}}\sum_{i=1}^n\cg(\x_i)\mathbb{E}_{\x}[\cg(\x)]^{-\frac{1}{2}} - \I_p\right\|.
    \end{align}
\end{assumption}
We note here that the above should be interpreted as the limit as both $n$ and $p$ grow together. For our subsequent results to be meaningful, it is further required that this convergence is sufficiently fast as $n, p \to \infty$.
We will show in Section \ref{sec:appx_error} that a wide class of popular augmentations will satisfy this assumption and converge at the rate $\mathcal{O}\left(\sqrt{\frac{\log n}{n}}\right)$. We will see that this rate is sufficient for our results to be tight in non-trivial regimes.




\begin{figure}[!t]
  \centering
\includegraphics[width=1\textwidth]{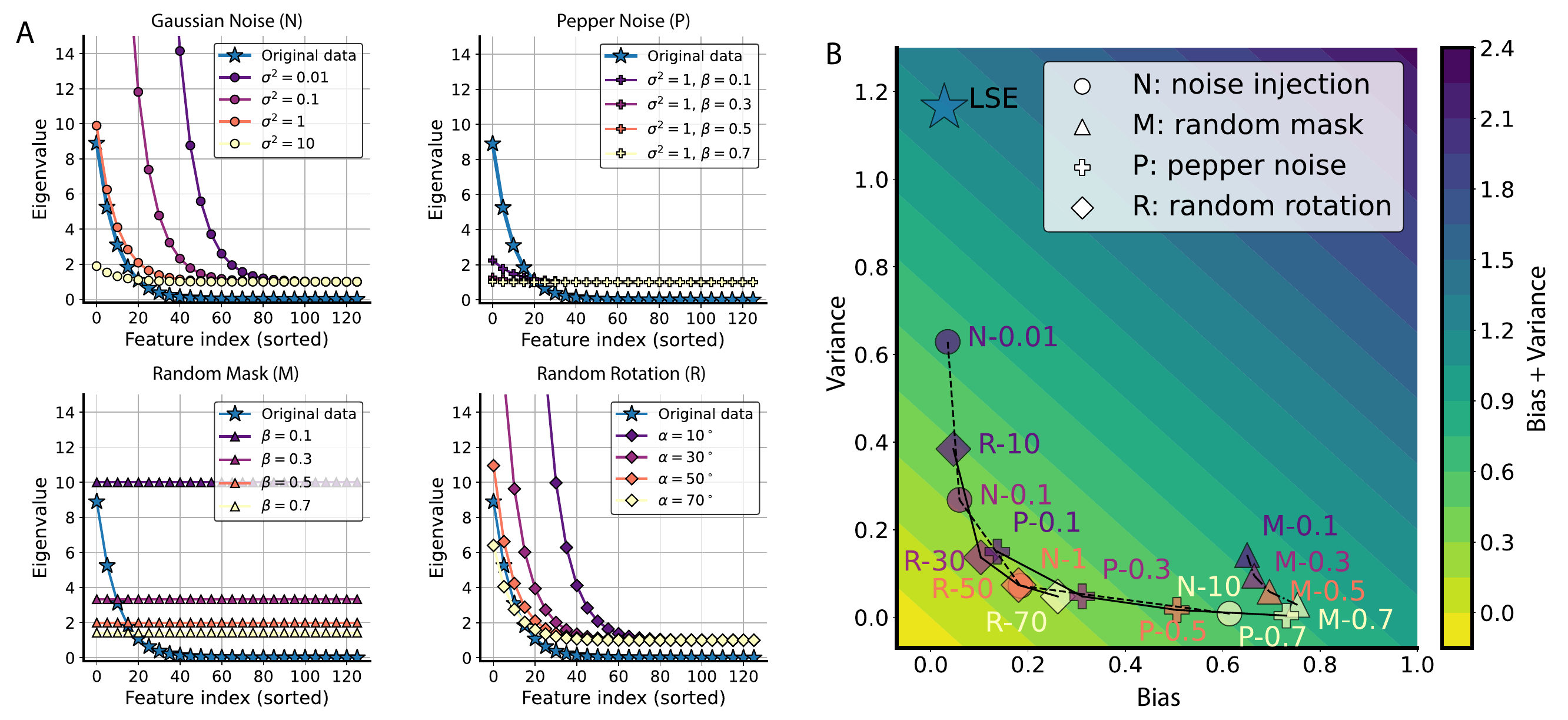}
\caption{\footnotesize{{\em Visualizing the augmented data spectrum and generalization for different forms of DA.} On the left in (A), we visualize the regularized augmented spectrum in Equation~\eqref{mod_spec_2}), clockwise for Gaussian noise, pepper noise, random mask, and our novel random rotation introduced in Section \ref{sec:rot}. On the right in (B), we show their corresponding generalization, where the number indicated for each data point denotes the strength of its augmentation  parameter. The LSE (star) represents the baseline of least-squared estimator without any augmentation.\label{spec_imp}}}
\end{figure}

\subsection{Regression analysis}\label{regress-results}

With the connection of DA to ridge regression established in Section~\ref{impli} and the deterministic approximation method established in Section~\ref{deter_strat}, we are ready to present our meta-theorem for the regression setting. The results for the augmented estimators which are unbiased-in-average are presented in Section~\ref{unbias_regress}, and biased-in-average augmented estimators are studied in Section~\ref{bias_regress}. The applications of the general theorem in this section will be discussed in detail in Section \ref{good_bad}.
\subsubsection{Regression analysis for general classes of unbiased augmentations}\label{unbias_regress}
In this section, we present the meta-theorem for estimators induced by unbiased-on-average augmentations (i.e., for which $\mu_{\mathcal{G}}(\x)=\x$) in 
Theorem \ref{gen_bound}. All proofs in this section can be found in Appendix \ref{main_proof_reg}.
To state the main result of this section, we introduce new notation for the relevant augmentation-transformed quantities.
\begin{definition}[\textbf{Augmentation-transformed quantities}]\label{def:aug_transformed_quantities}
    We define two spectral augmentation transformed quantities, the {covariance-of-the-mean-augmentation} $\bar{\SIG}$, and {augmentation-transformed data covariance} ${\SIG}_{aug}$, by
    \begin{align}
    &\bar{\SIG}:=\E_\x[(\mg(\x)-\E_{\x}[\mg(\x)])(\mg(\x)-\E_\x[\mg(\x)])^\top],\label{mod_spec}\\
    &{\SIG}_{aug}:=\mathbb{E}_{\x}[\cg(\x)]^{-1/2}\bar{\SIG}\mathbb{E}_{\x}[\cg(\x)]^{-1/2}, \label{mod_spec_2}
    \end{align}
    We also denote the eigenvalues of $\SIG_{aug}$ by $\lambda_1^{aug} \geq \lambda_2^{aug} \geq \dots \geq \lambda_p^{aug}$.
    Similarly, we define the 
    {augmentation-transformed data matrix} $\X_{aug}$, and {augmentation-transformed model parameter} ${\T}_{aug}^*$ as
    \begin{align*}
        \X_{aug} := \mu_G(\X)\mathbb{E}_{\x}[\cg(\x)]^{-1/2}, ~{\T}_{aug}^*:=\mathbb{E}_{\x}[\cg(\x)]^{1/2}\T^*.
    \end{align*}
    Note that since the rows of $\X_{aug}$ are still i.i.d.,  $\X_{aug}$ can be viewed as a modified data matrix with covariance $\SIG_{aug}$ and that $\bar{\SIG}=\SIG$ if the augmentation is unbiased in average.   
\end{definition}

Armed with this notation, we are ready to state our meta-theorem.
\begin{theorem}[\textbf{High probability bound for MSE with  unbiased DA}]\label{gen_bound}
    Consider an unbiased data augmentation $g$ and its corresponding estimator $\HT$, where 
    ${\Delta}_G$ is  defined in Eq.~\ref{eq:deltag} and
    $\kappa$ is the condition number of $\SIG_{\text{aug}}$. Assume for some integers $k_1$, $k_2$, the condition numbers for the matrices $\mathcal{A}_{k_1}(\mathbf{\X_{\text{aug}}};n)$, $\mathcal{A}_{k_2}(\mathbf{\X_{\text{aug}}};n)$ (defined in Section \ref{notation}) are bounded by $L_1$ and $L_2$ respectively with probability $1 - \delta'$, and that ${\Delta}_G\leq c'$ for some constant $c' < 1$.
    Then 
    , with probability $1- \delta'-4n^{-1}$, the test mean-squared error is bounded by
    \begin{align}\label{un_bound}
      \mathrm{MSE}~\lesssim &~  {{\color{blue}\mathrm{Bias}} + {\color{red}\mathrm{Variance}}+\color{olive}\mathrm{Approximation Error}},\\
      \frac{{\color{blue}\mathrm{Bias}}}{L_1^4} ~\lesssim &~
        {{\left(\left\|\mathbf{P}^{\SIG_{\text{aug}}}_{k_1+1:p}{\theta}_{\text{aug}}^{*}\right\|_{\SIG_{\text{aug}}}^{2}
      +\left\|\mathbf{P}^{\SIG_{\text{aug}}}_{1:k_1}{\theta}_{\text{aug}}^{*}\right\|_{\SIG_{\text{aug}}^{-1}}^{2}\frac{(\rho^{\text{aug}}_{k_1})^2}{{(\lambda^{\text{aug}}_{k_1+1})^{-2}}+(\lambda^{\text{aug}}_{1})^{-2}(\rho^{\text{aug}}_{k_1})^2}\right)}}, \nonumber\\
        \frac{{\color{red}\mathrm{Variance}}}{L_2^2}~\lesssim&~{{ \left(\frac{k_2}{n}+ \frac{n}{R^{\text{aug}}_k}\right)\log n}},~~{{\color{olive}\mathrm{Approx. Error}}}~\lesssim~\kappa^{\frac{1}{2}}{\Delta}_G \left(\|\T^*\|_{\SIG} +  \sqrt{\mathrm{Bias} + \mathrm{Variance}}\right)\nonumber.
    \end{align}
    Above, we defined $\rho^{\text{aug}}_{k}:=\rho_{k}(\SIG_{\text{aug}};n)$ and $R^{\text{aug}}_{k}:=R_{k}(\SIG_{\text{aug}};n)$ as shorthand.
    \end{theorem}

Theorem~\ref{gen_bound} illustrates the critical role that the spectrum of the augmentation-transformed data covariance $\SIG_{\text{aug}}$ plays in generalization. 
In particular, we find that, up to an approximation error term, the generalization error is characterized by the effective ranks $\rho_k^{aug}$ and $R_k^{aug}$(rather than the original effective ranks of the covariance, as in \cite{tsigler2020benign}). Intuitively, we expect an increase in the bias as $\rho^{aug}$ increases and variance reduction as $R^{aug}$ increases.

\paragraph{When is our bound in Theorem \ref{gen_bound} tight?}
A natural question is when and whether our bound in Theorem \ref{gen_bound} is tight. The tightness of the testing error for an estimator with a fixed regularizer is established (under some additional assumptions on the data distribution, such as sub-Gaussianity and constant condition number) in Theorem 5 of \cite{tsigler2020benign}. Hence, as long as the approximation error in our theorem is dominated by either the bias or variance, then our bound will also be tight. Roughly speaking this happens when the convergence of $n^{-1}\cg(\X)$ to $\mathbb{E}_{\x}[\cg(\x)]$  is sufficiently fast with respect to $n$, i.e.~$\Delta_G$ is sufficiently small.
This condition is formalized in the lemma below.
\begin{lemma}[\textbf{Condition on bias/variance dominating error approximation}]\label{cond_tight1}
Suppose the conditions of Theorem \ref{gen_bound} hold. If
\begin{align*}
   \kappa^{\frac{1}{2}} {\Delta}_G \overset{n}{\ll} \min\left(\mathrm{Bias + Variance}, \sqrt{\mathrm{Bias + Variance}}\right).
\end{align*}
Then there exists $c''>0$ such that,
\begin{align*}
    \frac{1}{c''}\leq\frac{\mathrm{Bias}(\HT) + \mathrm{Variance}(\HT)}{\mathrm{Bias}(\BT) + \mathrm{Variance}(\BT)}\leq c''.
\end{align*}
\end{lemma}
\begin{proof}
    The lemma follows from Theorem \ref{gen_bound} with the observation that 
    \begin{align*}
        \kappa^{\frac{1}{2}}{\Delta}_G\left(\|\T^*\|_\SIG+\sqrt{\mathrm{Bias}(\BT)}+\sqrt{\mathrm{Variance}(\BT)}\right) \overset{n}{\ll} \mathrm{Bias}(\BT) + \mathrm{Variance}(\BT).
    \end{align*}
\end{proof}
\vspace{-2mm}

\subsubsection{Regression analysis for general biased-on-average augmentations}\label{bias_regress}

All of our analysis thus far has assumed that the augmentation is \emph{unbiased on average}, i.e. that $\mu_{\mathcal{G}}(\x) = \x$.
We now derive and interpret the expression for the estimator that is induced by a general augmentation that can be biased.
We introduce the following additional definitions.
\begin{definition}\label{del_def}
    We define the \textbf{augmentation bias} and \textbf{bias covariance} induced by the augmentation $g$ as 
    \begin{align}
        \xi(\x) := \mu_g(\x) - \x,~~\operatorname{Cov}_{\xi}:=\mathbb{E}_\x\left[\xi(\x)\xi(\x)^\top\right].
    \end{align}
In a similar spirit to $\Delta_G$, we define 
$\Delta_{\xi} := \left\|\frac{1}{n}\sum_{i=1}^n(\mu_g(\x_i)-\x_i)(\mu_g(\x_i)-\x_i)^\top-\operatorname{Cov}_{\xi}\right\|$.
\end{definition}

Since $\xi(\x)$ is not zero for a biased augmentation, the closed-form expression for the aERM estimator $\HT$ becomes more complicated and we lose the exact equivalence to an ridge regression in~\eqref{simt}.
This is because biased DA induces a distribution-shift in the training data that does not appear in the test data.
Our next result for biased estimators, which is strictly more general than Theorem~\ref{gen_bound}, will show that this distribution-shift affects the test MSE through both \emph{covariate-shift} as well as \emph{label-shift}. To facilitate analysis, we impose the natural assumption that the mean augmentation $\mu(\x)$ remains sub-Gaussian.
\begin{assumption}\label{bias_assump}
    For the input data $\x$, the mean transformation $\mu(\x)$ admits the form $\mu(\x)=\bar{\SIG}^{1/2} \bar{\z}$, where $\bar{\SIG}$ is defined in Definition~\ref{def:aug_transformed_quantities} and $\bar{\z}$ is a centered and isotropic sub-Gaussian vector with sub-Gaussian norm $\sigma_{\bar{z}}$.  
\end{assumption}

We also recall the definition of the mean augmentation covariance  $\bar{\SIG}:=\E_\x[(\mg(\x)-\E_{\x}[\mg(\x)])(\mg(\x)-\E_\x[\mg(\x)])^\top]$.
Now we are ready to state our theorem for biased augmentations. The proof is deferred to Appendix~\ref{prof_thm2_sec}.
\begin{theorem}[\textbf{Bounds on the MSE for Biased Augmentations}]\label{bias-thm1}
Consider the estimator $\HT$ obtained by solving the aERM in (\ref{DAobj}).
Let $\mathrm{MSE}^o(\HT)$ denote the unbiased MSE bound in Eq. (\ref{un_bound}) of Theorem \ref{gen_bound}, and $\Delta_G$ defined in Eq.~\ref{eq:deltag}. 
Suppose the assumptions in Theorem \ref{gen_bound} hold for the mean augmentation $\mu(\x)$ and that ${\Delta}_G\leq c<1$. Then with probability $1- \delta'-4n^{-1}$ we have,
 \begin{align*}
    &~\mathrm{MSE}(\HT)\lesssim R_1^2 \cdot
    \left(\sqrt{\mathrm{MSE}^o(\HT)} + R_2 \right)^2,
\end{align*}
where 
\begin{align*}
    R_1 &= 1 + \|\SIG^{\frac{1}{2}}\bar{\SIG}^{-\frac{1}{2}}-\I_p\| \text{ and } \\
    R_2 &= \sqrt{{\|\bar{\SIG}(\mathbb{E}_\x[\cg(\x)])^{-1}\|}}\left(1+\frac{{\Delta}_G}{1-c}\right)\left(\sqrt{{\Delta}_\xi}\|\T^*\|+\|\T^*\|_{\operatorname{Cov}_{\xi}}\right)\\
    &\times \left(\sqrt{\frac{1}{\lambda_k^{\text{aug}}}}+\sqrt{\frac{\lambda_{k+1}^{\text{aug}}(1+\rho_k^{aug})}{(\lambda_1^{\text{aug}}\rho_0^{aug})^2}}\right).
\end{align*}
\end{theorem}

Our upper bound for the MSE in the biased augmentation case is a generalization of the bound in \cite{tsigler2020benign} to the scenario with distribution-shift. This result shows that two different factors can cause generalization error over and above the unbiased case: 1. \emph{covariate shift}, which is reflected in the multiplicative factor $R_1$; this term occurs because we are testing the estimator on a distribution with covariance $\SIG$ but our training covariates have covariance $\bar{\SIG}$ instead, 2. \emph{label shift}, which manifests itself as the additive error given by $R_2$. This term arises from the training mismatch between the true covariate observation and mean augmented covariate (i.e., $\X$ v.s. $\mu_G(\X)$). As a sanity check, we can see that $R_1=1$ and $R_2=0$ when the augmentation is unbiased-on-average, i.e., $\mg(\x)=\x,~\forall \x$, since $\SIG=\bar{\SIG}$, $\Delta_{\xi}=0$ and $\operatorname{Cov}_\xi=0$.
Thus, we directly recover Theorem~\ref{gen_bound} in this case.
Whether Theorem~\ref{bias-thm1} is tight in general is an interesting open question for future work.





\subsection{Classification analysis}\label{class-results}
In this subsection, we state the meta-theorem for generalization of DA in the classification task. We follow a similar path for the analysis as in regression by appealing to the connection between DA and ridge estimators and the deterministic approximation strategy outlined above. While the results in this section operate under stronger assumptions, we provide a similar set of results to the regression case.
The primary aim of these results is to compare the generalization behavior of DA between regression and classification settings, which we do in depth in Section~\ref{good_bad}.

\subsubsection{Classification analysis setup}
We adopt the random signed model from \cite{muthukumar2020class}, noting that we expect similar analysis to be possible for the Gaussian-mixture-model setting of~\cite{chatterji2021finite,wang2021binary} (we defer such analysis to a companion paper).
Given a target vector $\theta^* \in \reals^d$ and a label noise parameter $0 \leq \nu^* <1/2$, we assume the data are generated as binary labels $y_i \in \{-1,1\}$ according to the signal model
\begin{equation}
y_i=\begin{cases}
          \sgn(\x_i^\top \T^*) \quad &\text{with probability} \, 1-\nu^* \\
          -\sgn(\x_i^\top \T^*) \quad &\text{with probability} \, \nu^* \\
     \end{cases}
\end{equation}
Just as in~\cite{muthukumar2020class}, we make a \textit{1-sparse} assumption on the true signal $\T^* = \frac{1}{\sqrt{\lambda_t}} \e_t$. 
We denote $\x_{\text{sig}}:=\x_t$ to emphasize the signal feature.
Motivated by recent results which demonstrate the effectiveness of training with the squared loss for classification tasks~\citep{hui2020evaluation,muthukumar2020class}, we study the classification risk of the estimator $\Th$ which is computed by solving the aERM objective on the binary labels $y_i$ with respect to the squared loss (Eq.~\eqref{DAobj}).

\cite{muthukumar2020class} showed that two quantities, \textit{survival} and \textit{contamination}, play key roles in characterizing the risk, akin to the bias and variance in the regression task (in fact, as shown in the proof of Lemma \ref{cn-bound}, the contamination term scales identically to the variance from regression analysis). The definitions of these quantities are given below.
\begin{definition}[\textbf{Survival and contamination~\citep{muthukumar2020class}}]\label{su_cn_def}
Given an estimator $\hat{\theta}$, its survival (SU) and contamination (CN) are defined as 
\begin{align*}
    \mathrm{SU(\Th)} = \sqrt{\lambda_t}\Th_t,~~\mathrm{CN(\Th)} =  \sqrt{\sum_{j=1, j\neq t}^p \lambda_j \Th_j^2}.
\end{align*}
\end{definition}
\noindent For Gaussian data,~\cite{muthukumar2020class} derived the following closed-form expression for the 
Probability-of-Error
~(POE):
\begin{equation}\label{class-risk}
    \text{POE}(\Th) = \frac{1}{2} - \frac{1}{\pi}\tan^{-1}\frac{\mathrm{SU(\Th)}}{\mathrm{CN(\Th)}}.
\end{equation}
Thus, the POE depends on the ratio between survival SU and contamination CN, essentially a kind of \emph{signal-to-noise ratio} for the classification task.
In this work, we prove that a similar principle arises when we consider training with data augmentation in more general correlated input distributions.
Formally, we make the following assumption on the true signal and input distribution for our classification analysis.

\begin{assumption}\label{cls_assump}
    Assume the target signal is 1-sparse and given by $\T^* = \frac{1}{\sqrt{\lambda_t}} \e_t$. Additionally, assume the input can be factored as $\x=\SIG^{\frac{1}{2}} \z$, where $\SIG \succeq 0$ is diagonal, and $\z$ is a sub-Gaussian random vector with norm $\sigma_z$ and uniformly bounded density. 
    We denote $\x_{\text{sig}}=\x_t$ and $\x_{\text{noise}}=[\x_1,\dots,\x_{t-1},\x_{t+1},\dots,\x_p]^T$. 
    We further assume that the signal and noise features are independent and are augmented independently\footnote{As mentioned earlier, we expect that our framework can be extended beyond sub-Gaussian features to more general kernel settings. Under the slightly different label model used in \cite{mcrae2022harmless}, we believe that the independence between signal and noise features can also be relaxed.}, i.e., $\x_{\text{sig}}\perp  \x_{\text{noise}}$.
\end{assumption}


Similar to the regression case, our classification analysis consists of 1) expressing the excess risk in terms of $\BT$, the estimator corresponding to the averaged augmented covariance $\mathbb{E}_{\x}[\text{Cov}_g(\x)]$, 2) arguing that the survival and contamination can be viewed as the equivalent quantities for a ridge estimator with a modified data spectrum, and 3) upper and lower bounding the survival and contamination of this ridge estimator.
As in the case of regression analysis, step 1) is the most technically involved.

\subsubsection{Classification analysis for unbiased-on-average augmentations}
Now, we present our main theorem for the classification task under the setting in Assumption \ref{cls_assump}. The proof of this theorem is deferred to Appendix~\ref{main_proof_cls}.

\begin{theorem}[\textbf{Bounds on Probability of Classification Error}]\label{thm2} 
    Let $t \leq n$ be the index (arranged according to the eigenvalues of ${\SIG}_{\text{aug}}$) of the non-zero coordinate of $\T^*$, $\widetilde{\SIG}_{\text{aug}}$ be the leave-one-out modified spectrum corresponding to index $t$, and $\widetilde{\X}_{\text{aug}}$ be the leave-one-column-out data matrix corresponding to column $t$. Suppose there exists a $t \leq k \leq n$ such that with probability at least $1-\delta$, the condition numbers of $n \I + \widetilde{\X}^{\text{aug}}_{k+1:p}(\widetilde{\X}^{\text{aug}}_{k+1:p})^\top$, 
    $n \I + {\X}^{\text{aug}}_{k+1:p}(\X^{\text{aug}}_{k+1:p})^\top$, and $\widetilde{\X}_{k+1:p}\SIG_{k+1:p}\widetilde{\X}_{k+1:p}^T$
    are at most $L$. Then as long as $\Vert \BT - \HT\Vert_{\SIG} = O(\mathrm{SU})$ and $\Vert \BT - \HT\Vert_{\SIG} = O(\mathrm{CN})$,  
    \begin{align}\label{thm2_poe}
        \mathrm{POE}(\hat{\theta})~\lesssim~ &  \frac{{\color{red}\mathrm{CN}}}{\mathrm{{\color{blue}SU}}}\left(1 + \sigma_z \sqrt{\log{\frac{{\color{blue}\mathrm{SU}}}{\mathrm{{\color{red}CN}}}}}\right),
    \end{align}
    with probability at least $1-\delta-\exp(-\sqrt{n})-5n^{-1}$, where
    \begin{align*}
        \frac{\lambda^{\text{aug}}_t(1-2\nu^*)\left(1-\frac{k}{n}\right)}{L\left(\lambda^{\text{aug}}_{k+1}\rho_k(\SIG_{\text{aug}};n) + \lambda_t^{\text{aug}}L\right)}~\lesssim~\underbrace{{\color{blue}\mathrm{SU}}}_{\text{Survival}}~&\lesssim~\frac{L\lambda^{\text{aug}}_t(1-2\nu^*)}{\lambda^{\text{aug}}_{k+1}\rho_k(\SIG_{\text{aug}};n) + L^{-1}\lambda^{\text{aug}}_t\left(1-\frac{k}{n}\right)},\\
        \sqrt{\frac{\tilde{\lambda}_{k+1}^{aug}\rho_{k}(\tilde{\SIG}_{aug}^2;0)}{L^2(\lambda_1^{aug})^2(1+\rho_0(\SIG_{aug};\lambda))^2}}~\lesssim
        \underbrace{{\color{red}\mathrm{CN}}}_{\text{Contamination}} &\lesssim~ \sqrt{(1+\mathrm{SU}^2)L^2\left(\frac{k}{n} + \frac{n}{R_k(\tilde{\SIG}_{\text{aug}};n)}\right)\log n}
    \end{align*}
    Furthermore, if $\x$ is Gaussian, then we obtain even tighter bounds:
    \begin{align*}
        \frac{1}{2} - \frac{1}{\pi}\tan^{-1}c\frac{{\color{blue}\mathrm{SU}}}{ {\color{red}\mathrm{CN}}} ~\leq~ \mathrm{POE}(\HT) ~\leq ~\frac{1}{2} - \frac{1}{\pi}\tan^{-1}\frac{1}{c}\frac{{\color{blue}\mathrm{SU}}}{ {\color{red}\mathrm{CN}}}~\lesssim~ {\frac{{\color{red}\mathrm{CN}}}{{\color{blue} \mathrm{SU}}}},
    \end{align*}
where $c$ is a universal constant.

\end{theorem}

\begin{remark}
Based on the expression for the classification error for Gaussian data, we see that the survival needs to be asymptotically greater than the contamination for the POE to approach $0$ in the limit as $n,p \to \infty$. We note that the general upper bound we provide matches the tight upper and lower bounds for the Gaussian case up a log factor. Furthermore, the condition $\Vert \BT - \HT\Vert_{\SIG} = O(\mathrm{SU})$ and $\Vert \BT - \HT\Vert_{\SIG} = O(\mathrm{CN})$ is related to our condition for the tightness of our regression analysis, but a bit stronger (because our regression analysis only requires one of these relations to be true). We characterize when this stronger condition is met in Lemma \ref{err-class}.
\end{remark}

Based on the upper and lower bounds provided for SU and CN, we see that these quantities depend crucially on the effective ranks of the induced covariance matrix $\SIG_{aug}$. 
In particular, we note that SU is large when $\rho_k^{aug}$ is small and CN is small when $R_k^{aug}$ is large; good generalization relies on having a careful balance of these two factors to ensure that the ratio of CN to SU is small.
For favorable classification performance, Theorem~\ref{thm2} also requires $t \leq n$. This is a necessary product of our analogy to a ridge estimator and is equivalent to requiring that $\theta_{\text{aug}}^*$ lies within the eigenspace corresponding to the dominant eigenvalues of the spectrum $\SIG_{aug}$.
Such requirements have also been used in past analyses of both regression~\citep{tsigler2020benign} and classification~\citep{muthukumar2020class}.



\subsubsection{Classification analysis for general biased-on-average augmentations} As a counterpart of our regression analysis for estimators induced by biased-on-average augmentations (i.e. $\mu_g(\x) \neq \x$), we would also like to understand the impact of augmentation-induced bias on classification.
Interestingly, the effect of this bias in classification turns out to be much more benign than that in regression. 
As a simple example, consider a scaling augmentation of the type $g(\x):=2\x$. The induced bias is $\mu_g(\x) - \x = \x$, and the trained estimator $\HT$ is just half the estimator trained with $\x$, which, however, predicts the same labels in a classification task.
Therefore, we conclude that even with a large bias, the resultant estimator might be equivalent to the original one for classification tasks. 
In fact, as we show in the next result, augmentation bias is benign for the classification error metric under relatively mild conditions. The proof of this result is provided in Appendix \ref{pf-thm4}.

\begin{theorem}[\textbf{POE of biased estimators}]\label{bias_cls}
    Consider the 1-sparse model $\T^*=\mathbf{e}_t$.
    and let $\HT$ be the estimator that solves the aERM in (\ref{DAobj}) with biased augmentation (i.e., $\mu(\x)\neq \x$).
    Let Assumption \ref{bias_assump} holds, and the assumptions of Theorem \ref{thm2} be satisfied for data matrix $\mu(\X)$. If the mean augmentation $\mu(\x)$ modifies the $t$-th feature independently 
    of other features
    and
    the sign of the $t$-th feature is preserved under the mean augmentation transformation, i.e.,
    $\sgn\left(\mu(\x)_t\right)=\sgn\left(\x_t\right),$ $\forall \x$, then, the POE($\HT$) 
    is upper bounded by
    \begin{align*}
        \mathrm{POE}(\HT)\leq \mathrm{POE}^o(\HT),
    \end{align*}
    where $\mathrm{POE}^o(\HT)$ is any bound in Theorem \ref{thm2} with $\X$ and $\SIG$ replaced by $\mu(\X)$ and $\bar{\SIG}$, respectively.
\end{theorem}
At a high level, this result tells us that as long as the signal feature preserves the sign under the mean augmentation, the classification error is purely determined by the modified spectrum induced by DA.
Note that the sign preservation is only required in expectation and not for every realization of the augmentation, i.e., we only require $\E_g\left[g(\x)_t\right]$ has the same sign as $\x_t$, rather than requiring that $g(\x)_t$ have the same sign as $\x_t$ for every realization of $g$. The latter label-preserving property is is much more stringent and has been studied in \cite{wu2020generalization}.


\subsection{Classes of augmentations our theory can be applied to}\label{sec:appx_error}

In this section, we delineate important classes of augmentations for which our theory provides a sharp characterization of their impact on generalization. 
More formally, we show under these classes of augmentations that the approximation error term of Theorem~\ref{gen_bound} is negligible with respect to the bias/variance terms and our analysis is tight, i.e.~Lemma~\ref{cond_tight1} holds and Theorem~\ref{gen_bound} is tight up to constant/logarithmic factors.
Recall that Lemma~\ref{cond_tight1} requires the (normalized) error of the ``sample" augmentation covariance, denoted by $\Delta_G$, to be sufficiently small with respect to the sum of the bias and variance terms.
This section shows that the value of $\Delta_G$ inherently depends on the extent of correlation between the augmentations across features (where the correlation is defined for fixed data, and only with respect to the stochasticity in the augmentations).
At a high level, we show that: a) $\Delta_G$ is negligible as long as the correlations between the feature augmentations are weak enough, and that b) this sufficiently weak level of correlation is indeed the case for several popular classes of augmentations. 

We first analyze the simplest case of \emph{uncorrelated feature augmentations}, and then generalize our analysis to \emph{regionally correlated feature augmentations} and augmentations with \textit{``small'' off-diagonal component}.

\paragraph{Uncorrelated-feature augmentations:}
Many augmentations involve independently augmenting each of the features (or, more generally, applying an augmentation which is uncorrelated across features). This class subsumes many prevailing augmentations like random mask and salt-and-pepper noise. 
Because the augmentation covariance $\text{Cov}_{\mathcal{G}}(\x)$ is diagonal for such augmentations, we can show that $\Delta_G$ is small. 


\begin{proposition}[\textbf{Uncorrelated Feature Augmentations}]\label{ind_cor}
    Let the augmentation $g$ be composed of $p$ uncorrelated feature augmentation maps, i.e.,~$g(\x) = \begin{bmatrix} g_1(x_1) & \dots & g_d(x_d) \end{bmatrix}$ where $\{g_i(\cdot)\}_{i \in [p]}$ are uncorrelated (with respect to the randomness in the augmentation).
If the variance of each feature augmentation $\operatorname{Var}_{g_i}(g_i(x_i))$ (which is a random variable due to the randomness in $x_i$) is sub-exponential with sub-exponential norm $\sigma_i^2$ and mean $\bar{\sigma}_i^2$ for all $i \in [p]$, then we have
\begin{align*}
    {\Delta}_{G}\lesssim  \max_i\left(\frac{\sigma_i^2}{\bar{\sigma}_i^2}\right)\sqrt{\frac{\log n}{n}}.
\end{align*}
with probability at least $1 - \frac{1}{n}$.
\end{proposition}
Proposition~\ref{ind_cor} is proved in Appendix~\ref{prof_prop1_reg} and gives a bound on $\Delta_G$ of the order  $O(\sqrt{\frac{\log n}{n}})$. However, one might wonder whether the approximation error still vanishes for stochastic augmentations that include dependencies between features, i.e.~the random variables $\{g_i(x_i)\}_{i=1}^p$ are not necessarily uncorrelated for a fixed value of $\x$.
To address this question, the following subsection describes two techniques to bound $\Delta_G$ for two important types of such ``feature-dependent" augmentations.

\paragraph{Regionally correlated feature augmentations:}
First, we consider a popular class of augmentations that are correlated to a limited extent across features. This encompasses many ``patch''-based augmentations used in image applications, such as the PatchShuffle augmentation of \cite{kang2017patchshuffle}\footnote{We derive $\cg$ for this augmentation in Appendix \ref{common_est}.}. To define this type of augmentation, we categorize the features into $k$ groups denoted by $B_1,\ldots,B_k \subset [p]$.
In image applications, each group $B_j$ can represent a local region of an image.
We assume that only the augmentations within a group can be correlated, meaning that~$\E[g_{j_1}(x_{j_1}) g_{j_2}(x_{j_2})] \neq 0$ only if $j_1,j_2 \in B_j$ for some $j \in [k]$, i.e.~if the features belong to the same group.
We then overload notation and write the augmentation in block form as
$g(\x) = [g_1(\x_1),g_2(\x_2),\dots,g_k(\x_k)],$ where $g_j \sim \mathcal{G}_j$. In this notation, each \emph{sub-feature} $\x_j$ has smaller dimensionality $\x_j\in\mathbb{R}^{|B_j|}$ (and we have $\sum_{j \in [k]} |B_j| = p$ by definition) and the augmentation covariance for the $j$th block is denoted $\mathrm{Cov}_{\mathcal{G}_j}(\x_j)$.
Note that our assumption on correlations being only within the blocks $\{B_j\}_{j \in [k]}$ implies that the covariance matrix $\mathrm{Cov}_{\mathcal{G}}(\x)$ will have a block-diagonal structure for any data point $\x$.

We have the following proposition for regionally correlated feature augmentations. The proof is contained in   Appendix~\ref{sec:propweakcorrelationsproof}.

\begin{proposition}\label{prop:weakcorrelations}
Consider a correlated-feature augmentation of the form described above. Further, assume that the smallest eigenvalue of $\E_\x\mathrm{Cov}_{\mathcal{G}_k}(\x)$ is lower bounded by $\sigma$ for every $k$, and $g_k$ is component-wise bounded, i.e., $\|g_k(\x_k)\|_{\infty} \leq M$ for any $k$. Then, we have
\begin{align*}
        \Delta_G \lesssim \frac{M^2 \max_k|B_k|}{\sigma}\sqrt{\frac{\log p}{n}} 
\end{align*}
with probability at least $1 - \frac{1}{p}$.
\end{proposition}
\paragraph{Augmentations with a ``small'' off-diagonal component:}
At this stage, it is natural to ask whether any guarantees are possible for augmentations that do not enjoy the properties of independence or weak correlation.
While we do not provide a guarantee for arbitrary augmentations on high-dimensional data, we present a general technique that we later use to show that the approximation error is indeed vanishing for many popularly used augmentations that include more complex dependencies between features.
Specifically, we state and prove the following result.
\begin{proposition}\label{prop:dependent_bound}
Consider the decomposition $\cg(\X)=\D + \Q$, where $\D$ is a diagonal matrix representing the \textit{independent} feature augmentation part. 
Then, we have
\begin{align}\label{dependent_bound}
    {\Delta}_{G}\lesssim \frac{\|\D-\E\D\| + \|\Q-\E\Q\|}{\mu_p(\E_\x\cg(\x))}. 
\end{align}
\end{proposition}
The proof of Proposition~\ref{prop:dependent_bound}
is provided in Section~\ref{sec:propdependentboundproof}, along with further discussion on the approximation error for dependent feature augmentations.
We use Eq.~\eqref{dependent_bound} to show that even if the quantity $\|\Q - \E\Q\|$ is large (due to dependencies among features in the augmentations), it can be mitigated by the denominator of Eq.~\eqref{dependent_bound} for augmentations for which $\mu_p(\E_\x\cg(\x))$ is large. 
We use this in Appendix~\ref{non_diag} to characterize the approximation error for two examples of augmentations that induce global dependencies between features: a) the new \emph{random-rotation} augmentation that we introduced in Section~\ref{sec:rot}, b) the cutout augmentation which is popular in deep learning practice~\citep{devries2017improved}.



\section{Case Studies:  Applying Our Theory to Study Different Classes of DA}\label{case_study}
In this section, we will use the meta-theorems established in Section \ref{regress-results} and \ref{class-results} to get further insight into the impact of DA on generalization. In particular, we present and interpret generalization guarantees for commonly used augmentations including:  \textit{Gaussian noise injection, randomized mask, cutout, salt-and-pepper noise, and our newly proposed random-rotation augmentation}. 

\subsection{Gaussian noise injection}
As a preliminary example, we note that Theorem ~\ref{gen_bound} generalizes and recovers the existing bounds on the ridge and ridgeless estimators \citep{bartlett2020benign,tsigler2020benign}. 
This is consistent with classical results~\citep{bishop1995training} that show an equivalence between augmented ERM with Gaussian noise injection and ridge regularization. Specifically, an application of the theorem to Gaussian noise injection with variance $\sigma^2$ recovers existing bounds for ridge estimators with regularization parameter $\lambda = n\sigma^2$, where the number of samples $n$ controls the amount of regularization applied to the estimator. 
For completeness, we include this bound in Appendix~\ref{pf_reg_cor}. 


\subsection{Randomized masking}
Next, we consider the popular randomized masking augmentation (both the biased and unbiased variants), in which each coordinate of each data vector is set to $0$ with a given probability, denoted by the masking parameter $\beta \in [0,1]$. 
The unbiased variant of randomized masking rescales the features so that the augmented features are unbiased in expectation.
This type of augmentation has been widely used in practice~\citep{he2022masked, konda2015dropout}\footnote{We note that a superficially similar implicit regularization mechanism is at play in \emph{dropout}~\citep{bouthillier2015dropout}, where the parameters of a neural network are set to $0$ at random. In contrast to random masking, dropout zeroes out model parameters rather than data coordinates.}, and is a simplified version of the popular cutout augmentation \citep{devries2017improved}. 
The following corollary characterizes the generalization error arising from the randomized mask augmentation in regression tasks.

\begin{corollary}[\textbf{Regression bounds for unbiased randomized mask}]\label{cor_rm}
    Consider the unbiased randomized masking augmentation $g(\x) = [b_1\x_1,\dots,b_p\x_p]/(1-\beta)$, where $b_i$ are i.i.d. Bernoulli$(1-\beta)$. Define $\psi=\frac{\beta}{1-\beta}\in [0, \infty)$.
    Let $L_1$, $L_2$, $\kappa$, $\delta'$ be universal constants as defined in Theorem \ref{gen_bound}. Then, for any set $\mathcal{K}\subset \{1,2,\dots,p\} $ consisting of $k_1$ elements and some choice of $k_2\in [0, n]$, the regression MSE is upper-bounded by
    \begin{align*}
        \mathrm{MSE} \lesssim& 
       \underbrace{\left\|{\theta}^{*}_{\mathcal{K}}\right\|_{{\Sigma}_{\mathcal{K}}}^{2}
        +\left\|{\theta}_{\mathcal{K}^c}^{*}\right\|_{{\Sigma}_{\mathcal{K}^c}}^{2}\frac{(\psi n + p - k_1)^2}{n^2 + (\psi n + p -k_1)^2}}_{\mathrm{Bias}}\\
        +& \underbrace{\left(\frac{k_2}{n}+\frac{ n(p-k_2)}{(\psi n + p - k_2)^2}\right)\log n}_{\mathrm{Variance}}
        +\underbrace{ \sigma^2_z\sqrt{\frac{\log n}{n}}\|\T^*\|_{\SIG}}_{\mathrm{Approx. Error}}
    \end{align*}
    with probability at least $1 - \delta'-n^{-1}$. 
\end{corollary}
Note that $\psi=\frac{\beta}{1-\beta}$ increases monotonically in the mask probability $\beta$, Corollary~\ref{cor_rm} shows that bias increases with the mask intensity $\beta$, while the variance decreases. 
Figure~\ref{bias_var} empirically illustrates these phenomena through a bias-variance decomposition. 
In fact, the regression MSE is proportional to the expression for MSE of the least-squares estimator (LSE) on isotropic data, suggesting that randomized masking essentially has the effect of \emph{isotropizing the data}.
As prior work on overparameterized linear models demonstrates~\citep{muthukumar2020harmless,hastie2019surprises,bartlett2020benign}, the LSE enjoys particularly low variance, but particularly high bias when applied to isotropic, high-dimensional data.
For this reason, random masking turns out to be superior to Gaussian noise injection in reducing variance, but much more inferior in mitigating bias. 
We also note  that the approximation error is relatively minimal, of the order $\sqrt{\frac{\log n }{n}}$. It is easily checked that the approximation error is dominated by the bias and variance as long as $p \ll n^2$ (and hence the lower bounds of \cite{tsigler2020benign} imply tightness of our bound in this range).


We also derive  guarantees for regression with the biased variant of random masking in Corollary \ref{bias-rm} in Appendix \ref{pf_reg_cor}. In Appendix \ref{pf_cls_cor}, we provide bounds for unbiased random mask in the classification setting (note that Theorem \ref{bias_cls} implies the biased and unbiased case behave similarly for classification). 
\subsubsection{Feature-adaptive random masking}
We consider, as in Corollary \ref{het_mask}, the case of a nonuniform random masking augmentation in which the features that encode signal are masked with a lower probability than the remaining features. 
Specifically, we consider the $k$-sparse model where $\T^* = \sum_{i\in\mathcal{I}_{\mathcal{S}}}\alpha_i\mathbf{e}_i$ and $|\mathcal{I}_{\mathcal{S}}|=k$. Define the parameter $\psi:=\frac{\beta}{1-\beta}$ where $\beta$ is the probability of masking a given feature. Suppose that we employ a nonuniform mask across features, i.e.~$\psi_i=\psi_1$ if $i \in \mathcal{I}_{\mathcal{S}}$ and is equal to $\psi_0$ otherwise. Conceptually, a good mask should retain the semantics of the original data as much as possible while masking the irrelevant parts. We can study this principle analytically through the regression and classification generalization bounds for this type of non-uniform masking. Below we present the regression result, and defer the proofs to Appendix \ref{pf_reg_cor} and the analogous classification result to Corollary \ref{hrm-class} in Appendix \ref{pf_cls_cor}.

\begin{corollary}[\textbf{Non-uniform random mask in $k$-sparse model}]\label{het_mask} Consider the $k-$sparse model and the non-uniform random masking augmentation where $\psi=\psi_1$ if $i\in \mathcal{I}_{\mathcal{S}}$ and $\psi_0$ otherwise. Then, if $\psi_1 \leq \psi_0$, we have with probability at least $1-\delta - \exp(-\sqrt{n}) - 5n^{-1}$
\begin{align*}
    &\mathrm{Bias} \lesssim  \frac{\left(\psi_{1}n + \frac{\psi_{1}}{\psi_{0}}\left(p-|\mathcal{I}_{\mathcal{S}}|\right)\right)^2}{n^2 + \left(\psi_{1}n + \frac{\psi_{1}}{\psi_{0}}\left(p-|\mathcal{I}_{\mathcal{S}}|\right)\right)^2}\|\T^*\|_\SIG^2,~~
    \mathrm{Variance} \lesssim \frac{|\mathcal{I}_{\mathcal{S}}|}{n} + \frac{n\left(p - |\mathcal{I}_{\mathcal{S}}|\right)}{\left(\psi_0 n + p - |\mathcal{I}_{\mathcal{S}}|\right)^2},\\
    &\mathrm{Approx. Error} \lesssim \sqrt{\frac{\psi_1}{\psi_0}}\sigma_z^2\sqrt{\frac{\log n}{n}}\|\T^*\|_{\SIG}.
\end{align*}
On the other hand, if $\psi_1 > \psi_0$, we have (with the same probability)
\begin{align*}
    \mathrm{Bias} \lesssim  \|\T^*\|_{\SIG^2},~~\mathrm{Variance} \lesssim \frac{\left(\frac{\psi_1}{\psi_o}\right)^2 + \frac{|\mathcal{I}_{\mathcal{S}}|}{n}}{\left(\frac{\psi_1}{\psi_o}+\frac{|\mathcal{I}_{\mathcal{S}}|}{n}\right)^2}, ~~\mathrm{Approx. Error} &\lesssim \sqrt{\frac{\psi_0}{\psi_1}}\sigma_z^2\sqrt{\frac{\log n}{n}}\|\T^*\|_{\SIG}
\end{align*}
\end{corollary}

We can see that the bias decreases as the mask ratio $\psi_1/\psi_0$ between the signal part ($\mathcal{I}_{\mathcal{S}}$) and the noise part decreases. This corroborates the idea that a successful augmentation should retain semantic information as compared to the noisy parts of the data. Corollary~\ref{het_mask} implies that for consistency as $n, p \to \infty$, we require $\frac{1}{n} \ll \frac{\psi_1}{\psi_0} \ll \frac{n}{p}$.
This is because we must mask the noise features sufficiently more than the the signal feature for the bias to be small, but the two mask probabilities cannot be too different to allow the approximation error to decay to zero. We note that the bound has a sharp transition---if we mask the signal more than the noise, the bias bound becomes proportional to the null risk (i.e.~the bias of an estimator that always predicts $0$).

\subsection{Random cutout}
Next, we consider the popularly used \emph{cutout} augmentation~\citep{devries2017improved}, which picks a set of $k$ (out of $p$) consecutive data coordinates at random and sets them to zero.
Interestingly, our analysis shows that the effect of the cutout augmentation is very similar to the simpler-to-analyze random mask augmentation.
The following corollary shows that the generalization error of cutout is equivalent to that of randomized masking with dropout probability $\beta=\frac{k}{p}$. The proof of this corollary can be found in Appendix \ref{pf_reg_cor}.
\begin{corollary}[\textbf{Generalization of random cutout}]\label{cor:cutout}
    Let $\Th^{\text{cutout}}_k$ denote the random cutout estimator that zeroes out $k$ consecutive coordinates (the starting location of which is chosen uniformly at random). Also, let $\Th^{\text{mask}}_{\beta}$ be the random mask estimator with the masking probability given by $\beta$. We assume that $k=O(\sqrt{\frac{n}{\log p}})$.
    Then, for the choice $\beta=\frac{k}{p}$ we have
    $$\mathrm{MSE}(\Th^{\text{cutout}}_k)\asymp \mathrm{MSE}(\Th^{\text{mask}}_{\beta}),~~\mathrm{POE}(\Th^{\text{cutout}}_k)\asymp \mathrm{POE}(\Th^{\text{mask}}_{\beta}).$$
\end{corollary}
This result is consistent with our intuition, as the cutout augmentation zeroes out $\frac{k}{p}$ coordinates on average.

\subsection{Composite augmentation: Salt-and-pepper}
Our meta-theorem can also be applied to \textit{compositions} of multiple augmentations.
As a concrete example, we consider a ``salt-and-pepper'' style augmentation in which each coordinate is either replaced by random Gaussian noise with a given probability, or otherwise retained. Specifically, salt-and-pepper augmentation modifies the data as $g(\x) = [\x'_1,\dots,\x'_p]$, where 
    $\x'_i=\x_i/(1-\beta)$ with probability $1-\beta$ and otherwise $\x_i'=\mathcal{N}(\mu,\sigma^2)/(1-\beta)$. This is clearly a composite augmentation made up of randomized masking and Gaussian noise injection. For simplicity, we only consider the case where $\mu=0$, since it results in an augmentation which is unbiased on average. The regression error of this composite augmentation is described in the following corollary, which is proved in Appendix~\ref{pf_reg_cor}.
\begin{corollary}[
\textbf{Generalization of Salt-and-Pepper augmentation in regression}]\label{salt_cor}
    The bias, variance and approximation error of the estimator that are induced by salt-and-pepper augmentation (denoted by $\Th_{\text{pepper}}(\beta,\sigma^2)$) are respectively given by:
    \begin{align*}
         \mathrm{Bias}[\Th_{\text{pepper}}(\beta,\sigma^2)] ~&\lesssim ~
        \left(\frac{\lambda_1(1-\beta)+\sigma^2}{\sigma^2}\right)^2\mathrm{Bias}\left[\Th_{\text{gn}}\left(\frac{\beta\sigma^2}{(1-\beta)^2}\right)\right],\\
         \mathrm{Variance}[\Th_{\text{pepper}}(\beta,\sigma^2)] ~&\lesssim ~ \mathrm{ Variance}\left[\Th_{\text{gn}}\left(\frac{\beta\sigma^2}{(1-\beta)^2}\right)\right], \\
          \mathrm{Approx.Error}[\Th_{\text{pepper}}(\beta,\sigma^2)] ~&\asymp~ \mathrm{Approx.Error}[\Th_{\text{rm}}(\beta)].
    \end{align*}
    where $\Th_{\text{gn}}(z^2)$ and $\Th_{\text{rm}}(\gamma)$ denotes the estimators that are induced by Gaussian noise injection with variance $z^2$ and random mask with dropout probability $\gamma$, respectively.
    Moreover, the limiting MSE as $\sigma \to 0$ reduces to the MSE of the estimator induced by random masking (denoted by $\Th_{\text{rm}}(\beta)$):
    \begin{align*}
        \lim_{\sigma\rightarrow 0} \mathrm{MSE}[\Th_{\text{pepper}}(\beta,\sigma^2)] = \mathrm{MSE}[\Th_{\text{rm}}(\beta)].
    \end{align*}
\end{corollary}
Corollary~\ref{salt_cor} clearly indicates that the generalization performance of the salt-and-pepper augmentation interpolates between that of the random mask and Gaussian noise injections, in the sense that it reduces to random mask in the limit of $\sigma\rightarrow 0$, and also has a comparable bias and variance to Gaussian noise injection. More precisely, as we show in the proof of this corollary, this interpolation property is a result of the fact that the eigenvalues of the augmented covariance are the harmonic mean of the eigenvalues induced by random mask and Gaussian noise injection respectively, i.e.~
\begin{align*}
    \lambda_{{pepper}}(\beta,\sigma^2)^{-1} = \lambda_{{rm}}(\beta)^{-1} + \beta^{-1}\lambda_{{gn}}(\sigma^2)^{-1}.
\end{align*}

\subsection{A new ``random-rotation" augmentation}
\label{sec:rot}
Our framework can also serve as a testbed for designing new augmentations that have desired properties in terms of how they effect the spectrum.
As an example, we introduce a novel augmentation that performs multiple rotations in random planes. Specifically, for an input $\x\in\mathbb{R}^p$ and user specified rotation angle $\alpha$, we perform the following steps:
\begin{enumerate}
    \item Pick an orthonormal basis $[\uu_1,\uu_2,\dots,\uu_p]$ for the entire $p$-dimensional space uniformly at random, i.e.~from the Haar measure.
    \item Divide the basis into sets of $\frac{p}{2}$ orthogonal planes $\U_1,\U_2,\dots,\U_{\frac{p}{2}}$, where
    $\U_i = [\uu_{2i-1}, \uu_{2i}]$ and $i=1,2,\dots,\frac{p}{2}$.
    \item Rotate $\x$ by an angle $\alpha$ in each of these planes $\U_i$, $i=1,2,\dots,\frac{p}{2}$.
\end{enumerate} 
Note that in an implementation of aSGD, we would pick an independent orthonormal basis for each iteration and each training example in Step 1.
Ultimately, the augmentation mapping is given by
    \begin{align*}g(\x)&=\prod\limits_{i=1}^{\frac{p}{2}}\left[\I + \sin\alpha (\uu_{2i-1}\uu_{2i}^\top-\uu_{2i}\uu_{2i-1}^\top) + (\cos\alpha - 1)(\uu_{2i}\uu_{2i}^\top + \uu_{2i-1}\uu_{2i-1}^\top) \right]\x \\
    &=\left[\I + \sum\limits_{i=1}^{\frac{p}{2}}\sin\alpha (\uu_{2i-1}\uu_{2i}^\top-\uu_{2i}\uu_{2i-1}^\top) + (\cos\alpha - 1)(\uu_{2i}\uu_{2i}^\top + \uu_{2i-1}\uu_{2i-1}^\top)\right]\x.
    \end{align*}
The induced augmentation covariance is given by $$\cg(\X)=\frac{4(1-\cos \alpha)}{np}\left(
        \operatorname{Tr}\left(\X^\top\X\right)\I - \X^\top\X\right).$$ The full derivation is deferred to Appendix~\ref{common_est}.

We can use our theory to study the generalization error for our proposed augmentation and compare it with ridge regression. Interestingly, this augmentation enjoys good generalization performance, \textit{regardless of the signal model}. This result is summarized through the following Corollary.
 \begin{corollary}[\textbf{Generalization of random-rotation augmentation}]\label{rot} 
\vspace{1mm}
Let $\hat{\T}_{\text{rot}}$ denote the  estimator induced by the random-rotation augmentation with angle parameter $\alpha$. An application of Theorem~\ref{gen_bound} yields $ \mathrm{Bias}(\hat{\T}_{\text{rot}}) \asymp \mathrm{Bias}(\hat{\T}_{\text{lse}}),$
for sufficiently large $p$ (overparameterized regime), as well as the variance bound
    $\mathrm{Var}(\hat{\T}_{\text{rot}})\lesssim \mathrm{Var}(\hat{\T}_{\text{ridge,}\lambda}).$
Let $\hat{\T}_{\text{lse}}$ and $\hat{\T}_{\text{ridge,}\lambda}$ denote the least squared estimator and ridge estimator with ridge intensity $\lambda = {np^{-1}(1-\cos\alpha)\sum_{j}\lambda_j}$.
    The approximation error can also be shown to decay as
    \begin{align*}
        \mathrm{Approx.~ Error}(\hat{\T}_{\text{rot}})\lesssim \max\left(\frac{1}{n}, \frac{\lambda_1}{\sum_{j>1} \lambda_j}\right). 
    \end{align*}
\end{corollary}    
The proof of the bias and variance expressions are provided in Appendix~\ref{pf_cls_cor}, and the proof of the approximation error is provided in Appendix~\ref{non_diag} (this is the most involved step as random-rotation augmentations induce strong dependencies among features).
Corollary~\ref{rot} shows that, surprisingly, this simple augmentation leads to an estimator not only having the best asymptotic bias that matches that of LSE, but also reduces variance on the order of ridge regression.
\section{Experiments}\label{sec:exp}
In this section, we complement our theoretical analysis with empirical investigations. In particular, we explore: 1. Differences between aSGD which is used in practice and the closed-form aERM solution analyzed in this paper, 2. Comparisons between the generalization of different types of augmentations studied in this work, 3. Multiple factors that influence the efficacy of DA, including signal structure and covariate spectrum, and 4. Comparisons between different augmentation strategies, namely precomputed augmentations versus aERM. We provide our Python implementations in \href{https://github.com/nerdslab/augmentation-theory}{https://github.com/nerdslab/augmentation-theory}.

\subsection{Convergence of aSGD to the closed-form aERM solution}\label{sec:expts-aSGD}

In this paper, we mathematically study a-ERM (the solution in Equation~\eqref{DAobj}); however, the solution used in practice is obtained by running a-SGD (Algorithm~\ref{alg:sgd}).
In this set of experiments, we investigate the convergence of Algorithm \ref{alg:sgd} to the solution of Eq.~\ref{DAobj} to verify that our theory reflects the solutions obtained in practice. To this end, we use an example in the overparameterized regime with $p=128 \geq n=64$ with the random isotropic signal $\T^* \sim \mathcal{N}(\mathbf{0}, \I_p)$ and the observation noise $\epsilon \sim \mathcal{N}(0,0.25)$.
We choose a decaying covariate spectrum of the form $\SIG_{ii} \propto \gamma^{i}$, where $\gamma$ is chosen such that 
$\mu_p(\SIG) = 0.6 \mu_1(\SIG)$.
We want to understand the interplay between the convergence rate of aSGD with batch and augmentation size (formally, the augmentation size is the number of augmentations made for each draw of the training examples). We run the aSGD algorithm with different batch sizes and augmentation sizes in the range given by $(64, 1), (32, 2), \dots,(2,32),(1, 64)$. Note that the computation cost is proportional to the (batch size) $\times$ (augmentation size) per backward pass. Fig. \ref{fig:conv} illustrates the convergence rate in terms of the number of backward passes. We observe that the convergence rates are fairly robust to different choices of batch and augmentation sizes.

\begin{algorithm}
    \caption{\hspace{2mm} \texttt{Augmented Stochastic Gradient Descent (aSGD)}   \label{alg:sgd}} 
    \Input{\texttt{ Data} $\x_i,~i=1,\dots,n$; \texttt{Learning rates }$\eta_t,~t=1,\dots$; \texttt{transformation distribution $\mathcal{G}$};
    \texttt{batch size B}; 
    \texttt{aug size H};}
    
    \Init{$\Th\leftarrow \Th_0$}{}
    \While{termination condition not satisfied}{
    \For{k=1,\dots,$\frac{\texttt{n}}{\texttt{B}}$}{
    \For{i=1,\dots,B \text{in the batch $\mathcal{B}_k$}}{
    \texttt{Draw H augmentations $g_{ij}\sim \mathcal{G}$, $j=1,\dots,$H}
    }
    $\Th_{t+1}\leftarrow \Th_t - \eta_t\sum\limits_{i=1}^B\sum\limits_{j=1}^H\nabla_{\T}(\langle\T,g_{ij}(\x_i)\rangle -y_i)_2^2 |_{\T=\Th_t}$
    }
    }
\end{algorithm}

\vspace{-2mm}
\begin{figure}[!t]
    \centering
     {\includegraphics[width=0.95\textwidth]
      {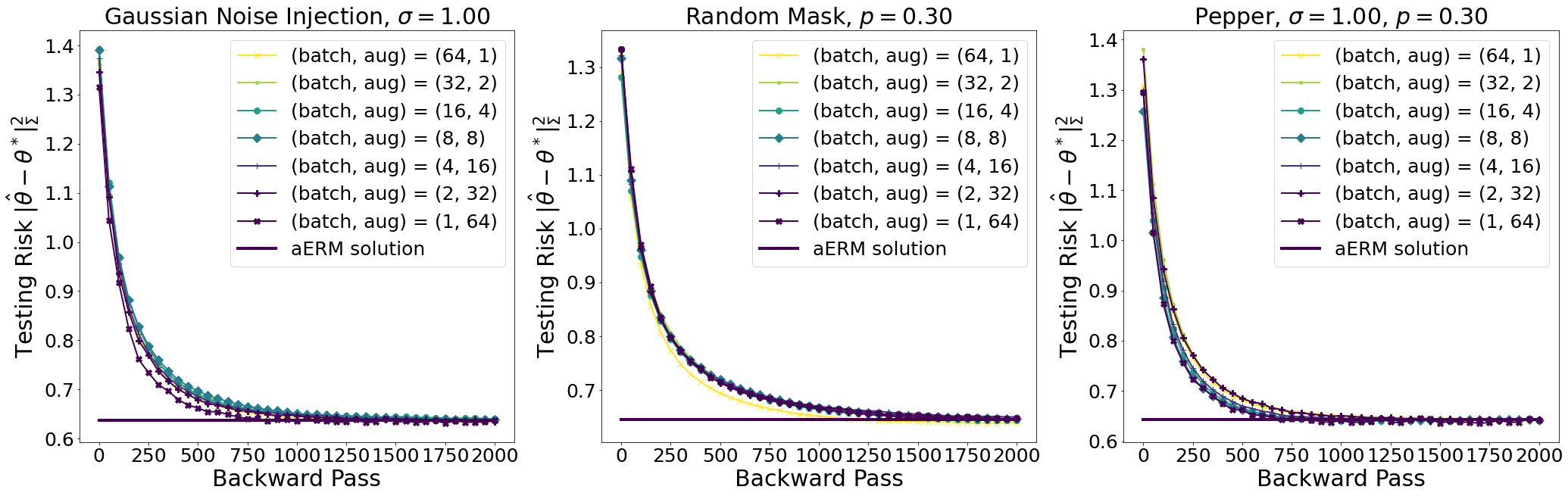}\vspace{-0mm}}
    \caption{\footnotesize{{\em Convergence of augmented stochastic gradient descent (a-SGD, Algorithm~\ref{alg:sgd})} as a function of the number of backward passes to the closed-form solution of the a-ERM objective (Equation~\eqref{DAobj}). The result shows fairly stable convergence across different batch sizes and augmentation copies per sample.\label{fig:conv}
    }
    \label{fig:sgd}}
\end{figure}

\subsection{Comparisons of different types of augmentations}\label{sec:expts-compareaugs}
In this section, we compare the generalization of: 1) Gaussian noise injection~\citep{bishop1995training}, 2) random mask~\citep{he2022masked}, and 3) random rotation (which we introduced in Section~\ref{sec:rot}). 
As in Section~\ref{sec:expts-aSGD}, we consider the random isotropic signal $\T^* \sim \mathcal{N}(\mathbf{0}, \I_p)$.
We compare regression and classification tasks; in the former, we set the noise standard deviation as $\sigma_\varepsilon=0.5$ while in the latter, we set the label noise parameter as $\nu^* = 0.1$. We consider diagonal covariance $\SIG$ and two choices of spectrum: 1. isotropic (i.e. $\SIG = \I_p$) and 2. decaying spectrum where $\Sigma_{ii}\propto \gamma^i$ with $\gamma=0.95$.

Figure~\ref{fig:which_good} illustrates different trade-offs (bias/variance for regression, contamination/survival for classification) for the three canonical augmentations. 
The hyperparameters for the respective augmentations are: 1) the standard deviation $\sigma\in\mathbb{R}^+$ of the Gaussian noise injection, 2) the masking probability $\beta\in [0, 1]$ of the random mask, and 3) the rotation angle $\alpha\in [0, 90]$.
We can make the following observations from Figure~\ref{fig:which_good}:
\begin{enumerate}
    \item For isotropic data, all three augmentations achieve similar results in terms of generalization, while for the case of decaying spectrum, Gaussian injection and random rotation outperform random mask when their respective hyperparameters are optimally tuned. 
    \item For regression, Gaussian injection requires careful hyperparameter tuning in the range $[0, 1.8]$, while random mask and random rotation are fairly robust in performance in the entire tested hyperparameter range. A possible explanation for this observation is that the random mask and rotation hyperparameters are \textit{scale free} of the data (while the noise injection hyperparameter is not). 
    \item In the classification task, all the augmentations enjoy relatively robust generalization with respect to their hyperparameters. This verifies our theoretical observations in Propositions \ref{gen_compare} and \ref{prop:non-uniformmask}.
    \item 
    Our novel random rotation augmentation achieves the best of both worlds across different data distributions and tasks, achieving comparable generalization to noise injection when optimally tuned, while also being robust with respect to hyperparameter choice (like random mask). This observation is consistent with the theoretical prediction of Corollary~\ref{rot}.
\end{enumerate}

\begin{figure}[!htbp]
      \subfloat[][\footnotesize{{\em Bias and variance distribution comparison in uniform covariate spectrum.} }]{\includegraphics[width=1\textwidth]
   {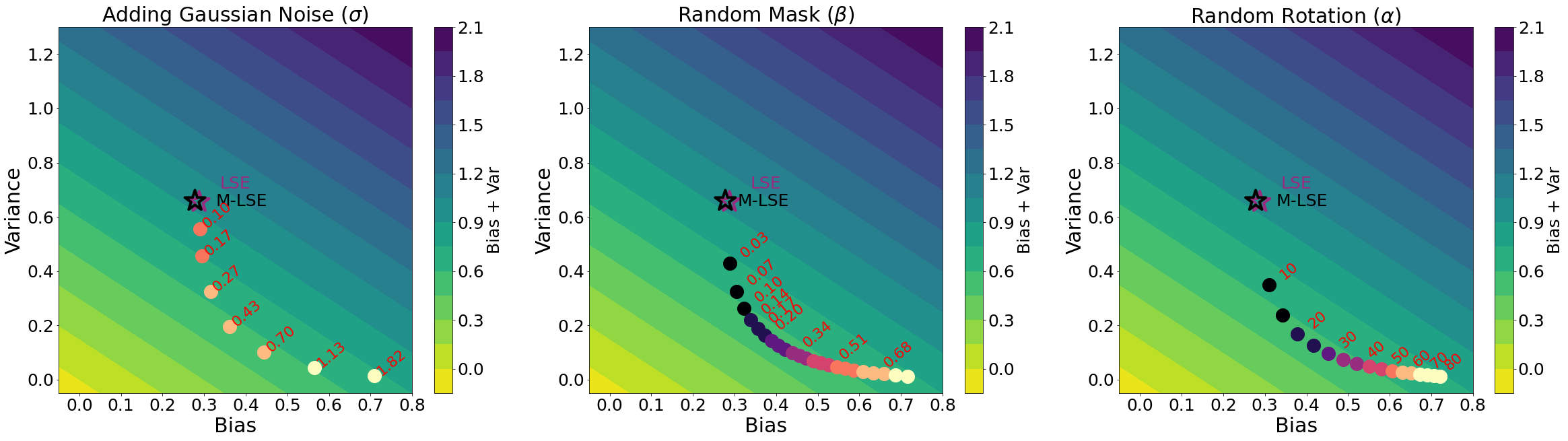}\vspace{-3mm}}\\
   \subfloat[][\footnotesize{{\em Log survival and contamination distribution comparison in uniform covariate spectrum.} }]
   {\includegraphics[width=1\textwidth]
   {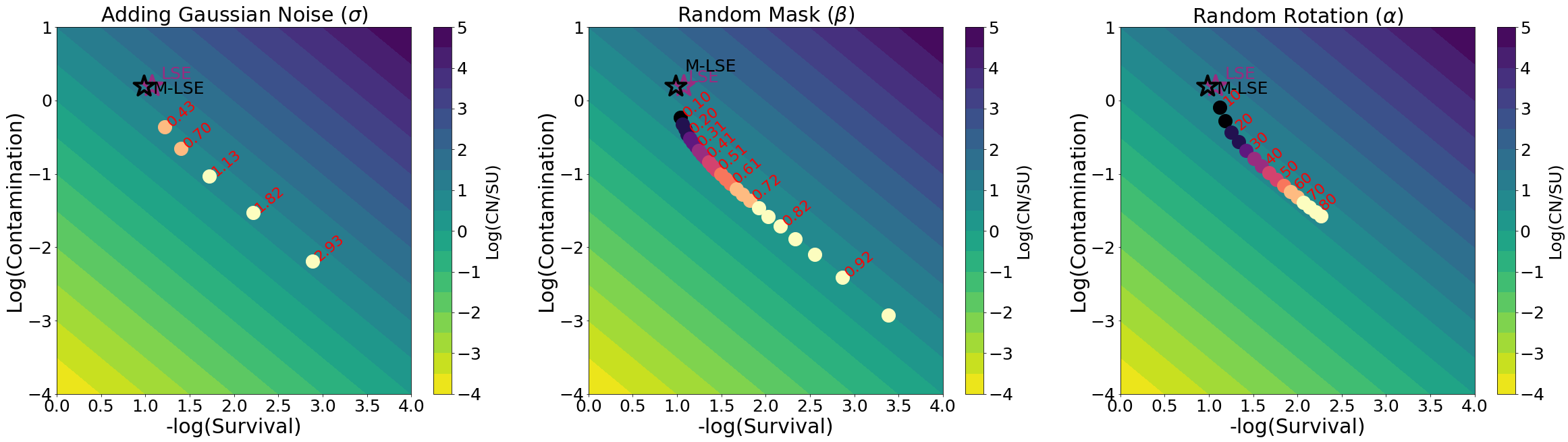}\vspace{-3mm}}\\
      \subfloat[][\footnotesize{{\em Bias and variance distribution  comparison in decaying covariate spectrum, $\gamma=0.95$}}]{\includegraphics[width=1\textwidth]
   {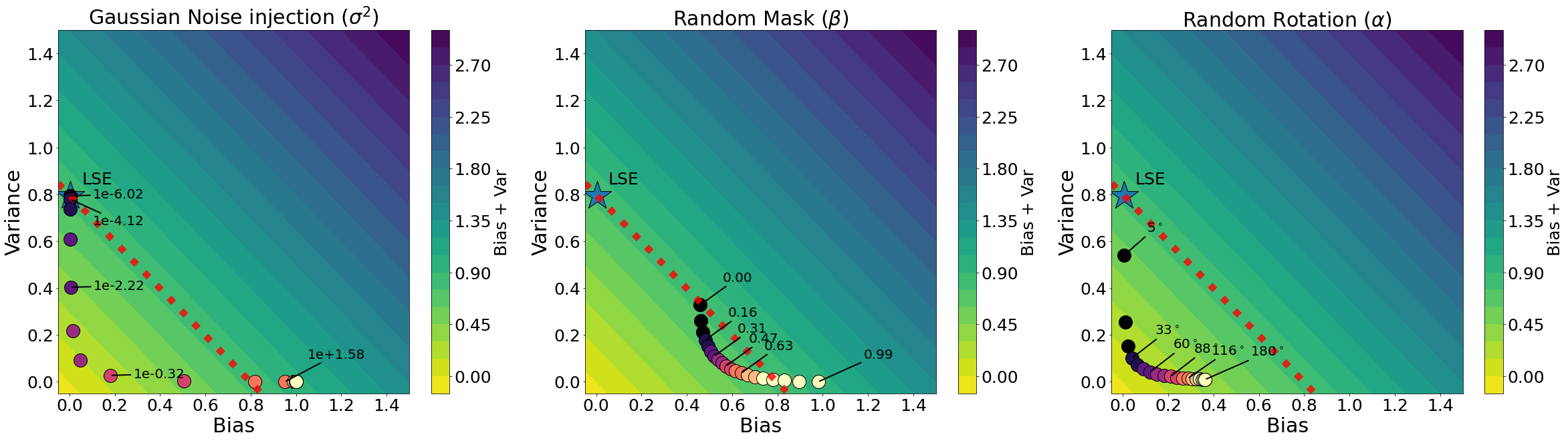}\vspace{-3mm}}\\
   \subfloat[][\footnotesize{{\em Log survival and contamination for  decaying covariate spectrum, $\gamma=0.95$}}]{\includegraphics[width=1\textwidth]
   {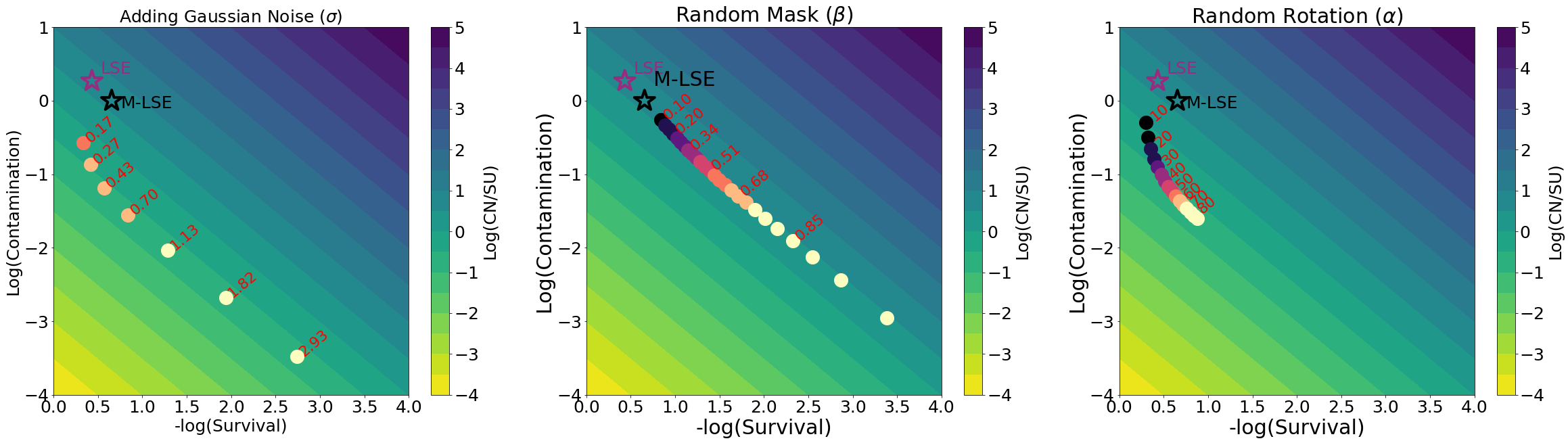}\vspace{-3mm}}
\caption{\footnotesize{\em Visualizing the generalization error for different augmentations, across regression and classification tasks.} 
In this figure we plot the bias/variance (a), (c) and contamination/survival distributions (b), (d) of Gaussian noise injection, random mask, and random rotation.
The numbers reflect the respective hyperparameters $\sigma,\beta,\alpha$. 
}\label{fig:which_good}
\end{figure}

\subsection{Studying the interactions between the original covariance and augmentations}
In this section, we try to understand the impact of the true model $\T^*$ and the data covariance $\SIG$ on the efficacy of different augmentations, focusing on the nonuniform random mask introduced in Section~\ref{the good}. We set the ambient dimension to $p=128$ and consider the noise standard deviation $\sigma_{\epsilon} = 0.5$.

\paragraph{Effect of the true model:} We study the impact of nonuniform masking on the 1-sparse model $\T^* = \mathbf{e}_1$, as depicted in Section \ref{sec:gen_bound_reg} in the regression task and consider isotropic covariance $\SIG = \I_p$.
We vary the probability of the signal feature mask $\beta_{sig}$ while keeping the probability of the noise feature mask $\beta$ fixed at $0.2$.
The results are summarized in Fig. \ref{sig_eff} and verify our analysis in Corollary \ref{het_mask} that noise features should be masked more compared to signal features so that the semantic component in the data is preserved. 
Furthermore, we observe that the differences manifest primarily in the bias, and the variance remains roughly the same.
This is consistent with our variance bound in Corollary~\ref{het_mask}, which depends only on the probability of the noise mask $\beta$.

\begin{figure}[!t]
    \centering
    {\includegraphics[width=1\textwidth]
   {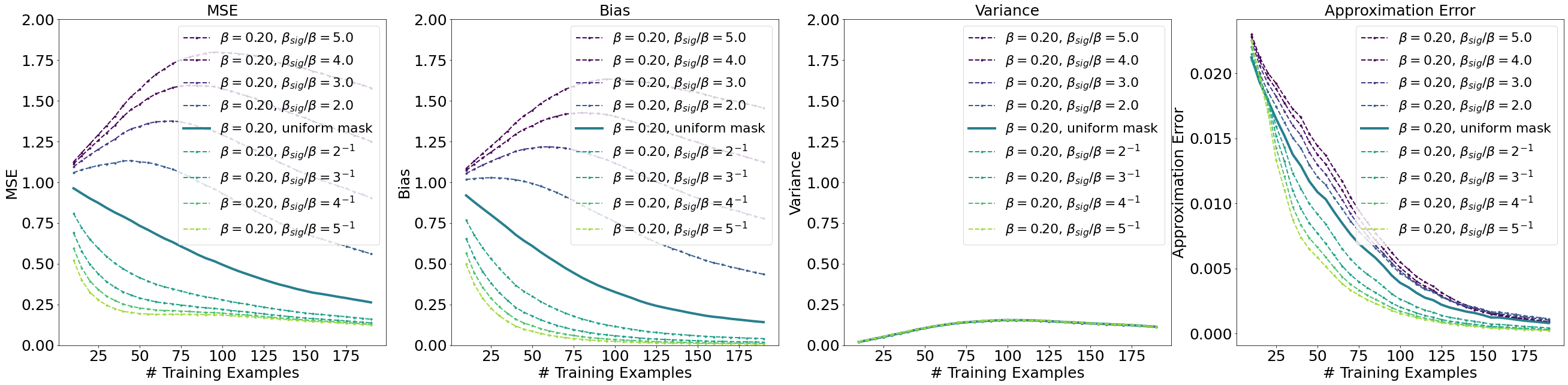}\vspace{-0mm}}
    \caption{\footnotesize{{\em Bias and variance decomposition for non-uniform random masking.} We vary the relative mask intensities ($\beta_{sig} / \beta$) across the signal and noise features. The result suggests that noise features can be augmented more heavily in comparison to the signal features.}\label{sig_eff}\vspace{-3mm}}
\end{figure}

\vspace{-1mm}
\begin{figure}[!t]
\centering
    \setcounter{subfigure}{0}
      {\includegraphics[width=.95\textwidth]
   {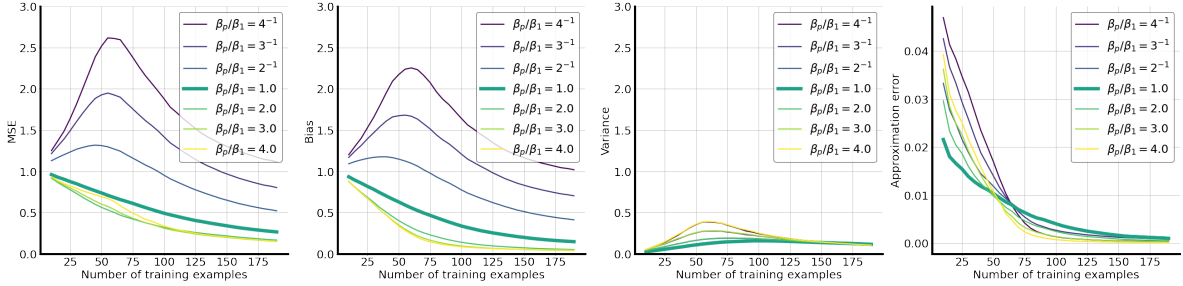}}
\caption{\footnotesize{{\em Bias and variance decomposition of random masking for a bi-level spectrum.} 
We investigate the bi-level random mask strategies in data with decaying spectrum $\propto 0.95^{i}$. The first half of features are masked with probability $\beta_1$ while the rest are with $\beta_p$.
We vary the ratio between the intensity $\beta_p/\beta_1$.
We observe that augmenting more for features with higher variance benefits generalization.
\label{cov_eff}}}
\vspace{-4mm}
\end{figure}

\paragraph{Effect of the covariance spectrum:} Next, to understand the impact of the covariance spectrum, we consider a setting with a decaying data spectrum $\SIG_{ii} \propto 0.95^i$.
We generate the true model using the random isotropic Gaussian $\T^* \sim \mathcal{N}(0, \I_p)$ and run the experiment $100$ times, reporting the average result.
We consider a \emph{bilevel masking strategy} where the masking probability for the first half of features is set to $\beta_1$, and the second half of features is set to $\beta_p$.
We vary the ratio between $\beta_p$ and $\beta_1$ to investigate whether a feature with larger eigenvalue should be augmented with stronger intensity or not. The result is presented in Fig. \ref{cov_eff}.
We observe from this figure that it is more beneficial to augment more for features with smaller eigenvalues.  

\begin{figure}[!t]
\centering
    \setcounter{subfigure}{0}
    \subfloat[][\footnotesize{Normal adding Gaussian noise with $\sigma=1$. }]
      {\includegraphics[width=.95\textwidth]
   {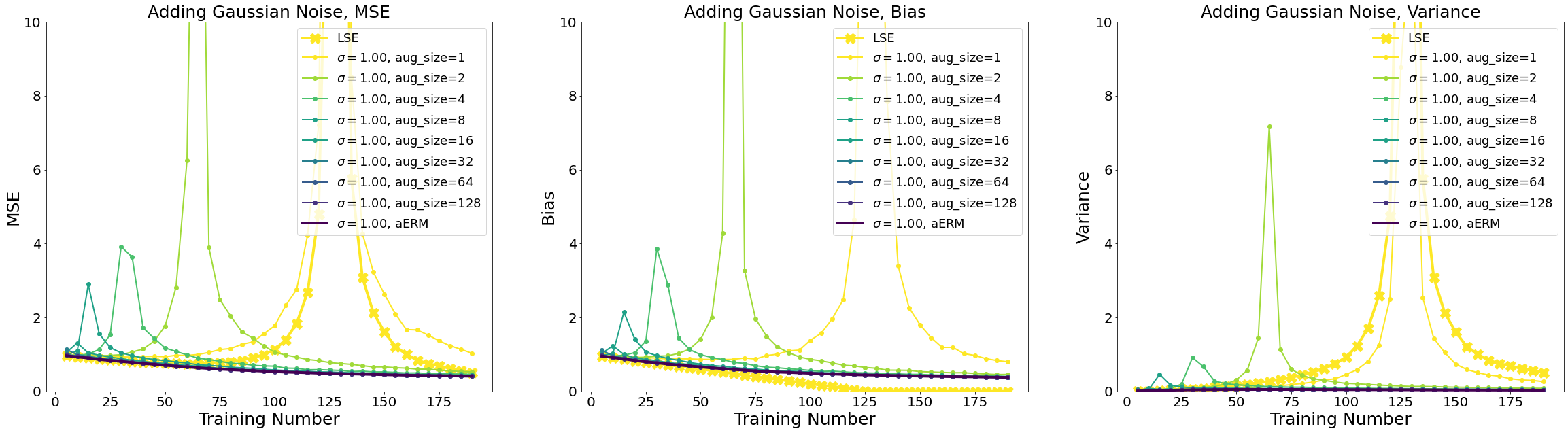}\vspace{-3mm}}
   \setcounter{subfigure}{1}
      \subfloat[][\footnotesize{Normal random mask with $\beta=0.3$. }]{\includegraphics[width=.95\textwidth]
   {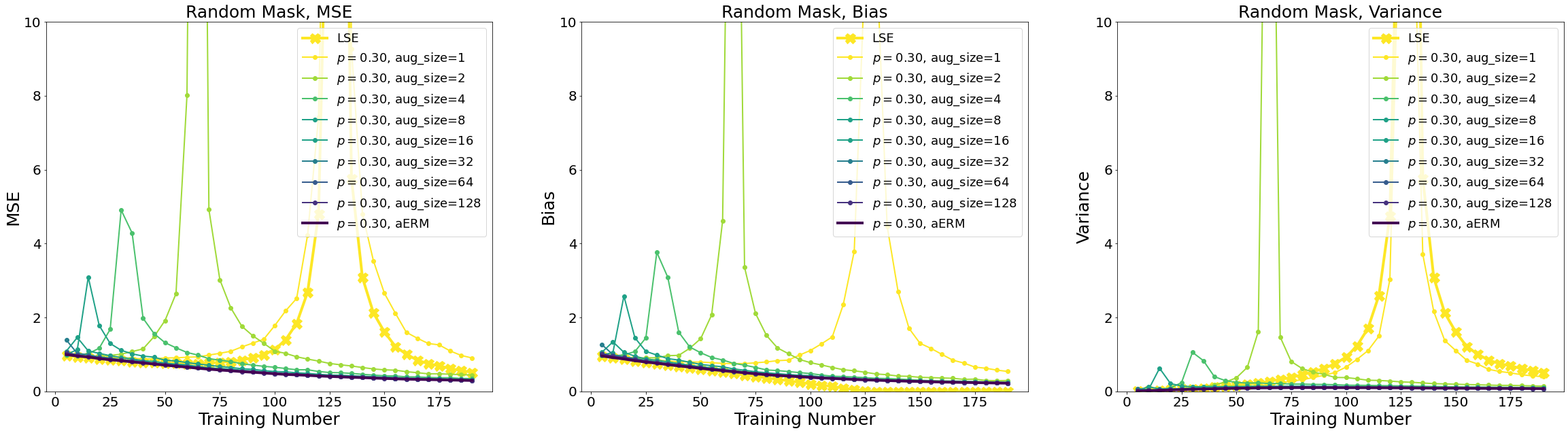}\vspace{-3mm}}
   
   
    \caption{\footnotesize{{\em Pre-computed augmentations versus aERM.} The estimators based on aERM have monotonicity in generalization error with respect to the number of training samples, while the pre-computing methods exhibit the double-descent phenomenon like least-squared estimators. We note that the pre-computing methods shifts the error peak left compared with LSE. Also, the peak appears approximately at the sample number equals to $\frac{p}{k}$, where $k$ is the augmentation size. }\label{fig:finite_vs}}
\end{figure}

\subsection{Comparisons of pre-computing samples vs. augmented ERM}\label{sec:pre_post}

In our final set of experiments, we dig into the differences between pre-computing augmented samples and creating augmentations on-the-fly. 
For this experiment, we generate isotropic random signal $\T^* \sim \mathcal{N}(\mathbf{0},\mathbf{I}_{128})$ and observation noise with standard deviation $\sigma=0.5$. For simplicity, we choose the isotropic covariate spectrum $\SIG=\I_{128}$.
In Figs. \ref{fig:finite_vs} (a)-(b), we observe the well-known double descent peaks \citep{belkin2020two,nakkiran2020optimal} when the training number approaches the ambient dimension $n=p=128$ for LSE, and observe that adding pre-computed augmentation shifts these peaks to the left. The peak for a pre-computing method with an augmentation size $k$ is observed to be approximately at $n=128/k$. Intuitively, this mode of augmentation virtually increases the size of the training data: in particular, if we had $128/k$ original data points the induced total training size (including original data points and augmentations) becomes equal to $(128/k) \times k = 128$. 

Interestingly, both the magnitude of the peak and the width decrease as we increase the augmentation size, and the peak almost disappears when $k >8$. The general behavior of pre-computing is observed to approach aERM as $k$ increases. Another interesting observation is that, unlike LSE which only has a double descent peak in the variance, pre-computing augmentations induces peaks in both the bias and the variance.
A possible explanation for peaks appearing even in the bias term is that the \emph{variance induced by a finite number of augmentations} is itself embedded in the bias term.

\section{The good, the bad and the ugly sides of data augmentation}\label{good_bad}
In this section, we will dive into specific implications of our theory and experiments, and list key properties through which augmentations can be good (improve learning), bad (hurt learning), or ``ugly'' (have unexpected/surprising outcomes).

\subsection{The good: when DA helps generalization}\label{the good}

\paragraph{1) Data-adaptive spectral  modification:} Section~\ref{impli} provided an interpretable and succinct characterization of the impact of any augmentation: namely, that it modifies the entire spectrum of the data covariance. We show that this modification to the covariance can be translated into modifications to the two \emph{effective ranks}, as defined in~\cite{bartlett2020benign}, and used to derive generalization bounds that reveal the impact of a given augmentation. 
This modification can itself be data-adaptive and lead to rich types of Tikhonov-style spectral regularization.
Our analysis of several augmentations (including but not exclusive to the examples listed in Table~\ref{common-augs}) reveals that DA can have a much richer impact on the covariance spectrum, leading to unique benefits in generalization.
For example, our theoretical and empirical analysis of the \emph{non-uniform} random mask (Cor. \ref{het_mask}) and random-rotation augmentation (Cor. \ref{rot}) reveals that it is possible to generate data-adaptive regularizers through DA that reduce variance without a harmful increase in bias.




\vspace{-2mm}
\paragraph{2) Variance reduction:} Our analysis in Section \ref{impli} implies that any stochastic augmentation will reduce variance.
Thus, in situations where the bias (of the original, unaugmented estimator) is already minimal or has a minimal impact on task performance, we expect many types of DA to be beneficial through this form of variance reduction.
For example, Section~\ref{sec:expts-compareaugs} revealed that the random masking augmentation leads to stable improvements in classification performance for any choice of masking rate.
We also see in Figure~\ref{fig:which_good} that Gaussian noise injection, random mask and random rotation \emph{all} improve overall generalization for both classification and moderately overparameterized regression ($p = 2n$), where bias is either minimal or does not significantly impact task performance.
The improvement is maximal for the random rotation augmentation and Gaussian noise injection, both of which incur less bias than the random masking augmentation. An interesting result of our experiments is that our proposed random-rotation augmentation achieves a good generalization performance for a wide range of hyperparameters (rotation angles).


\vspace{-2mm}
\paragraph{3) Reducing effective overparameterization and mitigating ``double descent" behavior:} The variance reduction effect is especially beneficial in overparameterized scenarios where $d$ exceeds $n$ but not by much, which is a scenario that often arises in ML practice.
Here, the variance of the original estimator would be very large, leading to the ``peak" observed in the double descent curve~\citep{belkin2019reconciling}.
In these regimes, DA reduces the effective dimension of the data and effectively creates a synthetic underparameterized regime.
As shown in Figure \ref{fig:finite_vs}, this benefit takes place for both the pre-computed augmentation implementation and the aERM implementation.
These results support the perspective that augmentations can act as ``virtual samples"~\citep{balestriero2022data}.





\subsection{The bad: when DA hurts generalization}\label{the bad}

\paragraph{1) Augmentations can erase helpful data structure by ``isotropizing'' the spectrum}
An important ingredient for generalization of ridge estimators in high-dimensional settings is low-dimensional structure in the data.
One key data structure that is often used to restrict the solution space, is to assume that the data have low-rank structure, or that the data covariance has a sufficiently fast rate of decay in its eigenvalues with the bulk of the signal energy contained in the top eigenvectors~\citep{muthukumar2020harmless,tsigler2020benign,muthukumar2020class}. 
Our theory reveals that some popular augmentations, such as random-masking and group-invariant augmentations, begin to \emph{erase} all of this helpful data structure by \emph{isotropizing} the ``equivalent'' data covariance in expectation.
Putting this in the language of Section 7.2.1, such isotropization has the benefit of variance reduction but this comes at a cost of increased bias.
While we believe this isotropization effect might be specific to high-dimensional linear models and may not occur even for nonlinear kernel methods, it is an important factor that our theory identifies as, overall, pessimistic for generalization.

\vspace{-2mm}
\paragraph{2) Increase in bias could offset variance reduction:} Our theory demonstrates that augmentations can have the undesired effect of increasing the bias of the aERM estimator.
This increase in bias is particularly acute when the data are high-dimensional~\citep{hastie2019surprises,muthukumar2020harmless}.
Figure~\ref{fig:which_good} illustrates that in this regime, certain augmentations like Gaussian noise injection and random mask can lead to poor performance due to high bias. 
Moreover, since bias increases with increased augmentation intensity, these augmentations can even hurt performance if applied too strongly!

Another class of augmentations for which we show counterintuitive effects is the class of \textit{group-invariant augmentations}, i.e.~augmentations that are created with the ostensible aim of inducing invariance in prediction within a specific algebraic group (for example, for the rotation or translation group, augmentations would consist of rotations or translations of images).
In previous work \cite{chen2020group}, the authors show that such group-invariant augmentations always improves generalization through variance reduction; however, they primarily considered the underparameterized regime where bias is much less significant.
We show in Appendix~\ref{pf_cls_cor}, such augmentations may generalize poorly in the overparameterized regime. 
\paragraph{3) Augmentations can induce distribution shift between training and test data}
Finally, we show that augmentations that are \emph{biased-on-average}, meaning that $\E[g(\x)] \neq \x$, can induce an undesirable distribution shift between training and test data that is harmful, particularly for regression tasks. For example, comparing the regression error bounds for the unbiased variant of randomized mask (Corollary~\ref{cor_rm}) and its biased variant (Corollary~\ref{bias-rm}), reveals that the biased variant incurs an additional penalty due to distribution shift.
On the other hand, as predicted in Corollary \ref{bias_cls}, the impact of bias in augmentation on classification tasks might not be as pronounced.
Thus, our results highlight the importance of debiasing augmentations when applied to regression tasks. 

\subsection{The ugly: discrepancies in DA's effect under multiple factors}\label{the ugly}

The previous two subsections highlight ways in which augmentations can both help and harm learning. Now we will discuss a few ways that augmentations can give rise to what we call ``ugly'' behavior, impacting performance in curious and unexpected ways.
\paragraph{1) Differences in under and overparameterized settings}
Our results also highlight ways in which augmentations will impact models differently in the under vs. overparameterized regimes.
Corollary \ref{cor_rm} shows that when applying random masking for regression tasks, the bias and the variance are given by $\mathcal{O}\left(\frac{(\psi n+p)^2}{(n + p)^2}\right)$ and $\mathcal{O}\left(\min(\frac{n}{p}, \frac{p}{n})\right)$, respectively (recall that $p$ is the data dimension and $n$ is the number of training examples). From this, we can draw the following insights: 1. the variance is vanishing in both regimes, and 2. the bias can be controlled in the underparameterized regime $p \ll n$ by adjusting $\psi$ but is otherwise non-vanishing in the overparameterized regime $p \gg n$. 
Said another way, the isotropization effect described in Section 7.2.2 can be beneficial in the underparameterized regime, as the contribution of bias is relatively minimal, but harmful in the overparameterized regime where the contribution of bias can be substantial.
This supports the benefits of group-invariant augmentations shown by~\cite{chen2020group} in the underparameterized regime.

\paragraph{2) Differences between augmentations that are precomputed or generated on-the-fly.}
Our experiments also demonstrated interesting subtleties between precomputing and on-the-fly-generated augmentations (Fig.~\ref{fig:finite_vs}). 
While a small number of pre-computed augmentations induces a similar double-descent behavior to the original LSE, the aERM error gracefully decreases without any interpolation peak.
We also note that the double descent MSE peak shifts for different numbers of pre-computed samples, appearing approximately at the location $n = \frac{p}{k}$, where $p$ and $k$ denote the data augmentation and number of pre-computed aumgentations per sample, respectively. 
The complete mitigation of double descent by aERM can be explained by our theory and is directly connected to the beneficial effect of variance reduction in mitigating double descent (previously observed for ridge regularization~\citep{hastie2019surprises}).
On the other hand, we believe pre-computed augmentation has a different effect of adding ``virtual samples", therefore leading to the observed effect of shifting the effective interpolation threshold to the right.
Proving this rigorously will require novel random matrix theory tools due to the synthetic but correlated nature of the virtual samples.
We defer mathematical analysis of this phenomenon to future work.

\paragraph{3) The effect of weak DA}
Our analysis framework also allows us to understand a counterintuitive phenomenon that emerges for ``weak'' augmentations. It is well-known that Gaussian noise injection approximates the LSE when the variance of the added noise approaches zero. Surprisingly, however, this does not imply that all kinds of DA approach the LSE in the limit of decreasing augmentation intensity. Suppose that the augmentation $g$ is characterized by some hyperparameter $\xi$ that reflects the intensity of the  augmentation (for e.g., mask probability $\beta$ in the case of randomized mask, or Gaussian noise standard deviation $\sigma$ in the case of Gaussian noise injection), and that $\operatorname{Cov}_G(\X)/\xi {\longrightarrow }\operatorname{Cov}_{\infty}$ as $\xi \rightarrow 0$ for some positive semidefinite matrix $\operatorname{Cov}_{\infty}$ that does not depend on $\xi$. Then, the limiting aERM estimator when the augmentation intensity $\xi$ approaches zero is given by
\begin{align}\label{limit_aug}
    \hat{\theta}_{aug}\overset{\xi\rightarrow 0}{\longrightarrow}\operatorname{Cov}_{\infty}^{-1}\X^\top\left(\X \operatorname{Cov}_{\infty}^{-1}\X^\top\right)^{\dagger}\y.
\end{align}
It can be easily checked that this estimator is the minimum-Mahalanobis-norm interpolant of the training data where the positive semi-definite matrix used for the Mahalanobis norm is given by $\operatorname{Cov}_{\infty}$. 
Thus, the choice of augmentation impacts the specific interpolator  that we obtain in the limit of minimally applied DA. For example, the above formula can be applied to random mask with $\operatorname{Cov}_{\infty} = n^{-1}\text{diag}(\X^T\X)\approx \SIG.$
Our empirical results in Figure \ref{conv_rm} confirm this effect as well, where we find a gap between the LSE and even very weak augmentations. We discuss this effect further in Appendix \ref{weakDA}.
\section{Conclusions and Future Work}\label{sec:discussion}
In this paper, we established a new framework to analyze the generalization error for linear models with data augmentation in underparameterized and overparameterized regimes. We characterized generalization error for both regression and classification tasks in terms of the interplay between the characteristics of the data augmentation and spectrum of the data covariance. As a side product, our results also generalize the recent line of research on \emph{harmless interpolation} from ridge/ridgeless regression to settings where the learning objectives are penalized by data dependent regularizers. 

While we do not formally study nonlinear models in this paper,  we believe our analysis provides powerful tools that we could build on to handle the nonlinear case in future work. Our approach extends most naturally to the case of \textit{kernel methods}~(\cite{scholkopf2002learning}), \emph{random features or last-layer retuning}~(\cite{mei2021learning}), and the \emph{neural tangent kernel regime}~(\cite{jacot2018neural}). In these cases, the primary technical challenge is understanding the effect of the augmentation covariance $\cg(X)$, which can be very different than in our analysis, as the feature map in kernel methods is typically nonlinear in the data.
Nevertheless, we believe our generalization analysis can be applied in a plug-and-play manner with such covariance calculations, by combining the insights of our work with tools established, e.g.,  in~\cite{mcrae2022harmless}.

While our current analysis focuses on the effect of augmentations on supervised learning, understanding how augmentation impacts self-supervised and contrastive learning is an important area for future work. In these approaches, the choice of augmentations can have even more harmful effects on learning  and in some cases, cause representational collapse \cite{cabannes2023ssl}. 
Thus, we hope that our results on the implicit spectral manipulation induced by DA can also be applied to study SSL in the future.



\section*{Acknowledgements}
We would like to thank Mehdi Azabou and Max Dabagia for feedback on the work at various stages and many helpful discussions.
This work was funded through NSF IIS-2212182, NSF IIS-2039741, a NSF Graduate Research Fellowship (DGE-2039655), NIH 1R01EB029852, and the support from the Canadian Institute for Advanced Research (CIFAR) through the Global Scholars Program (ELD).

\printbibliography

\newpage
\appendix
\section*{\ul{Appendix}}
\vspace{5mm}
\tableofcontents
\newpage
\section{General Auxiliary Lemmas}
\label{appen:lemma}
\paragraph{Notation}
For a data matrix $\X\in\mathbb{R}^{n\times p}$ with i.i.d. rows with covariance $\SIG$, recall we denote $\mathbf{P}^{\SIG}_{1:k-1}$ and $\mathbf{P}^{\SIG}_{k:\infty}$ as the projection matrices to the first $k-1$ and the remaining eigen-subspaces of $\SIG$, respectively. In addition, we have defined two effective ranks $\rho_{k}({\SIG};c)=\frac{c+\sum_{i>k} \lambda_{i}}{n \lambda_{k+1}} ,~~R_k({\SIG};c)=\frac{(c + \sum_{i>k}\lambda_i)^2}{\sum_{i>k}\lambda_i^2}.$
For convenience, we denote the residual Gram matrix by $\mathcal{A}_k(\mathbf{\X};\lambda) = \lambda\I_n + \mathbf{\X}\boldsymbol{P}^{\SIG}_{k:\infty}\mathbf{\X}^T$.

\begin{lemma}[\textbf{An useful identity for the ridge estimator \citep{tsigler2020benign}}]\label{good_eq}
For any matrix $\V\in\mathbb{R}^{p\times k}$ composed of $k$ independent orthonormal columns (therefore, $\V$ represents a $k$-dimensional subspace), the ridge estimator $\Th = (\X^\top\X + \lambda \I_p)^\top \X^\top \y$ has the property:
    \begin{align}
        (\I_k + \V^\top\X^\top\mathbf{P}_k^{-1}\X\V)\V^\top\hat{\theta} = \V^\top\X^\top\mathbf{P}_k^{-1}\y,
    \end{align}
where $\mathbf{P}_k := \lambda\I_n + \mathbf{\X}\V^{\perp}(\V^\perp)^\top{\X}^\top$ and $\V^{\perp}$ is a $p$ by $p-k$ matrix satisfying $(\V^{\perp})^\top\V = \mathbf{0}$ and $(\V^{\perp})^\top \V^{\perp}=\I_{p-k}$.
\end{lemma}



\begin{lemma}[\textbf{Bernstein-type inequality for sum of sub-exponential variables}]\label{ber_exp}
Let $\x_{1}, \ldots, \x_{n}$ be independent zero-mean sub-exponential random variables with sub-exponential norm at most $\sigma_x^2$. Then for every $a=$ $\left(a_{1}, \ldots, a_{n}\right) \in \mathbb{R}^{n}$ and every $t \geq 0$, we have
$$
\mathbb{P}\left\{\left|\sum_{i=1}^{n} a_{i} \x_{i}\right| \geq t\right\} \leq 2 \exp \left[-c \min \left(\frac{t^{2}}{\sigma_x^4\|a\|_{2}^{2}}, \frac{t}{\sigma_x^2\|a\|_{\infty}}\right)\right]
$$
where $c>0$ is an absolute constant.
\end{lemma}

\begin{lemma}[\textbf{Concentration of regularized truncated empirical covariance, Lemma 21 in \cite{tsigler2020benign}}]\label{con_rec}
Suppose $\Z =[\z_1,\z_2,\dots,\z_p]\in \mathbb{R}^{n \times p}$ is a matrix with independent isotropic sub-gaussian rows with norm $\sigma$. Consider $\SIG=\operatorname{diag}\left(\lambda_{1}, \ldots, \lambda_{p}\right)$ for some positive non-increasing sequence $\left\{\lambda_{i}\right\}_{i=1}^{p}$.

Denote $\A_{k}=\lambda \I_{n}+\sum_{i>k} \lambda_{i} \z_{i} \z_{i}^{\top}$ for some $\lambda \geq 0$. Suppose that it is known that for some $\delta, L>0$ independent of $n,p$ and some $k<n$ with probability at least $1-\delta$, the condition number of the matrix $\A_{k}$ is at most L. Then, for some absolute constant $c$ with probability at least $1-\delta-2 \exp (-c t)$
$$
\frac{(n-t \sigma^{2})}{L}\lambda_{k+1}\rho_{k}(\SIG;\lambda) \leq \mu_{n}\left(\A_{k}\right) \leq \mu_{1}\left(\A_{k}\right) \leq {\left(n+t \sigma^{2}\right) L\lambda_{k+1}}\rho_{k}(\SIG;\lambda)
$$

\end{lemma}

\begin{lemma}[\textbf{Concentration of leave-one-out empirical covariance}]\label{con_loo}
Under the same notations and assumptions in Lemma \ref{con_rec},
denote $\A_{-t}:=\lambda \I_{n}+\sum_{i\neq t} \lambda_{i} \z_{i} \z_{i}^{\top}$ for some $\lambda \geq 0$. Then for any $t \leq k \leq n$ such that the condition number of $\A_k$ is bounded by
$L$, we have
$$
\frac{(n-t \sigma^{2})}{L}\lambda_{k+1}\rho_{k}(\SIG;\lambda)\leq\mu_n(\A_{-t})\leq\mu_1(\A_{-t})\leq {\left(n+t \sigma^{2}\right) L\lambda_{1}}\rho_{0}(\SIG;\lambda)
$$

\end{lemma}
\begin{proof}
    The lemma follows by Lemma \ref{con_rec} and the observations of $\mu_1(\A_{-t})\leq \mu_1(\A_0))$ and $\A_{-t} \succeq \A_{k}$.
\end{proof}

\begin{lemma}[\textbf{Concentration of matrix with independent sub-gaussian rows, Theorem 5.39 in \cite{vershynin2010introduction}}] \label{con_subg}
Let $\X$ be an $n \times k$ matrix (with $n>k$) whose rows $\x_{i}$ are independent sub-gaussian isotropic random vectors in $\mathbb{R}^{k}$. Then for every $t \geq 0$ such that $\sqrt{n}-C\sqrt{k}-t>0$ for some constant $C>0$, we have with probability at least $1-2 \exp \left(-c t^{2}\right)$ that
$$
\sqrt{n}-C \sqrt{k}-t \leq s_{\min }(\X) \leq s_{\max }(\X) \leq \sqrt{n}+C \sqrt{k}+t
$$
Here $s_{\min}$ and $s_{\max}$ denotes the minimum and maximum singular values and $C, c>0$ are some constants depend only on the sub-gaussian norm of the rows.
\end{lemma}

\begin{lemma}[\textbf{Concentration of the sum of squared norms, Lemma 17 in \cite{tsigler2020benign}}]\label{ssn}
Suppose $\Z \in \mathbb{R}^{n \times p}$ is a matrix with independent isotropic sub-gaussian rows with norm $\sigma$. Consider $\SIG=\operatorname{diag}\left(\lambda_{1}, \ldots, \lambda_{p}\right)$ for some positive non-decreasing sequence $\left\{\lambda_{i}\right\}_{i=1}^{p}$. Then for some absolute constant $c$ and any $t \in(0, n)$ with probability at least $1-2 \exp (-c t)$
$$
\left(n-t \sigma^{2}\right) \sum_{i>k} \lambda_{i} \leq \sum_{i=1}^{n}\left\|\SIG_{k: \infty}^{1 / 2} \Z_{i, k: \infty}\right\|^{2} \leq\left(n+t \sigma^{2}\right) \sum_{i>k} \lambda_{i}
$$
\end{lemma}
\begin{lemma}[\textbf{Applications of Hanson-Wright inequality as done in \cite{muthukumar2020class}}]\label{hw} Let $\mathbf{\varepsilon}$ be a random vector composed of $n$ i.i.d. zero-mean sub-gaussian variables with norm $1$. Then, 

1. there exists universal constant $c>0$ such that for any fixed positive semi-definite matrix $\mathbf{\A}$, with probability $1-2\exp(-\sqrt{n})$, we have
$$
\left|\mathbf{\varepsilon}^{\top} \mathbf{\A} \mathbf{\varepsilon}-\mathbb{E}\left[\mathbf{\varepsilon}^{\top} \mathbf{\A} \mathbf{\varepsilon}\right]\right|\leq c\|\A\|n^{\frac{3}{4}}.
$$

2. there exists some universal constant $C>0$ such that with probability at least $1-\frac{1}{n}$
$$
\varepsilon^{\top} \A \varepsilon \leq C   \operatorname{tr}(\A)\log n.
$$
\end{lemma}

\begin{lemma}[\textbf{Operator norm bound of matrix with sub-gaussian rows \citep{tsigler2020benign}}]\label{op-bound}
 Suppose $\left\{\z_{i}\right\}_{i=1}^{n}$ is a sequence of independent sub-gaussian vectors in $\mathbb{R}^{p}$ with $\left\|\z_{i}\right\| \leq \sigma$. Consider $\SIG=\operatorname{diag}\left(\lambda_{1}, \ldots, \lambda_{p}\right)$ for some positive non-decreasing sequence $\left\{\lambda_{i}\right\}_{i=1}^{p} .$ Denote $\X$ to be the matrix with rows $\SIG^{1/2}\z_i $. 
Then for some absolute constant $c$, for any $t>0$ with probability at least $1-4 e^{-t / c}$
$$
\|\X\| \leq c \sigma \sqrt{\lambda_1(t+n)+\sum_{j=1}^p\lambda_j}.
$$
\end{lemma}

\section{Proofs of 
Regression Results}\label{main_proof_reg}
In this section, we will include essential lemmas in \ref{lemm_reg} to prove the main theorems for regression analysis in the sections \ref{prof_thm1_sec} and \ref{prof_thm2_sec}. Then, we will use these theorems to prove the propositions and corollaries in sections \ref{prof_prop1_reg} and \ref{pf_reg_cor}, respectively.
\subsection{Regression Lemmas}\label{lemm_reg}

\begin{lemma}[\textbf{Sharpened bias of ridge regression, extension of \cite{tsigler2020benign}}]\label{lem_bias}
    \begin{align}
        \frac{\rm{Bias}}{C_xL_1^4}&\lesssim 
        {{\left\|\mathbf{P}^{\SIG}_{k_1+1:p}{\theta}^{*}\right\|_{\SIG}^{2}
        +\left\|\mathbf{P}^{\SIG}_{1:k_1}{\theta}^{*}\right\|_{\SIG^{-1}}^{2}\frac{\rho_{k_1}^2(\SIG;n)}{{(\lambda_{k_1+1})^{-2}}+(\lambda_{1})^{-2}\rho_{k_1}^2(\SIG;n)}}}
    \end{align}
\end{lemma}
\begin{remark}
The reason we modify the bound from \cite{tsigler2020benign} is twofold: 1. we consider non-diagonal covariance matrix $\SIG$. This is because even if the original data covariance is diagonal, the equivalent spectrum might become non-diagonal after the data augmentation. Therefore, we modify the bound so that the eigenspaces of the data covariance matrix do not have to be aligned with the standard basis. 2. As we show in our work, some augmentations, e.g. random mask, have the effect of making the equivalent data spectrum isotropic. However, in this case, the bias bound in \cite{tsigler2020benign}, as shown below, can be vacuous as being almost the same as the null estimator so we modify the bound to remedy the case. 
\begin{align*}
        \rm{Bias~bound}~\asymp~ &{{\left\|\mathbf{P}^{\SIG}_{k_1+1:p}{\theta}^{*}\right\|_{\SIG}^{2}
        +\left\|\mathbf{P}^{\SIG}_{1:k_1}{\theta}^{*}\right\|_{\SIG^{-1}}^{2}{\lambda_{k_1+1}^2\rho_{k_1}^2(\SIG;n)}}}\\
        ~=~&{{\left\|\mathbf{P}^{\SIG}_{k_1+1:p}{\theta}^{*}\right\|^{2}
        +\left\|\mathbf{P}^{\SIG}_{1:k_1}{\theta}^{*}\right\|^{2}\frac{p-k_1}{n}}}
        \gtrsim \|\T^*\|_2^2,
\end{align*}
\end{remark}
\begin{proof}
    This lemma is a modification to  Theorem 1 in \cite{tsigler2020benign}, where we only change slightly in the estimation of the lower tail of the bias. For self-containment, we illustrate where we make the change.
    Consider the diagonalization $\SIG=\V\D\V^\top$. Let $\V_1,~\V_2$ be the matrices with columns consisting of the top $k$ eigenvectors of $\SIG$ and the remaining eigenvectors, respectively. Note that we have $\V=[\V_1,\V_2]$, $\boldsymbol{P}^{\SIG}_{1:k-1} = \V_1\V_1^\top,$ and $\boldsymbol{P}^{\SIG}_{k:\infty} = \V_2\V_2^\top$. Moreover, we have $\V_1 \V_1^\top + \V_2 \V_2^\top = \V\V^\top = \I_p$. Now, for the ridge estimator $\hat{\theta}= (\X^\top\X+\lambda \I_p)^{-1}\X^\top\y$, apply Lemma \ref{good_eq} with $\V=\V_1$ to obtain
    \begin{align}
        (\I_k + \V_1^\top\X^\top\mathcal{A}_k(\SIG;\lambda)^{-1}\X\V_1)\V_1^\top\hat{\theta} = \V_1^\top\X^\top\mathcal{A}_k(\SIG;\lambda)^{-1}\y,
    \end{align}
    where $\mathcal{A}_k(\SIG;\lambda) := \lambda \I_p  + \X\V_2\V_2^\top\X^\top$. As there will be no ambiguity of which covariance matrix the residual spectrum corresponds to, we will just write $\A_k$ from now on.
    
    To bound the bias, we split it into 
        \begin{align}
        \text{Bias}~{\leq}~2\|\V_1\V_1^\top(\mathbb{E}_{\varepsilon}[\Th]-\T^*)\|_{\SIG}^2 + 2\|\V_2\V_2^\top(\mathbb{E}_{\varepsilon}[\Th]-\T^*)\|_{\SIG}^2,
    \end{align}
    where the expectations are over the noise $\varepsilon$.
    Observe that the averaged estimator is $\mathbb{E}_{\varepsilon}[\Th]=(\X^\top\X+\lambda \I_p)^{-1}\X^\top\y$, so we can
    apply Lemma \ref{good_eq} with $\Th$ and $\y$ replaced by $\mathbb{E}_{\varepsilon}[\Th]$ and $\X\T^*$, respectively. As a result, we can write
    \begin{align*}
        (\I_k + \V_1^\top\X^\top\A_k^{-1}\X\V_1)\V_1^\top\mathbb{E}_{\varepsilon}[\Th]=~&\V_1^\top\X^\top\A_k^{-1}\X\T^*\\
        =~ &\V_1^\top\X^\top\A_k^{-1}\X(\V_1\V_1^\top + \V_2\V_2^\top)\T^*.
    \end{align*}
    Now, subtracting $\V_1^\top \T^* +\V_1^\top\X^\top\A_k^{-1}\X\V_1\V_1^\top\T^*  $ from both sides of the above equation followed by a left multiplication of $\V_1$ gives
    \begin{align*}
        &\V_1\V_1^\top(\mathbb{E}_{\varepsilon}\Th-\T^*) + \V_1\V_1^\top\X^\top\A_k^{-1}\X\V_1\V_1^\top(\mathbb{E}_{\varepsilon}\Th-\T^*) \\
        &=\V_1\V_1^\top\X^\top\A_k^{-1}\X\V_2\V_2^\top\T^*-\V_1\V_1^\top\T^*,
    \end{align*}
    where we use the identity $\I_p = \V_1\V_1^\top+\V_2\V_2^\top$.
    
    Now multiply both sides with $(\mathbb{E}_{\varepsilon}\Th-\T^*)^\top$, the R.H.S. is
    \begin{align}\label{rhs}
        = ~&~(\mathbb{E}_{\varepsilon}\Th-\T^*)^\top\V_1\V_1^\top\SIG^{1/2}\SIG^{-1/2}\X^\top\A_k^{-1}\X\V_2\V_2^\top\T^*-
        (\mathbb{E}_{\varepsilon}\Th-\T^*)^\top\V_1\V_1^\top\SIG^{1/2}\SIG^{-1/2}\T^*\nonumber\\
        \leq ~ & ~\|\V_1\V_1^\top(\mathbb{E}_{\varepsilon}\Th-\T^*)\|_{\SIG}\mu_n(\A_k)^{-1}\sqrt{\mu_1\left(\V_1\V_1^\top\SIG^{-1/2}\X^\top
        \X\SIG^{-1/2}\V_1\V_1^\top\right)}\|\X\V_2\V_2^\top\T^*\| \nonumber\\
        +~& \|\V_1\V_1^\top(\mathbb{E}_{\varepsilon}\Th-\T^*)\|_{\SIG}\|\V_1\V_1^\top\T^*\|_{\SIG^{-1}}.
    \end{align}
    Note that in the last term of the inequality, we have use the fact that \begin{align*}
        (\mathbb{E}_{\varepsilon}\Th-\T^*)^\top\V_1\V_1^\top\SIG^{1/2}\SIG^{-1/2}\T^*&=~(\mathbb{E}_{\varepsilon}\Th-\T^*)^\top\V_1\V_1^\top\SIG^{1/2}\SIG^{-1/2}(\V_1\V_1^\top +\V_2\V_2^\top)\T^*\\
        &=~ (\mathbb{E}_{\varepsilon}\Th-\T^*)^\top\V_1\V_1^\top\SIG^{1/2}\SIG^{-1/2}\V_1\V_1^\top\T^*.
    \end{align*}
    On the other hand, the L.H.S. is 
    \begin{align}\label{lhs_1}
        &~\geq \lambda_1^{-1}\|\V_1\V_1^\top(\mathbb{E}_{\varepsilon}\Th-\T^*)\|_{\SIG}^2+ (\mathbb{E}_{\varepsilon}\Th-\T^*)^\top\V_1\V_1^\top\X^\top\A_k^{-1}\X\V_1\V_1^\top(\mathbb{E}_{\varepsilon}\Th-\T^*),
    \end{align} in which the second term is 
    \begin{align}\label{lhs_2}
        = &~ (\mathbb{E}_{\varepsilon}\Th-\T^*)^\top\V_1\V_1^\top\SIG^{1/2} \V_1\V_1^\top\SIG^{-1/2}\X^\top\A_k^{-1}\X\SIG^{-1/2}\V_1\V_1^\top\SIG^{1/2}\V_1\V_1^\top(\mathbf{E}_{\varepsilon}\Th-\T^*)\nonumber\\
        ~\geq & ~ \|\V_1\V_1^\top(\mathbb{E}_{\varepsilon}\Th-\T^*)\|^2_{\SIG}\|\V_1\V_1^\top\SIG^{-1/2}\X^\top\A_k^{-1}\X\SIG^{-1/2}\V_1\V_1^\top\|\nonumber\\
        ~\geq & ~ \|\V_1\V_1^\top(\mathbb{E}_{\varepsilon}\Th-\T^*)\|^2_{\SIG}\mu_k(\V_1^\top\SIG^{-1/2}\X^\top\A_k^{-1}\X\SIG^{-1/2}\V_1)\nonumber\\
        ~\geq & ~ \|\V_1\V_1^\top(\mathbb{E}_{\varepsilon}\Th-\T^*)\|^2_{\SIG}\mu_1(\A_k)^{-1}\mu_k(\V_1^\top\SIG^{-1/2}\X^\top\X\SIG^{-1/2}\V_1).
    \end{align}
    Therefore, combining e.q. (\ref{rhs}), (\ref{lhs_1}) and (\ref{lhs_2}), we have 
    \begin{align*}
        &\|\V_1\V_1^\top(\mathbb{E}_{\varepsilon}\Th-\T^*)\|_{\SIG} \\
        &\leq   \frac{\mu_n^{-1}(\A_k)\sqrt{\mu_1\left(\V_1^\top\SIG^{-1/2}\X^\top\X\SIG^{-1/2}\V_1\right)}\|\X\V_2\V_2^\top\T^*\| +
        \|\V_1\V_1^\top\T^*\|_{\SIG^{-1}}}{\lambda_1^{-1} + \mu_1^{-1}(\A_k)\mu_k(\V_1^\top\SIG^{-1/2}\X^\top\X\SIG^{-1/2}\V_1)}. 
    \end{align*}
   
  Now, we turn to bound $\|\V_2\V_2^\top(\mathbf{E}_{\varepsilon}\Th-\T^*)\|_{\SIG}^2$. The proof follows the same step as \cite{tsigler2020benign} except we use projection matrices to accommodate for the non-diagonal covariance:
  \begin{align*}
    \|\V_2\V_2^\top(\mathbf{E}_{\varepsilon}\Th-\T^*)\|_{\SIG}^2 &\lesssim \underbrace{\|\boldsymbol{P}^{\SIG}_{k:\infty}\T^*\|_{\SIG}^2}_{T_1} +     \underbrace{\|\V_2\V_2^\top\X^\top(\X\X^\top+\lambda\I_n)^{-1}\X\V_2\V_2^\top\T^*\|_{\SIG}^2}_{T_2} \\
    &+\underbrace{\|\V_2\V_2^\top\X^\top(\X\X^\top+\lambda\I_n)^{-1}\X\V_1\V_1^\top\T^*\|_{\SIG}^2}_{T_3}
  \end{align*}
  $T_2$ is bounded by 
  \begin{align}\label{sec_t}
      \mu_n^{-2}(\A_k)\|\X\V_2\V_2^\top\SIG\V_2\V_2^\top\X^\top\|\|\X\V_2\V_2^\top\T^*\|_{\SIG}^2.
  \end{align}
  For $T_3$ on the other hand, recall $\X\X^\top + \lambda\I_p = \X\V_1\V_1^\top\X^\top + \A_k$. Then by the \newline Sherman–Morrison–Woodbury formula,
  we have
  \begin{align*}
      &(\X\X^\top+\lambda\I_p)^{-1}\X\V_1\\
      =& \left(\A_k^{-1}-\A_k^{-1}\X\V_1(\I_k + \V_1^\top\X^\top\A_k^{-1}\X\V_1)^{-1}\V_1^\top\X^\top\A_k^{-1}\right)\X\V_1\\
      =&~ \A_k^{-1}\X\V_1(\I_k + \V_1^\top\X^\top\A_k^{-1}\X\V_1)^{-1}.
  \end{align*}
  Therefore,
  \begin{align*}
      &\|\V_2\V_2^\top\X^\top(\X\X^\top+\lambda\I_p)^{-1}\X\V_1\V_1^\top\T^*\|_{\SIG}^2\\
      \leq ~&~\mu_n^{-2}(\A_k)\|\X\V_2\V_2^\top\SIG\V_2\V_2^\top\X^\top\|\|\X\V_1(\I_k + \V_1^\top\X^\top\A_k^{-1}\X\V_1)^{-1}\V_1^\top\T^*\|_2^2,
  \end{align*}
  where
  \begin{align*}
      &\X\V_1(\I_k + \V_1^\top\X^\top\A_k^{-1}\X\V_1)^{-1}\V_1^\top\T^*\\
      ~  \overset{(a)}{=}&~
      \X\V_1(\V_1^\top\SIG^{-1/2})(\SIG^{1/2}\V_1)(\I_k + \V_1^\top\X^\top\A_k^{-1}\X\V_1)^{-1}(\V_1^\top\SIG^{1/2})(\SIG^{-1/2}\V_1)\V_1^\top\T^*\\
      ~  \overset{(b)}{=}&~
      \X\SIG^{-1/2}(\SIG^{1/2}\V_1)(\I_k + \V_1^\top\X^\top\A_k^{-1}\X\V_1)^{-1}(\V_1^\top\SIG^{1/2})(\SIG^{-1/2}\V_1)\V_1^\top\T^*\\
      ~  \overset{(c)}{=}&~
      \X\SIG^{-1/2}\V_1(\V_1^\top\SIG^{-1}\V_1 + \V_1^\top\SIG^{-1/2}\X^\top\A_k^{-1}\X\SIG^{-1/2}\V_1)^{-1}\SIG^{-1/2}\V_1\V_1^\top\T^*,
  \end{align*}
  where (a) follows from $\V_1^\top\V_1=\I_k$, (b) from $$\X\V_1(\V_1^\top\SIG^{-1/2})(\SIG^{1/2}\V_1)=\X(\V_1\V_1^\top + \V_2\V_2^\top)\SIG^{-1/2}\SIG^{1/2}\V_1=\X\SIG^{-1/2}\SIG^{1/2}\V_1$$
  as $\V_1^\top\V_2=0$ and $\V_1\V_1^\top + \V_2\V_2^\top=\I_p$, and (c) follows from the facts
  \begin{align*}
      \X\SIG^{-1/2}\SIG^{1/2}\V_1&=\X\SIG^{-1/2}\V_1\left(\V_1^\top\SIG^{1/2}\V_1\right)\\
      (\V_1^\top\SIG^{1/2}\V_1)^{-1}&=\V_1^\top\SIG^{-1/2}\V_1.
  \end{align*}
  Therefore, we have
  \begin{align*}
    &\|\X\V_1(\I + \V_1^\top\X^\top\A_k^{-1}\X\V_1)^{-1}\V_1^\top\T^*\|_2^2\\
    ~  \leq &~\frac{\mu_1\left(\V_1^\top\SIG^{-1/2}\X^\top\X\SIG^{-1/2}\V_1\right)}{\lambda_{1}^{-2} +\mu_1^{-2}(\A_k)\mu^2_k(\V_1^\top\SIG^{-1/2}\X^\top\X\SIG^{-1/2}\V_1) }\|\boldsymbol{P}^{\SIG}_{1:k-1}\T^*\|_{\SIG^{-1}}.
  \end{align*}
  Now, adding all the terms above together, the bias is
    \begin{align*}
        \text{Bias} ~\lesssim&~ 
         {  \frac{\mu_n^{-2}(\A_k)\mu_1\left(\V_1^\top\SIG^{-1/2}\X^\top\X\SIG^{-1/2}\V_1\right)\|\X\V_2\V_2^\top\T^*\|_2^2 +
        \|\V_1\V_1^\top\T^*\|^2_{\SIG^{-1}}}{\lambda_1^{-2} + \mu_1^{-2}(\A_k)\mu^2_k(\V_1^\top\SIG^{-1/2}\X^\top\X\SIG^{-1/2}\V_1)}} \\
        ~ + & ~ { \|\X\V_2\V_2^\top\SIG\V_2\V_2^\top\X^\top\|\frac{\mu_n^{-2}(\A_k)\mu_1\left(\V_1^\top\SIG^{-1/2}\X^\top\X\SIG^{-1/2}\V_1\right)\|\V_1\V_1^\top\T^*\|^2_{\SIG^{-1}}}{\lambda_1^{-2} +  \mu_1^{-2}(\A_k)\mu_k\left(\V_1^\top\SIG^{-1/2}\X^\top\X\SIG^{-1/2}\V_1\right)^2}}\\
        ~ + & ~\|\X\V_2\V_2^\top\SIG\V_2\V_2^\top\X^\top\|\mu_n^{-2}(\A_k)\|\X\V_2\V_2^\top\T^*\|_{\SIG}^2 + \|\boldsymbol{P}^{\SIG}_{k:\infty}\T^*\|_{\SIG}^2,
    \end{align*}
    where for the diagonal covariance $\SIG$, the first two terms are sharpened with additional $\lambda_1^{-2}$ in the denominators as compared to \cite{tsigler2020benign}.
    As in \cite{tsigler2020benign}, these terms can be bounded by concentration bounds:
    $\mu_i\left(\V_1^\top\SIG^{-1/2}\X^\top\X\SIG^{-1/2}\V_1\right)$ by Lemma \ref{con_subg}, $\mu_j(\A_k)$ by Lemma \ref{con_rec}, $
    \|\X\V_2\V_2\|_2^2$ and $\|\X\V_2\V_2^\top\SIG\V_2\V_2^\top\X^\top\|$ by Lemma \ref{ssn}. The details can be found in the proof of MSE bound of \cite{tsigler2020benign}.
\end{proof}

\begin{lemma}[\textbf{Variance bound of ridge regression for non-diagonal covariance data \citep{tsigler2020benign}}]\label{lem_var}
    Consider the regression task with the model setting in Section 3 where the input variable $\x$ possibly has non-diagonal covariance $\SIG$ with eigenvalues $\lambda_1\geq\lambda_2\dots\lambda_p$. Given a ridge estimator $\hat{\theta}= (\X^\top\X+\lambda \I)^{-1}\X^\top\y$ and $\lambda \geq 0$, if we know that for some $k_2$, the condition number of $\mathcal{A}_{k_2}(\mathbf{\X};\lambda)$ is bounded by $L_2$ with probability $1-\delta$, where $\delta<1-\exp(-n/c_x^2)$, then there exists some constant $\tilde{C}_x$ depending only on $\sigma_x$ such that with probability at least $1-\delta-n^{-1}$,
    \begin{align}
        \frac{\mathrm{Variance}}{\sigma_{\varepsilon}^{2} L_2^2\Tilde{C}_x}&\lesssim 
        {{\left( \frac{k_2}{n}+ \frac{n}{R_{k_2}(\SIG;n)}\right)\log n}}.
    \end{align}
\end{lemma}

\begin{lemma}[\textbf{Generalization bound of ridge regression for non-diagonal covariance data, extension of \cite{tsigler2020benign}}]\label{lem_gen}
    Consider the regression task with the model setting in Section 3 where the input variable $\x$ has possibly non-diagonal covariance $\SIG$ with eigenvalues $\lambda_1\geq\lambda_2\dots$.  Then, given a ridge regression estimator $\hat{\theta}= (\X^\top\X+\lambda \I)^{-1}\X^\top\y$ and $\lambda \geq 0$, suppose we know that for some $k_1$ and $k_2$, the condition numbers of $\mathcal{A}_{k_1}(\mathbf{\X};\lambda)$ and $\mathcal{A}_{k_2}(\mathbf{\X};\lambda)$ are bounded by $L_1$ and $L_2$ with probability $1-\delta$, where $\delta<1-\exp(-n/c_x^2)$, then there exists some constants $C_x, \tilde{C}_x$ depending only on $\sigma_x$ such that with probability at least $1-n^{-1}$,
    \begin{align}
        \mathrm{MSE}&\lesssim 
        \underbrace{{\color{blue}C_xL_1^4\left(\left\|\mathbf{P}^{\SIG}_{k_1+1:p}{\theta}^{*}\right\|_{\SIG}^{2}
        +\left\|\mathbf{P}^{\SIG}_{1:k_1}{\theta}^{*}\right\|_{\SIG^{-1}}^{2}\frac{\rho_{k_1}^2(\SIG;n)}{{(\lambda_{k_1+1})^{-2}}+(\lambda_{1})^{-2}\rho_{k_1}^2(\SIG;n)}\right)}}_{\mathrm{Bias}} \nonumber\\
        &+\underbrace{{\color{red} \sigma_{\varepsilon}^{2} L_2^2\Tilde{C}_x\left(\frac{k_2}{n}+ \frac{n}{R_{k_2}(\SIG;n)}\right)\log n}}_{\mathrm{Variance}}
    \end{align}
\end{lemma}
\begin{proof}
    The statement is a direct combination of Lemma \ref{lem_bias}, \ref{lem_var} and the bias-variance decomposition of MSE from \cite{tsigler2020benign}.
\end{proof}

\begin{lemma}[\textbf{Bounds on the approximation error for regression}]\label{lem:approx}
Denote $$\HT:=(\X^\top\X+n\cg(\X))^{-1}\X^\top\y ,~~\BT:=(\X^\top\X+n\mathbb{E}_\x\cg(\x))^{-1}\X^\top\y,$$ and $\kappa$ the condition number of $\SIG_{\text{aug}}$. Assume for some constant $c<1$ that $$\Delta_G:=\|\mathbb{E}_{\x}[\cg(\x)]^{-\frac{1}{2}}\cg(\X)\mathbb{E}_{\x}[\cg(\x)]^{-\frac{1}{2}}-\I\|\leq c.$$
Then the approximation error is bounded by,
\begin{align*}
    \|\hat{\T}_{\text{aug}}-\BT\|_\SIG \lesssim \kappa^{\frac{1}{2}}\Delta_G \left(\|\T^*\|_{\SIG} +  \sqrt{\mathrm{Bias}(\BT)} + \sqrt{\mathrm{Variance}(\BT)}\right).
\end{align*}
\begin{proof}
    For ease of notation, we denote $\D={\cg}$, $\bar{\D}=\mathbb{E}_{\x}[\cg(\x)]$, and $\boldsymbol{\Delta}=\bar{\D}^{-\frac{1}{2}}\D\bar{\D}^{-\frac{1}{2}} - \I$. Then
    \begin{align}\label{sect}
        \|\hat{\T}_{\text{aug}}-\BT\|_\SIG 
        &=\|(\X^\top\X + n{\D})^{-1}\X^\top\y - (\X^\top\X + n\bar{\D})^{-1}\X^\top\y \|_\SIG\nonumber\\
        &= \|(\X^\top\X + n{\D})^{-1}(\X^\top\X + n\bar{\D} - \X^\top\X - n\D)(\X^\top\X + n{\bar{\D}})^{-1}\X^\top\y\|_\SIG\nonumber\\
        &=n\|\SIG^{\frac{1}{2}}\bar{\D}^{-\frac{1}{2}}\bar{\D}^{\frac{1}{2}}(\X^\top\X + n{\D})^{-1}\bar{\D}^{\frac{1}{2}}\boldsymbol{\Delta}\bar{\D}^{\frac{1}{2}}\Tb_{\text{aug}}\|_2,\nonumber\\
        &\lesssim n\|\SIG^{\frac{1}{2}}\bar{\D}^{-\frac{1}{2}}\|\|\bar{\D}^{\frac{1}{2}}(\X^\top\X + n{\D})^{-1}\bar{\D}^{\frac{1}{2}}\|\|\boldsymbol{\Delta}
        \|\|\bar{\D}^{\frac{1}{2}}\SIG^{-\frac{1}{2}}\|\|\Tb_{\text{aug}}\|_2\nonumber\\
        &\lesssim 
        n\kappa^{\frac{1}{2}}\Delta_G\|\Tb_{\text{aug}}\|_{\SIG}\|\bar{\D}^{\frac{1}{2}}(\X^\top\X + n{\D})^{-1}\bar{\D}^{\frac{1}{2}}\|
    \end{align}
    By (\ref{new_bias}), $\|\Tb_{\text{aug}}\|_{\SIG}$ can be bounded as,
    \begin{align*}
     \|\BT\|_\SIG\leq
     &~\|\T^*\|_{\SIG}+\|\BT-\T^*\|_\SIG\lesssim \|\T^*\|_{\SIG} + \sqrt{\mathrm{Bias}(\BT)} + \sqrt{\mathrm{Variance}(\BT)}.
    \end{align*}
    It remains to bound $\|\bar{\D}^{\frac{1}{2}}(\X^\top\X + n{\D})^{-1}\bar{\D}^{\frac{1}{2}}\|$.
    
    Now, observe
    \begin{align*}
        &\|\bar{\D}^{\frac{1}{2}}(\X^\top\X + n{\D})^{-1}\bar{\D}^{\frac{1}{2}}\| = \left(\mu_p\left(\bar{\D}^{\frac{1}{2}}(\X^\top\X + n{\D})^{-1}\bar{\D}^{\frac{1}{2}}\right)^{-1}\right)^{-1}
        \\
        =&~~ 
        \left(\mu_p\left(\bar{\D}^{-\frac{1}{2}}(\X^\top\X + n{{\D}})\bar{\D}^{-\frac{1}{2}}\right)\right)^{-1}\\
        \leq &\left(\mu_p\left(\bar{\D}^{-\frac{1}{2}}(\X^\top\X + n{\bar{\D}})\bar{\D}^{-\frac{1}{2}}\right)-\|\bar{\D}^{-\frac{1}{2}}(\X^\top\X+n\bar{\D}-\X^\top\X-n\D)\bar{\D}^{-\frac{1}{2}}\|\right)^{-1}.
    \end{align*}
    However,
    \begin{align*}
        \left(\bar{\D}^{\frac{1}{2}}(\X^\top\X + n{\bar{\D}})^{-1}\bar{\D}^{\frac{1}{2}}\right)^{-1} = (\tilde{\X}^\top\tilde{\X} + n\I),
    \end{align*}
    where $\tilde{\X}$ has sub-gaussian rows with covariance $\SIG_{\text{aug}}$. Hence, the first term is at least $n$, while the second term is just $n{\Delta}_G$ by definition. So by the assumption that ${\Delta}_G < c$ for some $c<1$, we have,
    \begin{align*}
        \|\bar{\D}^{\frac{1}{2}}(\X^\top\X + n{\D})^{-1}\bar{\D}^{\frac{1}{2}}\| \lesssim \frac{1}{n},
    \end{align*}
    and finally we have,
    \begin{align*}
        \|\hat{\T}_{\text{aug}}-\BT\|_\SIG \lesssim \kappa^{\frac{1}{2}}\Delta_G \left(\|\T^*\|_{\SIG} +  \sqrt{\mathrm{Bias}(\BT)} + \sqrt{\mathrm{Variance}(\BT)}\right).
    \end{align*}
\end{proof}
\end{lemma}

\subsection{Proof of Theorem \ref{gen_bound}}\label{prof_thm1_sec}
    \begin{reptheorem}{gen_bound}[\textbf{Bounds of Mean-Squared Error for Regression}] 
    Consider an unbiased data augmentation $g$ and its corresponding estimator $\HT$. Recall the definition
    $${\Delta}_G := \|\mathbb{E}_{\x}[\cg(\x)]^{-\frac{1}{2}}\cg(\X)\mathbb{E}_{\x}[\cg(\x)]^{-\frac{1}{2}}-\I_p\|,$$ and let $\kappa$ be the condition number of $\SIG_{\text{aug}}$. Assume with probability $1-\delta'$, we have that for some integers $k_1$ and $k_2$, the condition numbers for the matrices $\mathcal{A}_{k_1}(\mathbf{\X_{\text{aug}}};n)$, $\mathcal{A}_{k_2}(\mathbf{\X_{\text{aug}}};n)$ are bounded by $L_1$ and $L_2$ respectively, and that ${\Delta}_G\leq c'$ for some constant $c' < 1$.
    Then there exist some constants $C_x, \tilde{C}_x$ depending only on $\sigma_x$ and $\sigma_\varepsilon$, such that, with probability $1- \delta'-4n^{-1}$, the testing mean-squared error is bounded by
    \begin{align*}
      \mathrm{MSE}~\lesssim &~  {{\color{blue}\mathrm{Bias}} + {\color{red}\mathrm{Variance}}+\color{olive}\mathrm{Approximation Error}},\\
      \frac{{\color{blue}\mathrm{Bias}}}{C_xL_1^4} ~\lesssim &~
        {{\left(\left\|\mathbf{P}^{\SIG_{\text{aug}}}_{k_1+1:p}{\theta}_{\text{aug}}^{*}\right\|_{\SIG_{\text{aug}}}^{2}
      +\left\|\mathbf{P}^{\SIG_{\text{aug}}}_{1:k_1}{\theta}_{\text{aug}}^{*}\right\|_{\SIG_{\text{aug}}^{-1}}^{2}\frac{(\rho^{\text{aug}}_{k_1})^2}{{(\lambda^{\text{aug}}_{k_1+1})^{-2}}+(\lambda^{\text{aug}}_{1})^{-2}(\rho^{\text{aug}}_{k_1})^2}\right)}}, \nonumber\\
        \frac{{\color{red}\mathrm{Variance}}}{\sigma_{\varepsilon}^{2} L_2^2\Tilde{C}_x}~\lesssim&~{{ \left(\frac{k_2}{n}+ \frac{n}{R^{\text{aug}}_k}\right)\log n}},~~{{\color{olive}\mathrm{Approx. Error}}}~\lesssim~\kappa^{\frac{1}{2}}{\Delta}_G \left(\|\T^*\|_{\SIG} +  \sqrt{\mathrm{Bias} + \mathrm{Variance}}\right)\nonumber,
    \end{align*}
    where $\rho^{\text{aug}}_{k}:=\rho_{k}(\SIG_{\text{aug}};n)$ and $R^{\text{aug}}_{k}:=R_{k}(\SIG_{\text{aug}};n)$.
\end{reptheorem}
\begin{proof}
\begin{align}\label{simmse}
\text{MSE} = \mathbb{E}_\x[(\x^\top(\hat{\T}_{\text{aug}}-\T^*))^2|\X,\varepsilon]=\|\hat{\T}_{\text{aug}}-\T^*\|_{\SIG}^2.
\end{align}

Because the possible dependency of $\cg(\X)$ on $\X$, we approximate the $\Th_{\text{aug}}$ with the estimator $\BT:=(\X^\top\X + n\mathbb{E}_{\x}[\cg(\x)])^{-1}\X^\top\y$. Now, by the triangle inequality, the MSE can be bounded as
\begin{align}
    \text{MSE} \leq  2\|\BT-\T^*\|^2_\SIG + 2\|\hat{\T}_{\text{aug}}-\BT\|^2_\SIG
\end{align}

We can bound the first term by using its connection to ridge regression:
\begin{align}
    &\hat{\T}_{\text{aug}}= (\X^\top\X + n \mathbb{E}_{\x}[\cg(\x)])^{-1}\X^\top\y\nonumber\\
    &=\mathbb{E}_{\x}[\cg(\x)]^{-1/2}(n\I_p + \mathbb{E}_{\x}[\cg(\x)]^{-1/2}\X^\top\X \mathbb{E}_{\x}[\cg(\x)]^{-1/2})^{-1}
    \mathbb{E}_{\x}[\cg(\x)]^{-1/2}\X^\top\y\nonumber\\
    &= \mathbb{E}_{\x}[\cg(\x)]^{-1/2}(n\I_p + \tilde{\mathbf{\X}}^\top\tilde{\mathbf{\X}})^{-1}
    \tilde{\mathbf{\X}}^\top\y ~~~~~(\tilde{\mathbf{\X}}:=\X\mathbb{E}_{\x}[\cg(\x)]^{-1/2})\nonumber\\
    &=\mathbb{E}_{\x}[\cg(\x)]^{-1/2}\hat{\theta}_{\text{ridge}}, ~~~~~(\hat{\T}_{\text{ridge}}:=(n\I_p + \tilde{\X}^\top\tilde{\mathbf{\X}})^{-1}\tilde{\mathbf{\X}}^\top\y).
\end{align}
So the MSE becomes $\|\hat{\T}_{\text{ridge}}-\mathbb{E}_{\x}[\cg(\x)]^{1/2}\T^*\|^2_{\mathbb{E}_{\x}[\cg(\x)]^{-1/2}\SIG\mathbb{E}_{\x}[\cg(\x)]^{-1/2}}$.
These observations have shown an approximate equivalence to a ridge estimator with data matrix $\tilde{\X}$, which has data covariance = $\mathbb{E}_{\x}[\cg(\x)]^{-1/2}\SIG\mathbb{E}_{\x}[\cg(\x)]^{-1/2}$, ridge intensity $\lambda=n$, and true model parameter $\mathbb{E}_{\x}[\cg(\x)]^{1/2}\T^*$. 
Hence, we can apply Lemma \ref{lem_gen} to bound $\|\BT-\T^*\|^2_\SIG$, where $\|\mathbb{E}_{\varepsilon}[\BT]-\T^*\|^2_\SIG$ and $\|\mathbb{E}_{\varepsilon}[\BT]-\BT\|^2_\SIG$ are exactly the bias and variance in Theorem \ref{gen_bound}, respectively.
Specifically, we have,
\begin{align}\label{new_bias}
    &\|\mathbb{E}_{\varepsilon}[\BT]-\T^*\|^2_\SIG\lesssim\nonumber\\
     &~{{C_xL_1^4\left(\left\|\mathbf{P}^{\SIG_{\text{aug}}}_{k_1+1:p}{\theta}_{\text{aug}}^{*}\right\|_{\SIG_{\text{aug}}}^{2}
    +\left\|\mathbf{P}^{\SIG_{\text{aug}}}_{1:k_1}{\theta}_{\text{aug}}^{*}\right\|_{\SIG_{\text{aug}}^{-1}}^{2}\frac{\rho_{k_1}^2(\SIG_{\text{aug}};n)}{{(\lambda^{\text{aug}}_{k_1+1})^{-2}}+(\lambda^{\text{aug}}_{1})^{-2}\rho_{k_1}^2(\SIG_{\text{aug}};n)}\right)}},\\
    &\|\mathbb{E}_{\varepsilon}[\BT]-\BT\|^2_\SIG
    \lesssim ~
        \sigma_{\varepsilon}^{2} tL_2^2\Tilde{C}_x\left(\frac{k_2}{n}+ \frac{n}{R_{k_2}(\SIG_{\text{aug}};n)}\right).
\end{align}

For the second error term $\|\hat{\T}_{\text{aug}}-\BT\|^2_\SIG$, we apply Lemma \ref{lem:approx}.

\end{proof}

\subsection{Proof of Theorem \ref{bias-thm1}}\label{prof_thm2_sec}
\begin{reptheorem}{bias-thm1}[\text{Bounds of MSE for Biased Estimator}]
Consider the estimator $\HT$ obtained by solving the aERM in (\ref{DAobj}).
Let $\mathrm{MSE}^o(\HT)$ denote the unbiased MSE bound in Eq. (\ref{un_bound}) of Theorem \ref{gen_bound}, $\Cb:=\mathbb{E}_\x[\cg(\x)]$, and 

$${\Delta}_G=\|n^{-1}\Cb^{-\frac{1}{2}}{\cg(\X)\Cb^{-\frac{1}{2}}}-\I_p\|.$$ Suppose the assumptions in Theorem \ref{gen_bound} hold for the mean augmentation $\mu(\x)$ and that ${\Delta}_G\leq c<1$.
Recall the definition of the mean augmentation covariance  $\bar{\SIG}:=\E_\x[(\mu_g(\x)-\E_{\x}[\x])(\mu_g(\x)-\E_\x[\x])^\top]$.Then with probability $1- \delta'-4n^{-1}$ we have,
 \begin{align*}
    &~\mathrm{MSE}(\HT)\lesssim R_1^2 \cdot
    \left(\sqrt{\mathrm{MSE}^o(\HT)} + R_2 \right)^2,
\end{align*}
where 
\begin{align*}
    R_1 &= 1 + \|\SIG^{\frac{1}{2}}\bar{\SIG}^{-\frac{1}{2}}-\I_p\|,\\
    R_2 &= \sqrt{{\|\bar{\SIG}\Cb^{-1}\|}}\left(1+\frac{{\Delta}_G}{1-c}\right)\left(\sqrt{{\Delta}_\xi}\|\T^*\|+\|\T^*\|_{\operatorname{Cov}_{\xi}}\right)\left(\sqrt{\frac{1}{\lambda_k^{\text{aug}}}}+\sqrt{\frac{\lambda_{k+1}^{\text{aug}}(1+\rho_k(\SIG_{\text{aug}};n))}{(\lambda_1^{\text{aug}}\rho_0(\SIG_{\text{aug}};n))^2}}\right).
\end{align*}

\end{reptheorem}

\begin{proof}
    \begin{align*}
    \mathrm{MSE}(\HT) = \|\HT-\T^*\|^2_{\SIG}\leq \left(\underbrace{\|\HT - \T^*\|_{ \bar{\SIG}}}_{L_1}
        + \underbrace{\left|\|\HT-\T^*\|_{\SIG} - \|\HT-\T^*\|_{\bar{\SIG}}\right|}_{L_2}\right)^2.
    \end{align*}
    Now we will bound $L_2$ and $L_1$ in a sequence. 
    For the $L_2$, denote $\Delta = \HT - \T^*$, then
    \begin{align*}
        &\left|\|\HT-\T^*\|_{\SIG} - \|\HT-\T^*\|_{\bar{\SIG}}\right|=\left|\sqrt{\Delta^\top\SIG\Delta}-\sqrt{\Delta^\top\bar{\SIG}\Delta}\right|\\
        = ~& \frac{\left|\Delta^\top(\SIG-\bar{\SIG})\Delta\right|}{\|\Delta\|_\SIG + \|\Delta\|_{\bar{\SIG}}} \leq \frac{\| \Delta^\top(\SIG^{\frac{1}{2}}-\bar{\SIG}^{\frac{1}{2}})\| \|(\SIG^{\frac{1}{2}}+\bar{\SIG}^{\frac{1}{2}})\Delta\|}{\|\Delta\|_\SIG + \|\Delta\|_{\bar{\SIG}}}\\
        \leq~ & \| \Delta^\top(\SIG^{\frac{1}{2}}-\bar{\SIG}^{\frac{1}{2}})\|\leq \|\Delta\|_{\bar{\SIG}}\|\SIG^{\frac{1}{2}}\bar{\SIG}^{-\frac{1}{2}}-\I_p\|=\|\HT - \T^*\|_{ \bar{\SIG}}\|\SIG^{\frac{1}{2}}\bar{\SIG}^{-\frac{1}{2}}-\I_p\|.
    \end{align*}
    Hence, it remains to bound $\|\HT - \T^*\|_{ \bar{\SIG}}$ because
    \begin{align}
        L_1 + L_2 \leq (1 + \|\SIG^{\frac{1}{2}}\bar{\SIG}^{-\frac{1}{2}}-\I_p\|)\|\HT - \T^*\|_{ \bar{\SIG}}\label{l1122}.
    \end{align}
    Now observe that $\|\HT - \T^*\|_{ \bar{\SIG}}$ is just like the test error of an estimator where the covariate has the distribution of $
    \mg(\x)$. However, recall the caveat that when $g$ is biased, there will be both a covariate shift and a misalignment of the observations in the estimator. Therefore, we have to take the latter into account. Specifically, recall that our observations $\y$ are, in fact, $\X\T^* + \mathbf{n}$. To match the covariate distribution $\mg(\x)$, we define $\tilde{\y}=\mu(\X)\T^* + \mathbf{n}$. Although we do not actually observe $\tilde{\y}$, we can bound the error between observing $\y$ and $\tilde{\y}$. Therefore, we denote $\tilde{\T}_{\text{aug}}:= (\mu(\X)^\top\mu(\X) + \cg(\X))^{-1}\mu(\X)^\top \tilde{\y}$. This is the biased estimator that uses the biased augmentation $g$ and also has an observation distribution that matches the covariate distribution. Then,
    \begin{align}
        \|\HT - \T^*\|_{ \bar{\SIG}}\lesssim \underbrace{\|\tilde{\T}_{\text{aug}}-\T^*\|_{\bar{\SIG}}}_{L_3}+ \underbrace{\|\HT-\tilde{\T}_{\text{aug}}\|_{\bar{\SIG}}}_{L_4}\label{l1+l2}.
    \end{align}
    Now, since $\tilde{\T}_{\text{aug}}$ has observations matching its covariate distribution $\mg(\x)$, we can apply Theorem \ref{gen_bound} to bound $L_3$:
    \begin{align}
        \|\tilde{\T}_{\text{aug}}-\T^*\|_{\bar{\SIG}}\leq \sqrt{\mathrm{MSE}^o},\label{l1212}
    \end{align}
    where $\mathrm{MSE}^o$ is the bound in E.q. (\ref{un_bound}).
    It remains to bound  $L_4$. Note that this error arises from the additive error between $\y$ and $\tilde{\y}$.  Recall $\Cb:=\mathbb{E}_\x[\cg(\x)]$, then,
    \begin{align*}
        \|\HT-\tilde{\T}_{\text{aug}}\|_{\bar{\SIG}} &= \|(\mu(\X)^\top\mu(\X) + \cg(\X))^{-1}\mu(\X)^\top(\y-\tilde{\y})\|_{\bar{\SIG}}\\
        &= \|\bar{\SIG}^{\frac{1}{2}}(\mu(\X)^\top\mu(\X) + \cg(\X))^{-1}\mu(\X)^\top(\y-\tilde{\y})\|\\
        &\leq  \underbrace{\|\bar{\SIG}^{\frac{1}{2}}(\mu(\X)^\top\mu(\X) + \cg(\X))^{-1}\mu(\X)^\top\|}_{L_5}\underbrace{\|(\y-\tilde{\y})\|}_{L_6}.
    \end{align*}
    We first bound $L_5$,
    \begin{align*}
     &\|\bar{\SIG}^{\frac{1}{2}}(\mu(\X)^\top\mu(\X) + \cg(\X))^{-1}\mu(\X)^\top\|\\
         \leq ~&
     \underbrace{\|\bar{\SIG}^{\frac{1}{2}}(\mu(\X)^\top\mu(\X) + n\Cb)^{-1}\mu(\X)^\top\|}_{L_7} \\
          +~& \underbrace{\|\bar{\SIG}^{\frac{1}{2}}(\mu(\X)^\top\mu(\X) + \cg(\X))^{-1}\mu(\X)^\top- \bar{\SIG}^{\frac{1}{2}}(\mu(\X)^\top\mu(\X) + n\Cb)^{-1}\mu(\X)^\top\|}_{L_8}.
    \end{align*}
    Observe that 
    \begin{align*}
    &L_7 = \|\bar{\SIG}^{\frac{1}{2}}(\mu(\X)^\top\mu(\X) + n\Cb)^{-1}\mu(\X)^\top\| = 
    \|\bar{\SIG}^{\frac{1}{2}}\Cb^{-\frac{1}{2}}(\tilde{\X}\tilde{\X}^\top + n\I_n)^{-1}\tilde{\X}\|\\
    \leq ~& 
    \underbrace{\|\bar{\SIG}^{\frac{1}{2}}\Cb^{-\frac{1}{2}}(\tilde{\X}\tilde{\X}^\top + n\I_n)^{-1}\tilde{\X}_{1:k}\|}_{L_9} + \underbrace{\|\bar{\SIG}^{\frac{1}{2}}\Cb^{-\frac{1}{2}}(\tilde{\X}\tilde{\X}^\top + n\I_n)^{-1}\tilde{\X}_{k+1:p}\|}_{L_{10}},
    \end{align*}
    where $\tilde{\X}$ has sub-gaussian rows with covariance ${\SIG}_{\text{aug}}$ as defined in E.q. (\ref{mod_spec_2}).
    
    Now, we bound $L_9$ and $L_{10}$.
    For convenience, denote $\A=\tilde{\X}\tilde{\X}^\top + n\I_n$ and $\A_k = \tilde{\X}_{k+1:p}\tilde{\X}_{k+1:p}^\top + n\I_n$.
    By Woodbury matrix identity, we have
    \begin{align*}
        \A^{-1}\tilde{\X}_{1:k}=\A_k^{-1}\tilde{\X}_{1:k}(\I_p + \tilde{\X}_{1:k}^\top\A_k^{-1}\tilde{\X}_{1:k})^{-1}.
    \end{align*}
    Hence, $L_9$ is bounded by
    \begin{align}\label{sss1}
       &\|\bar{\SIG}^{\frac{1}{2}}\Cb^{-\frac{1}{2}}(\tilde{\X}\tilde{\X}^\top + n\I_n)^{-1}\tilde{\X}_{1:k}\|= \|\bar{\SIG}^{\frac{1}{2}}\Cb^{-\frac{1}{2}}\A_k^{-1}\tilde{\X}_{1:k}(\I_p + \tilde{\X}_{1:k}^\top\A_k^{-1}\tilde{\X}_{1:k})^{-1}\|\nonumber\\
       \leq ~&
       \mu_n(\A_k)^{-1}\|\bar{\SIG}^{\frac{1}{2}}\Cb^{-\frac{1}{2}}\|\|\tilde{\X}_{1:k}(\I_p + \tilde{\X}_{1:k}^\top\A_k^{-1}\tilde{\X}_{1:k})^{-1}\|\nonumber\\
       =~&
       \mu_n(\A_k)^{-1}\|\bar{\SIG}^{\frac{1}{2}}\Cb^{-\frac{1}{2}}\|\|\tilde{\Z}_{1:k}(\SIG_{\text{aug,1:k}}^{-1} + \tilde{\Z}_{1:k}^\top\A_k^{-1}\tilde{\Z}_{1:k})^{-1}\SIG_{\text{aug, 1:k}}^{-\frac{1}{2}}\|\nonumber\\
       \leq ~&\mu_n(\A_k)^{-1}\|\bar{\SIG}^{\frac{1}{2}}\Cb^{-\frac{1}{2}}\|\|{\SIG}_{\text{aug,1:k}}^{-\frac{1}{2}}\|\|\tilde{\Z}_{1:k}(\SIG_{\text{aug,1:k}}^{-1} + \tilde{\Z}_{1:k}^\top\A_k^{-1}\tilde{\Z}_{1:k})^{-1}\|,
    \end{align}
    where $\tilde{\Z}$ has sub-gaussian rows with isotropic covariance $\I_p$.
    Now applying Lemma \ref{con_subg}, we have, with probability $1-5n^{-3}$,
    \begin{align*}
        \|\tilde{\Z}_{1:k}(\SIG_{\text{aug,1:k}}^{-1} + \tilde{\Z}_{1:k}^\top\A_k^{-1}\tilde{\Z}_{1:k})^{-1}\|~\lesssim~& \|\tilde{\Z}_{1:k}\|\mu_k^{-1}(\tilde{\Z}_{1:k}^\top\A_k^{-1}\tilde{\Z}_{1:k})\\
        ~\lesssim~ &\mu_1(\A_k)\frac{\sqrt{n}}{\mu_k^{-1}(\tilde{\Z}_{1:k}^\top\tilde{\Z}_{1:k})}~\lesssim~ \frac{\mu_1(\A_k)}{\sqrt{n}}.
    \end{align*}
    Combining the above and E.q. (\ref{sss1}) with Lemma \ref{con_rec}, 
    we have with probability $1-\delta-2n^{-3}$ that
    \begin{align}
        L_9=\|\bar{\SIG}^{\frac{1}{2}}\Cb^{-\frac{1}{2}}(\tilde{\X}\tilde{\X}^\top + n\I_n)^{-1}\tilde{\X}_{1:k}\|\lesssim \sqrt{\frac{\|\bar{\SIG}\Cb^{-1}\|}{\lambda_k^{\text{aug}}{n}}},
    \end{align}
    where ${\lambda}^{\text{aug}}_k$ is the $k$-th eigenvalue of ${\SIG}_{\text{aug}}$.
    On the other hand, by Lemma \ref{con_rec} and \ref{op-bound},
    \begin{align*}
        L_{10}=\|\bar{\SIG}^{\frac{1}{2}}\Cb^{-\frac{1}{2}}(\tilde{\X}\tilde{\X}^\top + n\I_n)^{-1}\tilde{\X}_{k+1:p}\|~\lesssim~ &
        \frac{1}{\lambda_{1}^{\text{aug}}\rho_0(\SIG_{\text{aug}};n)}\sqrt{\frac{\|\bar{\SIG}\Cb^{-1}\|(\lambda^{\text{aug}}_{k+1}n + \sum_{j>k}\lambda_j^{\text{aug}})}{n^2}}\\
        ~=~&\sqrt{\frac{\|\bar{\SIG}\Cb^{-1}\|\lambda_{k+1}^{\text{aug}}(1+\rho_k(\SIG_{\text{aug}};n))}{n(\lambda_1^{\text{aug}}\rho_0(\SIG_{\text{aug}};n))^2}},
    \end{align*}
    with probability $1-\delta'-\exp(-c t)$ (where we set $t := \log n$ for the final theorem statement).
    Hence, 
    \begin{align}\label{term1}
        L_7 \leq L_9 + L_{10}
        \lesssim \sqrt{\frac{\|\bar{\SIG}\Cb^{-1}\|}{{n}}}\left(\sqrt{\frac{1}{\lambda_k^{\text{aug}}}}+\sqrt{\frac{\lambda_{k+1}^{\text{aug}}(1+\rho_k(\SIG_{\text{aug}};n))}{(\lambda_1^{\text{aug}}\rho_0(\SIG_{\text{aug}};n))^2}}\right).
    \end{align}
    Next, we bound $L_8$:
    \begin{align*}
        &\|\bar{\SIG}^{\frac{1}{2}}(\mu(\X)^\top\mu(\X) + \cg(\X))^{-1}\mu(\X)^\top- \bar{\SIG}^{\frac{1}{2}}(\mu(\X)^\top\mu(\X) + n\Cb)^{-1}\mu(\X)^\top\|\\
        =~ &n\|\bar{\SIG}^{\frac{1}{2}}(\mu(\X)^\top\mu(\X) + \cg(\X))^{-1}\left(n^{-1}{\cg(\X)}-\Cb\right)(\mu(\X)^\top\mu(\X) + n\Cb)^{-1}\mu(\X)^\top\|\\
        \lesssim~&
        n\underbrace{\|\bar{\SIG}^{\frac{1}{2}}(\mu(\X)^\top\mu(\X) + \cg(\X))^{-1}\Cb^{\frac{1}{2}}\|}_{L_{11}}\|n^{-1}\Cb^{-\frac{1}{2}}{\cg(\X)\Cb^{-\frac{1}{2}}}-\I_p\|\\
      &~\cdot \underbrace{\|\Cb^{\frac{1}{2}}(\mu(\X)^\top\mu(\X) + n\Cb)^{-1}\mu(\X)^\top\|}_{L_{12}}.
    \end{align*}
    The term $L_{11}$ is identical to (\ref{term1}) and can be bounded with that inequality. 
    In the meantime, the term $L_{12} = \|\bar{\SIG}^{\frac{1}{2}}(\mu(\X)^\top\mu(\X) + \cg(\X))^{-1}\Cb^{\frac{1}{2}}\|$ can be bounded by noting that,
    \begin{align*}
        & \mu_p\left(\left(\Cb^{\frac{1}{2}}(\mu(\X)^\top\mu(\X) + \cg(\X))^{-1}\Cb^{\frac{1}{2}}\right)^{-1}\right)\\
        \gtrsim~&\mu_p\left(\left(\Cb^{\frac{1}{2}}(\mu(\X)^\top\mu(\X) + n\Cb)^{-1}\Cb^{\frac{1}{2}}\right)^{-1}\right)
        \\
        -~&
        \|\Cb^{-\frac{1}{2}}(\mu(\X)^\top\mu(\X) + n\Cb)\Cb^{-\frac{1}{2}}-\Cb^{-\frac{1}{2}}(\mu(\X)^\top\mu(\X) + \cg(\X))\Cb^{-\frac{1}{2}}\|.
    \end{align*}
    Here, by Lemma \ref{con_rec}
    \begin{align}
        \mu_p\left(\left(\Cb^{\frac{1}{2}}(\mu(\X)^\top\mu(\X) + n\Cb)^{-1}\Cb^{\frac{1}{2}}\right)^{-1}\right)=\mu_p\left(\left(\tilde{\X}^\top \tilde{\X} + n\I_p\right)\right) \geq n
    \end{align}
    Also,
    \begin{align*}
        &\|\Cb^{-\frac{1}{2}}(\mu(\X)^\top\mu(\X) + n\Cb)\Cb^{-\frac{1}{2}}-\Cb^{-\frac{1}{2}}(\mu(\X)^\top\mu(\X) + \cg(\X))\Cb^{-\frac{1}{2}}\|\\
        =&~
        \|\Cb^{-\frac{1}{2}}{\cg(\X)\Cb^{-\frac{1}{2}}}-n\I_p\| = n{\Delta}_G
    \end{align*}
    Adding the above inequalities together, $L_8$ is bounded by
    \begin{align}
        &\|\bar{\SIG}^{\frac{1}{2}}\mu(\X)(\mu(\X)^\top\mu(\X) + \cg(\X))^{-1}\Cb^{\frac{1}{2}}- \bar{\SIG}^{\frac{1}{2}}\mu(\X)(\mu(\X)^\top\mu(\X) + n\Cb)^{-1}\Cb^{\frac{1}{2}}\|\nonumber\\
        &~\lesssim 
        \frac{{\Delta}_G}{1-{\Delta}_G}\sqrt{\frac{\|\bar{\SIG}\Cb^{-1}\|}{n}}\left(\sqrt{\frac{1}{\lambda_k^{\text{aug}}}}+\sqrt{\frac{\lambda_{k+1}^{\text{aug}}(1+\rho_k(\SIG_{\text{aug}};n))}{(\lambda_1^{\text{aug}}\rho_0(\SIG_{\text{aug}};n))^2}}\right)\label{l8},
    \end{align}
    by our assumption that ${\Delta}_G \leq c$ for some $c < 1$. E.q. (\ref{term1}) and (\ref{l8}) now imply 
    \begin{align}
        L_5&= \|\bar{\SIG}^{\frac{1}{2}}(\mu(\X)^\top\mu(\X) + \cg(\X))^{-1}\mu(\X)^\top\|\leq L_7 + L_8\nonumber\\
        &\lesssim \sqrt{\frac{\|\bar{\SIG}\Cb^{-1}\|}{n}}\left(\sqrt{\frac{1}{\lambda_k^{\text{aug}}}}+\sqrt{\frac{\lambda_{k+1}^{\text{aug}}(1+\rho_k(\SIG_{\text{aug}};n))}{(\lambda_1^{\text{aug}}\rho_0(\SIG_{\text{aug}};n))^2}}\right)\cdot \left(1+\frac{{\Delta}_G}{1-c}\right)\label{l5}.
    \end{align}
    On the other hand, 
    \begin{align}
        L_6=\|\y-\tilde{\y}\|&~=\|(\mu(\X)-\X)\T^*\|=\sqrt{n}\|\T^*\|_{n^{-1}(\mu(\X)-\X)(\mu(\X)-\X)^\top}\nonumber\\
        &~\leq \sqrt{n}\left(\|\T^*\|{\sqrt{\|n^{-1}(\mu(\X)-\X)(\mu(\X)-\X)^\top-\operatorname{Cov}_{\xi}\|}}+\|\T^*\|_{\operatorname{Cov}_{\xi}}\right)\nonumber\\
        &~\leq 
        \sqrt{n}\left(\sqrt{\Delta_\xi}\|\T^*\|+\|\T^*\|_{\operatorname{Cov}_{\xi}}\right)\label{l6},
    \end{align}
    where $\operatorname{Cov}_{\xi}$ is defined in Definition \ref{del_def}.
    
    Combining E.q. (\ref{l5}) and (\ref{l6}), we obtain the following:
    \begin{align}
        L_4 = \|\HT-\tilde{\T}_{\text{aug}}\|_{\bar{\SIG}}=L_5 \cdot L_6~\lesssim~& \sqrt{{\|\bar{\SIG}\Cb^{-1}\|}}\left(1+\frac{{\Delta}_G}{1-c}\right)\left(\sqrt{\Delta_\xi}\|\T^*\|+\|\T^*\|_{\operatorname{Cov}_{\xi}}\right) \nonumber\\
        \cdot~&\left(\sqrt{\frac{1}{\lambda_k^{\text{aug}}}}+\sqrt{\frac{\lambda_{k+1}^{\text{aug}}(1+\rho_k(\SIG_{\text{aug}};n))}{(\lambda_1^{\text{aug}}\rho_0(\SIG_{\text{aug}};n))^2}}\right)\label{l4}
    \end{align}
    
    Finally, putting together the results of Eq. (\ref{l1122}), (\ref{l1+l2}), (\ref{l1212}) and (\ref{l4}) completes the proof.
\end{proof}

\subsection{Proof of Proposition \ref{ind_cor}}\label{prof_prop1_reg}
\begin{repproposition}
{ind_cor}[\textbf{Independent Feature Augmentations}]
    Let $g$ be an independent feature augmentation,
    and $\pi: \{1,2,\dots,p\}\rightarrow \{1,2,\dots,p\}$ be the function that maps the original feature index to the sorted index according to the eigenvalues of ${\SIG}_{\text{aug}}$ in a non-increasing order. Then, data augmentation has a spectrum reordering effect which changes the MSE through the bias modification:
\begin{align*}
    \frac{\text{Bias}}{C_xL_1^4}&\lesssim 
       \left\|{\theta}^{*}_{\pi(k_1+1:p)}\right\|_{{\Sigma}_{\pi(k_1+1:p)}}^{2}
    +\left\|{\theta}_{\pi(1: k_1)}^{*}\right\|_{\mathbb{E}_{\x}[\cg(\x)]^2{\Sigma}_{\pi(1: k_1)}^{-1}}^{2}\frac{(\rho_{k_1}^{\text{aug}})^2}{{(\lambda^{\text{aug}}_{k_1+1})^{-2}}+(\lambda^{\text{aug}}_{1})^{-2}(\rho_{k_1}^{\text{aug}})^2},
\end{align*}
where $\pi(a:b)$ denotes the indices of $\pi(a),\pi(a+1),\dots,\pi(b)$.
Furthermore, if the variance of each feature augmentation $\operatorname{Var}_{g_i}(g_i(x))$ is a sub-exponential random variable with sub-exponential norm $\sigma_i^2$ and mean $\bar{\sigma}_i^2$, $\forall i \in \{1,2,\dots,p\}$, and $p=O(n^{\alpha})$ for some $\alpha >0$, then there exists a constant $c$, depending only on $\alpha$, such that with probability $1-n^{-3}$, 
\begin{align*}
    {\Delta}_{G}\lesssim  \max_i\left(\frac{\sigma_i^2}{\bar{\sigma}_i^2}\right)\sqrt{\frac{\log n}{n}}.
\end{align*}
\end{repproposition}
\begin{proof}
    For independent feature augmentation, $\mathbb{E}_{\x}[\cg(\x)]$ is a diagonal matrix. Since the original covariance $\SIG$ is also diagonal by our model assumption,
    the augmentation modified spectrum $\SIG_{\text{aug}}$ is diagonal. Furthermore, the diagonal implies the projections to $\SIG_{\text{aug}}$'s first $k-1$ and the rest eigenspaces are to the features $\pi(1:k-1)$ and $\pi(k, p)$. Lastly, because $P^{\SIG_{\text{aug}}}$ commutes with $\mathbb{E}_{\x}[\cg(\x)]$, we have
    \begin{align*}
        \left\|\mathbf{P}^{\SIG_{\text{aug}}}_{k_1+1:p}{\theta}_{\text{aug}}^{*}\right\|^{2}_{\SIG_{\text{aug}}}&=   ({\theta}_{\text{aug}}^{*})^\top \mathbf{P}^{\SIG_{\text{aug}}}_{k_1+1:p}{\theta}_{\text{aug}}^{*}\\
        &=(\T^*)^\top\bar{\D}^{1/2}\mathbf{P}^{\SIG_{\text{aug}}}_{k_1+1:p}\bar{\D}^{-1/2}\SIG\bar{\D}^{-1/2}\mathbf{P}^{\SIG_{\text{aug}}}_{k_1+1:p}\bar{\D}^{1/2}\T^*\\
        &=(\T^*)^\top\mathbf{P}^{\SIG_{\text{aug}}}_{k_1+1:p}\bar{\D}^{1/2}\bar{\D}^{-1/2}\SIG\bar{\D}^{-1/2}\bar{\D}^{1/2}\mathbf{P}^{\SIG_{\text{aug}}}_{k_1+1:p}\T^*\\
        &=\|\mathbf{P}^{\SIG_{\text{aug}}}_{k_1+1:p}\T^*\|_{\SIG}^2=\left\|{\theta}^{*}_{\pi(k_1+1:p)}\right\|_{{\Sigma}_{\pi(k_1+1:p)}}^{2},\\
        \left\|\mathbf{P}^{\SIG_{\text{aug}}}_{1:k_1}{\theta}_{\text{aug}}^{*}\right\|^{2}_{\SIG_{\text{aug}}^{-1}}&=(\T^*)^\top\mathbf{P}^{\SIG_{\text{aug}}}_{1:k_1}\bar{\D}^{1/2}\bar{\D}^{1/2}\SIG^{-1}\bar{\D}^{1/2}\bar{\D}^{1/2}\mathbf{P}^{\SIG_{\text{aug}}}_{1:k_1}\T^*\\
        &=\left\|{\theta}_{\pi(1:k_1)}^{*}\right\|_{\mathbb{E}_{\x}[\cg(\x)]^2{\Sigma}_{\pi(1:k_1)}^{-1}}^{2},
    \end{align*}
    where $\bar{\D} = \mathbb{E}_{\x}[\cg(\x)]$.
    
    To prove the approximation error bound, we proceed as follows.
    By independence assumption on feature augmentation, $\cg(\X)$ is diagonal. Hence, to bound ${\Delta}_G$, we only need to control the diagonals of $\Q := n^{-1}\mathbb{E}_{\x}[\cg(\x)]^{-\frac{1}{2}}\cg(\X)\mathbb{E}_{\x}[\cg(\x)]^{-\frac{1}{2}} - \I$. Now, denoting $\D = n^{-1}\mathbb{E}_{\x}[\cg(\x)]^{-\frac{1}{2}}\cg(\X)\mathbb{E}_{\x}[\cg(\x)]^{-\frac{1}{2}}$, we have $\Q=\D-\I$.
    For any $i \in \{1,2,\dots,p\}$, 
    $\D_{ii} =  n^{-1}\sum_{j=1}^n\frac{\operatorname{Var}_{g_i}(\x_{ji})}{\mathbb{E}_\x[\operatorname{Var}_{g_i}(\x)]}$, where $\x_{ji}$ is the $i$-th element of the $j$-th row of $\X$. By our assumptions of $\operatorname{Var}_{g_i}(\x_{ji})$, $j=1,2,\dots,n$, being
    identical and independent sub-exponential random variables with sub-exponential norm $\sigma_{i}^2$ and mean $\bar{\sigma}_{i}^2$. 
    we can apply concentration bounds to
    $\Q_{ii} = \frac{1}{\bar{\sigma_i}^2} \left(n^{-1}\sum_{j=1}^n\operatorname{Var}_{g_i}(\x_j)-\mathbb{E}_\x[\operatorname{Var}_{g_i}(\x)]\right) $ as it is a sum of i.i.d. sub-exponential random variables with sub-exponential norm ${\sigma_i}^2/{\bar{\sigma}_i^2}$. Specifically, we apply the Bernstein inequality in Lemma \ref{ber_exp} with $t\propto \sigma_i^2 \sqrt{\frac{\log n}{n}}$ to conclude that there exists a constant $c'$ such that, with probability $1-n^{-1}$, we have,
    \begin{align}
        \Q_{ii} = \frac{1}{\bar{\sigma_i}^2} \left(n^{-1}\sum_{j=1}^n\operatorname{Var}_{g_i}(\x_j)-\mathbb{E}_\x[\operatorname{Var}_{g_i}(\x)]\right) \leq c'\frac{\sigma_i^2}{\bar{\sigma}_i^2}\sqrt{\frac{\log n}{n}}.
    \end{align}
    Then, we apply a union bound over $i$ and obtain
    \begin{align*}
        \|\boldsymbol{\Delta}_{G}\|\leq \max_i\|\Q_{ii}\|\lesssim  \max_i\left(\frac{\sigma_i^2}{\bar{\sigma}_i^2}\right)\sqrt{\frac{\log n}{n}},
    \end{align*}
        with probability $1-n^{-1}$.
     Note that we can get the same error rate after the union bound as long as $p$ grows polynomially with $n$.

\end{proof}

\subsection{Proof of Proposition~\ref{prop:weakcorrelations}}\label{sec:propweakcorrelationsproof}

\begin{repproposition}{prop:weakcorrelations}
Consider a correlated-feature augmentation of the form described above. Further, assume that the smallest eigenvalue of $\E_\x\mathrm{Cov}_{\mathcal{G}_k}(\x)$ is lower bounded by $\sigma$ for every $k$, and $g_k$ is component-wise bounded, i.e., $\|g_k(\x_k)\|_{\infty} \leq M$ for any $k$. Then, we have
\begin{align*}
        \Delta_G \lesssim \frac{M^2 \max_k|B_k|}{\sigma}\sqrt{\frac{\log p}{n}} 
\end{align*}
with probability at least $1 - \frac{1}{p}$.
\end{repproposition}
\begin{proof}
We begin by bounding $\Delta_{G_k}$, which is the component of $\Delta_G$ corresponding to the $k$th block of $\cg$.
Specifically, we have
\begin{align*}
    \Delta_{G_k}:=~&\left\|\frac{1}{n}\mathbb{E}_{\x}[\text{Cov}_{\mathcal{G}_k}(\x)]^{-\frac{1}{2}}\sum_{i=1}^n\text{Cov}_{\mathcal{G}_k}(\x_i)\mathbb{E}_{\x}[\text{Cov}_{\mathcal{G}_k}(\x)]^{-\frac{1}{2}} - \I_{|B_k|}\right\|\\
    ~=~&\left\|\mathbb{E}_{\x}[\text{Cov}_{\mathcal{G}_k}(\x)]^{-\frac{1}{2}}\left(\frac{1}{n}\sum_{i=1}^n\text{Cov}_{\mathcal{G}_k}(\x_i)- \mathbb{E}_{\x}[\text{Cov}_{\mathcal{G}_k}(\x)]\right)\mathbb{E}_{\x}[\text{Cov}_{\mathcal{G}_k}(\x)]^{-\frac{1}{2}} \right\|\\
    ~\leq~&\sigma^{-1}\left\|\frac{1}{n}\sum_{i=1}^n\text{Cov}_{\mathcal{G}_k}(\x_i)- \mathbb{E}_{\x}[\text{Cov}_{\mathcal{G}_k}(\x)] \right\|,
\end{align*}
where the last inequality utilizes the assumption that $\mu_p(\E_\x\mathrm{Cov}_{\mathcal{G}_k}(\x)) \geq \sigma$.
Note that $\left\|\frac{1}{n}\sum_{i=1}^n\text{Cov}_{\mathcal{G}_k}(\x_i)- \mathbb{E}_{\x}[\text{Cov}_{\mathcal{G}_k}(\x)] \right\|$ is the norm of a sum of $n$ independent, zero-mean, bounded random matrices. In particular, note that 
\begin{align*}
    \lambda_{\text{max}}(\text{Cov}_{\mathcal{G}_k}(\x_i) - \E_\x \text{Cov}_{\mathcal{G}_k}(\x)) &\leq \Vert \text{Cov}_{\mathcal{G}_k}(\x_i)\Vert + \E_\x \Vert \text{Cov}_{\mathcal{G}_k}(\x_i)\Vert\ \\
    &\leq 2\E \Vert (g(\x_i) - \mu_g(\x_i))(g(\x_i) - \mu_g(\x_i))^\top\Vert\\
    &\leq 2\E \Vert g(\x_i) - \mu_g(\x_i)\Vert_2^2 \\
    &\leq 2 |B_k| \E \Vert g(\x_i) - \mu_g(\x_i)\Vert_\infty^2\\
    &\leq 8 |B_k| M^2
\end{align*}
where the last inequality follows from the assumption that $\|g_k(\x_k)\|_{\infty} \leq M$ almost surely. 
Moreover, we have
\begin{align*}
    \left\| \sum_{i=1}^n  \E [\text{Cov}_{\mathcal{G}_k}(\x_i) - \E_\x \text{Cov}_{\mathcal{G}_k}(\x)]^2 \right\| &\leq \sum_{i=1}^n \E \Vert \text{Cov}_{\mathcal{G}_k}(\x_i) - \E_\x \text{Cov}_{\mathcal{G}_k}(\x)\Vert^2  \\
    &\leq 64 |B_k|^2 M^4 n
\end{align*}
We can then apply the Matrix Bernstein inequality, e.g., \cite[Theorem 1.4]{tropp2012user} with $t = 32 |B_k| M^2 \sqrt{n \log p}$, to conclude that
\begin{align*}
    \Delta_{G_k} \lesssim \frac{|B_k| M^2}{\sigma} \sqrt{\frac{\log p}{n}}
    \end{align*}
with probability at least $1-\frac{1}{p^2}$.
Finally, applying a union bound over each of the $B_k$ (i.e.~at most $p$ events) yields the result.
\end{proof}

\subsection{Proof of Proposition~\ref{prop:dependent_bound}}\label{sec:propdependentboundproof}

\begin{repproposition}{prop:dependent_bound}
Consider the decomposition $\cg(\X)=\D + \Q$, where $\D$ is a diagonal matrix representing the \textit{independent} feature augmentation part. 
Then, we have
\begin{align}\label{dependent_bound}
    {\Delta}_{G}\lesssim \frac{\|\D-\E\D\| + \|\Q-\E\Q\|}{\mu_p(\E_\x\cg(\x))}. 
\end{align}
\end{repproposition}
\begin{proof}
This proof proceeds by partition the augmented covariance operator into diagonal and nondiagonal parts $\D$ and $\Q$ (i.e., $\operatorname{Cov}_G(\X)=\D + \Q$).
We then bound the terms separately as below:
\begin{align*}
    {\Delta}_{G} &= \|\E_\x\cg(\x)^{-1/2}(\D + \Q)\E_\x\cg(\x)^{-1/2}-\I_p\|\\
    &=\|\E_\x\cg(\x)^{-1/2}(\D+\Q-\E_\x\cg(\x))\E_\x\cg(\x)^{-1/2}\|\\
    &\leq \frac{\|\D-\E\D\| + \|\Q-\E\Q\|}{\mu_p(\E_\x\cg(\x))},~~\because~\mathbb{E}\D + \mathbb{E}\Q= \E_\x\cg(\x).
\end{align*}
\end{proof}

\subsection{Proofs of Corollaries}\label{pf_reg_cor}

\begin{corollary}[\textbf{Generalization of Gaussian Noise Injection }]{\label{gen_agn}}
    Consider the data augmentation which adds samples with independent additive Gaussian noise: $g(\x) = \x + \mathbf{n}$, where $\mathbf{n}\sim\mathcal{N}(0,\sigma^2)$.
    The estimator is given by $\hat{\T} = (\X^\top\X + \sigma^2n\mathbf{I}_p)^{-1}\X^\top y$. Let $L$ denote the condition number of $n\sigma^2\I + \X_{1:k}\X_{1:k}^\top$. Then, we can bound the error as MSE $\leq$ Bias $+$ Variance, where with high probability
    \begin{align*}
        &\mathrm{MSE} \lesssim \|\TBS\|^2_{\SB} +
            \|\TAS\|^2_{\SA^{-1}}\lambda_{k+1}^2\rho_k^2(\SIG;n\sigma^2) + 
            R_k^{-1}(\SIG;n\sigma^2) + kn^{-1}.
    \end{align*}
\end{corollary}
\begin{proof}
    Since this belongs to the independent feature augmentation class, we can apply Corollary \ref{ind_cor}. Below are the quantities in the corollary.
    \begin{align*}
        &\mathbb{E}_{\x}[\cg(\x)] = \sigma^2 \I,~{\T}_\aug^*=\sigma\T^*,~{\SIG}_\aug=\sigma^{-2} \SIG,~{\lambda}^{\aug} = \sigma^{-2}\lambda,
    \end{align*}
    hence,
    \begin{align*}
        \rho_k^{\aug} &= \rho_k(\SIG_\aug;n) = \frac{n + \sum\limits_{i=k+1}^p\lambda^\aug_i}{n\lambda_{k+1}^\aug}= \frac{n\sigma^2 + \sum\limits_{i=k+1}^p\lambda_i}{n\lambda_{k+1}}=\rho_k(\SIG;n\sigma^2),\\
        R_k^{\aug} &= R_k(\SIG_\aug;n) = \frac{\left(n + \sum\limits_{i=k+1}^p\lambda^\aug_i\right)^2}{\sum\limits_{i=k+1}^p(\lambda_{k+1}^\aug)^2}= \frac{\left(n\sigma^2 + \sum\limits_{i=k+1}^p\lambda_i\right)^2}{n\sum\limits_{i=k+1}^p\lambda_{k+1}^2}=R_k(\SIG;n\sigma^2).
    \end{align*}
    Note that $R_k(\SIG;n\sigma^2)$ and $\rho_k(\SIG_\aug;n\sigma^2)$ are the effective dimensions of the original spectrum for ridge regression with regularization parameter $n\sigma^2$, as defined in \cite{tsigler2020benign}.
    Finally, the approximation error term is zero because ${\Delta}_G=0$.
\end{proof}


\begin{repcorollary}{cor_rm}[\textbf{Generalization of random mask augmentation}]
    Consider the unbiased randomized masking augmentation $g(\x) = [b_1\x_1,\dots,b_p\x_p]/(1-\beta)$, where $b_i$ are i.i.d. Bernoulli$(1-\beta)$. Define $\psi=\frac{\beta}{1-\beta}\in [0, \infty)$.
    Let $L_1$, $L_2$, $\kappa$, $\delta'$ be universal constants as defined in Theorem \ref{gen_bound}. Assume $p=O(n^{\alpha})$ for some $\alpha>0$. Then, for any set $\mathcal{K}\subset \{1,2,\dots,p\} $ consisting of $k_1$ elements and $k_2\in [0, n]$, there exists some constant $c'$, which depends solely on $\sigma_z$ and $\sigma_{\varepsilon}$ (the sub-Guassian norms of the covariates and noise), such that the regression MSE is upper-bounded by
    \begin{align*}
        \mathrm{MSE} \lesssim& 
       \underbrace{\left\|{\theta}^{*}_{\mathcal{K}}\right\|_{{\Sigma}_{\mathcal{K}}}^{2}
        +\left\|{\theta}_{\mathcal{K}^c}^{*}\right\|_{{\Sigma}_{\mathcal{K}^c}}^{2}\frac{(\psi n + p - k_1)^2}{n^2 + (\psi n + p -k_1)^2}}_{\mathrm{Bias}}\\
        +& \underbrace{\left(\frac{k_2}{n}+\frac{ n(p-k_2)}{(\psi n + p - k_2)^2}\right)\log n}_{\mathrm{Variance}}
        +\underbrace{ \sigma^2_z\sqrt{\frac{\log n}{n}}\|\T^*\|_{\SIG}}_{\mathrm{Approx. Error}}
    \end{align*}
    with probability at least $1 - \delta'-n^{-1}$.
\end{repcorollary}
\begin{proof}
    Random mask belongs to independent feature augmentation class, so we can apply Proposition \ref{ind_cor}. We calculate the quantities used in the corollary.
    \begin{align*}
        &\mathbb{E}_{\x}[\cg(\x)] = \psi \text{diag}(\SIG)=\psi \SIG,~{\T}_\aug^*=\psi^{1/2} \SIG^{1/2}\T^*,~{\SIG}_\aug=\psi^{-1}\I,~{\lambda}^{\aug} = \psi^{-1}.
    \end{align*}
    The effective ranks of the augmentation modified spectrum are 
    \begin{align}
        \rho^\aug_k &= \frac{\psi n + p -k}{n},\\
        R_k^\aug &= \frac{(\psi n + p - k)^2}{p-k}.
    \end{align}
    
    Now, we apply Proposition \ref{ind_cor}. Because random mask has effectively isotropized the spectrum, the mapping $\pi$ in the proposition can be chosen arbitrarily. Hence, we can chose $\pi(1:k_1)$ to be any set with $k$ elements. For the approximation error term, we first note that $\kappa = 1$. Furthermore, $\text{Var}_{g_i}(\x_j) = \psi \x_j^2$. So, its subexponential norm is bounded by $\psi \lambda_j \sigma_z^2$, and its expectation is given by $\psi \lambda_j$.
    Putting all the pieces together, we derive the MSE bound as
    \begin{align*}
        \text{Bias}&\lesssim 
       \left\|{\theta}^{*}_{\mathcal{K}}\right\|_{{\Sigma}_{\mathcal{K}}}^{2}
        +\left\|{\theta}_{\mathcal{K}^c}^{*}\right\|_{{\Sigma}_{\mathcal{K}^c}}^{2}\frac{(\psi n + p - k_1)^2}{n^2 + (\psi n + p -k_1)^2},\\
        \text{Variance}&\lesssim \frac{k_2}{n}+\frac{n(p-k_2)}{(\psi n + p - k_2)^2},\\
        \text{Approx. Error}&\lesssim { \sigma_z^2\sqrt{\frac{\log n}{n}}\|\T^*\|_{\SIG}}.
    \end{align*}
\end{proof}

\begin{repcorollary}{cor:cutout}[\textbf{Bounds of random cutout}]
    Let $\Th^{\text{cutout}}_k$ denote the random cutout estimator that zeroes out $k$ consecutive coordinates (the starting location of which is chosen uniformly at random). Also, let $\Th^{\text{mask}}_{\beta}$ be the random mask estimator with the masking probability given by $\beta$. We assume that $k=O(\sqrt{\frac{n}{\log p}})$.
    Then, for the choice $\beta=\frac{k}{p}$ we have
    $$\mathrm{MSE}(\Th^{\text{cutout}}_k)\asymp \mathrm{MSE}(\Th^{\text{mask}}_{\beta}),~~\mathrm{POE}(\Th^{\text{cutout}}_k)\asymp \mathrm{POE}(\Th^{\text{mask}}_{\beta}).$$
\end{repcorollary}
\begin{proof}
    This can be verified directly by noticing that for random cutout $$\mathbb{E}_\x\cg(\x)=\frac{k}{p-k}\mathrm{diag}(\SIG),$$ while for random mask 
    $$\mathbb{E}_\x\cg(\x)=\psi\mathrm{diag}(\SIG).$$ Furthermore, the approximation is negligible when $k\ll\min(\sqrt{\frac{n}{\log p}}, \frac{p}{\sqrt{n}})$ as shown in Appendix \ref{sec:approx_cut}. Now, setting $\psi=\frac{k}{d-k}$ gives $\beta=\frac{k}{p}$. 
\end{proof}

\begin{repcorollary}{het_mask}[\textbf{Non-uniform random mask}] Consider a general $k-$sparse model where $\T^*=\sum_{i\in\mathcal{I}_{\mathcal{S}}}\alpha_i\mathbf{e}_i$, where $|\mathcal{I}_{\mathcal{S}}|=k$. Suppose we employ non-uniform random mask where $\psi=\psi_1$ if $i\in \mathcal{I}_{\mathcal{S}}$ and $=\psi_0$ otherwise. Then, if $\psi_1 \leq \psi_0$, we have
\begin{align*}
    \mathrm{Bias} &\lesssim  \frac{\left(\psi_{1}n + \frac{\psi_{1}}{\psi_{0}}\left(p-|\mathcal{I}_{\mathcal{S}}|\right)\right)^2}{n^2 + \left(\psi_{1}n + \frac{\psi_{1}}{\psi_{0}}\left(p-|\mathcal{I}_{\mathcal{S}}|\right)\right)^2}\|\T^*\|_\SIG^2,\\
    \mathrm{Variance} &\lesssim \frac{|\mathcal{I}_{\mathcal{S}}|}{n} + \frac{n\left(p - |\mathcal{I}_{\mathcal{S}}|\right)}{\left(\psi_0 n + p - |\mathcal{I}_{\mathcal{S}}|\right)^2},\\
    \mathrm{Approx. Error} &\lesssim \sqrt{\frac{\psi_0}{\psi_1}}\sigma_z^2\sqrt{\frac{\log n}{n}}\|\T^*\|_{\SIG}
\end{align*}
while if $\psi_1 > \psi_0$, we have
\begin{align*}
    \mathrm{Bias} \lesssim  \|\T^*\|_\SIG^2,~~\mathrm{Variance} \lesssim \frac{\left(\frac{\psi_1}{\psi_o}\right)^2 + \frac{|\mathcal{I}_{\mathcal{S}}|}{n}}{\left(\frac{\psi_1}{\psi_o}+\frac{|\mathcal{I}_{\mathcal{S}}|}{n}\right)^2}, ~~\mathrm{Approx. Error} &\lesssim \sqrt{\frac{\psi_1}{\psi_0}}\sigma_z^2\sqrt{\frac{\log n}{n}}\|\T^*\|_{\SIG}\\
\end{align*}
\end{repcorollary}
\begin{proof}
Let $\Psi$ denote the diagonal matrix with $\Psi_{i,i}=\psi_1$ if $i\in \mathcal{I}_{\mathcal{S}}$ and $\psi_0$ otherwise.
Then, we apply Corollary \ref{ind_cor} with:
    \begin{align*} &\mathbb{E}_{\x}[\cg(\x)] = \Psi \text{diag}(\SIG)=\Psi \SIG,~{\T}_\aug^*=\Psi^{1/2} \SIG^{1/2}\T^*,~{\SIG}_\aug=\Psi^{-1}.
\end{align*}
Now as in the proof of Proposition \ref{cor_rm}, we calculate the effective ranks. For the $k^*$ partitioning the spectrum, we choose $k^* =|\mathcal{I}_{\mathcal{S}}|$ when $\psi_1 \leq \psi_0$, while $k^*\asymp n$ for $\psi_1 > \psi_0$. The proof for the approximation error term is identical to in the uniform random mask case.
\end{proof}

\begin{repcorollary}{salt_cor}[\textbf{Generalization of Pepper/Salt augmentation}]
    The MSE components of the estimator that are induced by salt-and-pepper augmentation (denoted by $\Th_{\text{pepper}}(\beta,\sigma^2)$) have the properties,
    \begin{align*}
         \mathrm{Bias}[\Th_{\text{pepper}}(\beta,\sigma^2)] ~&\lesssim ~
        \left(\frac{\lambda_1(1-\beta)+\sigma^2}{\sigma^2}\right)^2\mathrm{Bias}\left[\Th_{\text{gn}}\left(\frac{\beta\sigma^2}{(1-\beta)^2}\right)\right],\\
         \mathrm{Variance}[\Th_{\text{pepper}}(\beta,\sigma^2)] ~&\lesssim ~ \mathrm{Variance}\left[\Th_{\text{gn}}\left(\frac{\beta\sigma^2}{(1-\beta)^2}\right)\right], \\
          \mathrm{Approx.Error}[\Th_{\text{pepper}}(\beta,\sigma^2)] ~&\asymp~ \mathrm{Approx.Error}[\Th_{\text{rm}}(\beta)].
    \end{align*}
    where $\Th_{\text{gn}}(z^2)$ and $\Th_{\text{rm}}(\gamma)$ denotes the estimators that are induced by Gaussian noise injection with variance $z^2$ and random mask with dropout probability $\gamma$, respectively.
    Moreover, the limiting MSE as $\sigma \to 0$ reduces to the MSE of the estimator induced by random masking (denoted by $\Th_{\text{rm}}(\beta)$):
    \begin{align*}
        \lim_{\sigma\rightarrow 0} \mathrm{MSE}[\Th_{\text{pepper}}(\beta,\sigma^2)] = \mathrm{MSE}[\Th_{\text{rm}}(\beta)].
    \end{align*}
\end{repcorollary}

\begin{proof}
     Proposition \ref{ind_cor} is applicable to salt/pepper augmentation. The related quantities in the proposition are:
    \begin{align*}
        &\mathbb{E}_{\x}[\cg(\x)] = \psi \SIG + \frac{\psi \sigma^2}{1-\beta}\I,~{\T}_\aug^*=\sqrt{ \psi \SIG + \frac{\psi \sigma^2}{1-\beta}\I}\T^*,~{\lambda}_i^{\aug} = \frac{\lambda_i}{\psi(\lambda_i + \frac{\sigma^2}{1-\beta})}.
    \end{align*}
    Observe that the expression of $\lambda_i^{\text{aug}}$ implies that the augmented eigenvalues of salt/pepper augmentation is a harmonic sum of that of random mask and Gaussian noise injection,
    \begin{align}\label{harmo}
        \lambda_{{pepper}}(\beta,\sigma^2)^{-1} = \lambda_{{rm}}(\beta)^{-1} + \beta^{-1}\lambda_{{gn}}(\sigma^2)^{-1}.
    \end{align}
    Hence, the statement of MSE limit is clear as we take $\sigma\rightarrow 0$ in \eqref{harmo} along with the fact that $\lambda_{\text{gn}}{\rightarrow}\infty$.
    Now we prove the bias statement. By Proposition \ref{ind_cor},
    \begin{align}\label{eq111}
        \Th_{\text{pepper}}(\beta,\sigma)\lesssim 
        \|\T^*_{k+1:p}\|_{\SIG_{k+1:p}}^2
        + \left\|{\theta}_{\pi(1: k_1)}^{*}\right\|_{\mathbb{E}_{\x}[\cg(\x)]^2{\Sigma}_{\pi(1: k_1)}^{-1}}^{2}(\lambda_{k+1}^{\text{aug}}\rho_k^{\text{aug}})^2.
    \end{align}
    In particular,
    \begin{align}
        &\left\|{\theta}_{\pi(1: k_1)}^{*}\right\|_{\mathbb{E}_{\x}[\cg(\x)]^2{\Sigma}_{\pi(1: k_1)}^{-1}}^{2}=\sum\limits_{i\leq k} \frac{\left(\psi \lambda_i + \frac{\psi \sigma^2}{1-\beta}\right)^2}{\lambda_i}(\T^*_i)^2,\label{eq112}\\
        &\lambda_{k+1}^\aug\rho^\aug_k = \frac{n + {\sum\limits_{i>k}}\frac{\lambda_i}{\psi(\lambda_i + \frac{\sigma^2}{1-\beta})}}{n}
        \leq \frac{n + \sum\limits_{i>k}\frac{\lambda_i}{\psi \frac{\sigma^2}{1-\beta}}}{n}\label{eq113}.
    \end{align}
    Now the result follows by combining Eq. (\ref{eq111}), (\ref{eq112}) and (\ref{eq113}). 
    
    The variance statement can be proved using similar calculations. From Corollary \ref{ind_cor}, we only need to compare $R_k$ of salt/pepper with that of Gaussian noise injection.
    Without lose of generality, we assume $k$ is chosen in the corollary such that $\lambda_i \leq c'\frac{\sigma^2}{1-\beta} $ for all $i \geq k$ for some constant $c'$. Then,
    \begin{align*}
        R_k \geq \frac{\left(n + \sum\limits_{i\geq k}\frac{\lambda_i}{\psi(\lambda_i + \frac{\sigma^2}{1-\beta})}\right)^2}{\sum\limits_{i\geq k}\left(\frac{\lambda_i}{\psi(\lambda_i + \frac{\sigma^2}{1-\beta})}\right)^2}\geq \frac{\left(n + \sum\limits_{i\geq k}\frac{\lambda_i}{\psi((c'+1)\frac{\sigma^2}{1-\beta})}\right)^2}{\sum\limits_{i\geq k}\left(\frac{\lambda_i}{\psi( \frac{\sigma^2}{1-\beta})}\right)^2}\geq \frac{1}{(c'+1)^2} \frac{\left(n + \sum\limits_{i\geq k}\frac{\lambda_i}{\frac{\beta\sigma^2}{(1-\beta)^2}}\right)^2}{\sum\limits_{i\geq k}\left(\frac{\lambda_i}{ \frac{\beta\sigma^2}{(1-\beta)^2}}\right)^2},
    \end{align*}
    The statement now follows by noting that the last quantity is the $R_k$ of Gaussian noise injection with noise variance $\frac{\beta\sigma^2}{(1-\beta)^2}$ up to a constant scaling factor.
    
    Finally, the approximation error statement holds because the augmented covariance is that of random mask summed with a constant matrix. 
\end{proof}

\begin{corollary}[\textbf{Generalization of biased mask augmentation}]\label{bias-rm}
    Consider the biased random mask augmentation $g(\x) = [b_1\x_1,\dots,b_p\x_p]$, where $b_i$ are i.i.d. Bernoulli(1-$\beta$). Define $\psi=\frac{\beta}{1-\beta}\in [0, \infty)$. Assume the assumptions in Corollary \ref{cor_rm} hold. Then  with probability $1-\delta'-3pn^{-5}$, the generalization error is upper bounded by 
    \begin{align*}
        {\mathrm{MSE}(\HT)}\leq \left(\sqrt{\mathrm{MSE}^o} + \psi\left(1+\frac{\log n}{n}\right)\cdot\left(\left(\lambda_1 + \frac{\sum_{j}\lambda_j}{n}\right)\|\T^*\|+\|\T^*\|_{\SIG}\right) \right)^2,
    \end{align*}
    where $\mathrm{MSE}^o$ is the bound given in Corollary \ref{cor_rm}.
\end{corollary}
\begin{proof}
    This proof is a direct application of Theorem \ref{bias-thm1} by the two steps: First, plugging in 
    \begin{align*}
        \SIG_{\text{aug}}=\frac{1-\beta}{\beta}\I,~\bar{\SIG}=(1-\beta)^2\SIG,~\E_\x\cg(\x)=\beta(1-\beta)\SIG.
    \end{align*}
    Secondly, observing $\delta(\x)=-\beta \x$, $\mathrm{Cov}_\delta=\beta^2\SIG$, so concentration bound in Lemma \ref{op-bound} gives that $$\Delta_\delta\lesssim\beta^2\left(  \frac {\lambda_1n + \sum_j\lambda_j}{n}\right).$$
\end{proof}
\section{Proofs of Classification Results}\label{main_proof_cls}

\subsection{Classification Lemmas}

\begin{lemma}[\textbf{Upper bound on probability of classification error for correlated sub-Gaussian input}]\label{poe_bound}
Consider the 1-sparse model $\T^*=\frac{1}{\sqrt{\lambda_t}}\mathbf{e}_t$ described in Section  \ref{class-results} and input distribution satisfying Assumption \ref{cls_assump}, where $\x_{sig} = \x_t$ is the feature corresponding to the non-zero coordinate of $\theta^*$. Given any estimator $\hat{\theta}$ having $\Th_t\geq 0$, the probability of classification error (POE) is upper bounded by
\begin{align}\label{poe_1}
    \mathrm{POE}(\hat{\theta})\lesssim \frac{\mathrm{CN}}{\mathrm{SU}}\left(1 + \sigma_z \sqrt{\log{\frac{\mathrm{SU}}{\mathrm{CN}}}}\right).
\end{align}
Furthermore, if we assume $\x$ is Gaussian, then
\begin{align}\label{poe_2}
     \mathrm{POE}(\hat{\theta})=\frac{1}{2} - \frac{1}{\pi}\tan^{-1}{\frac{\mathrm{SU(\Th)}}{ \mathrm{CN(\Th)}}}\leq {\frac{\mathrm{CN(\Th)}}{ \mathrm{SU(\Th)}}}.
\end{align}
\begin{proof}

We first note that the assumption that $\hat{\theta}_t \geq 0$ is satisfied in the situations we consider, based on the lower bounds on survival which we provide in Lemma \ref{su-bound}. Assume without loss of generality that $\x_{sig} = \x_{t}  =\x_1$.
\begin{align*}
    \mathrm{POE}(\Th) &= \mathbb{P}\left(\sgn({\x_{\text{sig}}}) \neq \sgn(\langle \x, \Th \rangle)\right)\\
    &=  \mathbb{P}\left(\sgn(\x_{\text{sig}}) \neq {\sgn}(\x_{\text{sig}}(\Th_1 + \frac{\x_2}{\x_{\text{sig}}}\Th_2 + \cdots + \frac{\x_p}{\x_{\text{sig}}}\Th_p))\right)\\
    &= \mathbb{P}\left(\Th_1 + \frac{\x_2}{\x_{\text{sig}}}\Th_2 + \cdots + \frac{\x_p}{\x_{\text{sig}}}\Th_p < 0\right)\\
    &= \mathbb{E}_{\x_{\text{sig}}} \mathbb{P}\left( \frac{\x_2}{\x_{\text{sig}}}\Th_2 + \cdots + \frac{\x_p}{\x_{\text{sig}}}\Th_p < -|\Th_1|\right).
\end{align*}
Now, because $\z' :=[\frac{\x_2}{\sqrt{\lambda_2}}, \frac{\x_3}{\sqrt{\lambda_3}},\dots,\frac{\x_p}{\sqrt{\lambda_p}}]$ is a sub-Gaussian vector with norm $\sigma_z$, $\langle \z', \mathbf{u}\rangle$ is a sub-Gaussian variable with norm $\|\mathbf{u}\|$ for any fixed $\mathbf{u}$.
Let $\mathbf{u} = \frac{1}{\x_{\text{sig}}}[{\sqrt{\lambda_2}}\Th_2, {\sqrt{\lambda_3}}\Th_3,\dots,{\sqrt{\lambda_p}}\Th_p]$, which, by assumption, is independent of $\z'$. Then,
\begin{align*}
    & \mathbb{E}_{\x_{\text{sig}}} \mathbb{P}\left( \frac{\x_2}{\x_{\text{sig}}}\Th_2 + \cdots + \frac{\x_p}{\x_{\text{sig}}}\Th_p < -|\Th_1|\right)= \mathbb{E}_{\x_{\text{sig}}} \mathbb{P}\left(\langle \z', \mathbf{u}\rangle\leq -|\Th_1|\right)\\
    &\leq \mathbb{E}_{\x_{\text{sig}}} \exp\left(- \frac{ \Th_1^2}{\sum_{j\geq2}\lambda_j(\frac{\Th_j}{\x_{\text{sig}}})^2\sigma_z^2}\right)\\
    &= \mathbb{E}_{\x_{\text{sig}}} \exp\left(-\frac{\x_{\text{sig}}^2}{\lambda_1\sigma_z^2}\frac{\mathrm{SU}(\Th)^2}{\mathrm{CN}(\Th)^2}\right)\\
    &\leq \mathbb{P}(\x_{\text{sig}}^2 < \delta) + 3\exp\left(-\frac{\delta}{\lambda_1\sigma_z^2}\frac{\mathrm{SU}(\Th)^2}{\mathrm{CN}(\Th)^2}\right)\\
    &\lesssim \sqrt{\frac{\delta}{\lambda_1}} + 3\exp\left(-\frac{\delta}{\lambda_1\sigma_z^2}\frac{\mathrm{SU}(\Th)^2}{\mathrm{CN}(\Th)^2}\right),
\end{align*}
where the last inequality follows from the assumption that $\z_{sig}$ has bounded density and a small ball probability bound from \citep[Exercise 2.2.10]{vershynin2010introduction}. Choosing $\delta ={\lambda_1\sigma_z^2\log \frac{SU}{CN}}/{\left(\frac{SU}{CN}\right)^2}$ yields the result.

The second statement follows from Proposition 17 in \cite{muthukumar2020class} and the bound $\tan^{-1}(x) \geq \frac{\pi}{2} - \frac{1}{x}$, for all $x>0$.
\end{proof}
\end{lemma}

\begin{lemma}[\textbf{Survival of ridge estimator for dependent features}]\label{su-bound}
Consider the classification task under the model and assumption described in Section \ref{class-results} where  $\SIG = \text{diag}(\lambda_1, \dots, \lambda_p)$ and the true signal $\theta^* = \frac{1}{\sqrt{\lambda_t}} \e_t $ is 1-sparse in coordinate $t$. Let $\Th = \X^\top(\X \X^\top + \lambda \I)^{-1}\y$ be a ridge estimator. Suppose for some $t \leq k \leq n$ that $\lambda_{k+1}\rho_k(\SIG;\lambda)\geq c$ for some constant $c>0$, and with probability at least $1-\delta$ that the condition number of $\lambda \I + \X_{k+1:p}\X_{k+1:p}^T$ is at most $L$, then with probability $1-\delta-\exp(-\sqrt{n})$, we have:

\begin{align}
    &\frac{\lambda_t(1-2\nu^*)\left(1-\frac{k}{n}\right)}{L\left(\lambda_{k+1}\rho_k(\SIG;\lambda) + \lambda_tL\right)} \lesssim \mathrm{SU(\Th)} \lesssim \frac{L\lambda_t(1-2\nu^*)}{\lambda_{k+1}\rho_k(\SIG;\lambda) + L^{-1}\lambda_t\left(1-\frac{k}{n}\right)}.
\end{align}
\begin{proof}
    Our bound is a generalization to Theorem 22 in \cite{muthukumar2020class} for correlated features and ridge estimator. We only require the signal and noise features to be independent. 
    
    Denote $\tilde{\X}$ to be the matrix consisting of the columns of $\X$ except for the $t$-th column,  and $\mathbf{A}_{-t}:=\tilde{\X}\tilde{\X}^T + \lambda\I$.
    As the proof in \cite{muthukumar2020class}, our proof begins with writing the SU in terms of a quadratic form of signal vector and applying Hanson-Wright inequality, Lemma \ref{hw}, by invoking the independence between the signal and noise. The result is that, with probability $1-\exp(-\sqrt{n})$,
    \begin{align}\label{su_int}
        &\mathrm{SU} \gtrsim \frac{\lambda_{t} \cdot\left(\left(1-2 \nu^{*}\right) \operatorname{tr}\left(\mathbf{A}_{-t}^{-1}\right)-2 c_{1}\left\|\mathbf{A}_{-t}^{-1}\right\| \cdot n^{3 / 4}\right)}{1+\lambda_{t}\left(\operatorname{tr}\left(\mathbf{A}_{-t}^{-1}\right)+c_{1}\left\|\mathbf{A}_{-t}^{-1}\right\| \cdot n^{3 / 4}\right)} \text { and } \\
        &\mathrm{SU} \lesssim \frac{\lambda_{t} \cdot\left(\left(1-2 \nu^{*}\right) \operatorname{tr}\left(\mathbf{A}_{-t}^{-1}\right)+2 c_{1}\left\|\mathbf{A}_{-t}^{-1}\right\| \cdot n^{3 / 4}\right)}{1+\lambda_{t}\left(\operatorname{tr}\left(\mathbf{A}_{-t}^{-1}\right)-c_{1}\left\|\mathbf{A}_{-t}^{-1}\right\|\cdot n^{3 / 4}\right)},
    \end{align}
    Now observe $\|\mathbf{A}_{-t}^{-1}\|=\mu_n(\A_{-t})^{-1}$, so by Lemma \ref{con_loo}, we have
    \begin{align}\label{spec}
        \|\A_{-t}^{-1}\|\lesssim \frac{L}{n{\lambda_{k+1}\rho_k(\SIG;\lambda)}}.    
    \end{align}
    By our assumption $\lambda_{k+1}\rho_k(\SIG;\lambda)\geq c$, we have 
    \begin{align}\label{cond}
        \lambda_1\rho_0(\SIG;\lambda)=n^{-1}\sum_{i=1}^k\lambda_i + \lambda_{k+1}\rho_k(\SIG;\lambda)\leq \lambda_{k+1}\rho_k(\SIG;\lambda)(1 + \frac{k\lambda_1}{nc}).
    \end{align}
    Also, using the same Lemma and (\ref{cond}),
    \begin{align}\label{tr_b}
        \frac{(1-k/n)(1 + \frac{k\lambda_1}{nc})^{-1}}{L{\lambda_{k+1}\rho_k(\SIG;\lambda)}}\lesssim \frac{1-k/n}{L{\lambda_{1}\rho_0(\SIG;\lambda)}}\lesssim\frac{n-k}{\mu_{k+1}(\A_{-t})}\lesssim\mathrm{tr}\left(\mathbf{A}_{-t}^{-1}\right)=&\sum\limits_{i=1}^n\frac{1}{\mu_i(\A_{-t})}\lesssim\frac{n}{\mu_n(\A_{-t})}\nonumber\\
        \lesssim & 
         \frac{L}{{\lambda_{k+1}\rho_k(\SIG;\lambda)}}.    
    \end{align}
    Finally, plugging in the bounds in (\ref{tr_b}) and (\ref{spec}) into (\ref{su_int}) completes the proof.
\end{proof}
\end{lemma}

\begin{lemma}[\textbf{Contamination of ridge estimator for dependent features} ]\label{cn-bound}
Consider the classification task under the model and assumption described in Section \ref{class-results} where  $\SIG = \text{diag}(\lambda_1, \dots, \lambda_p)$ and the true signal $\theta^* = \frac{1}{\sqrt{\lambda_t}} \e_t $ is 1-sparse in coordinate $t$. Denote the leave-signal-out covariance and data matrix as $\tilde{\SIG} = \text{diag}(\lambda_1, \dots, \lambda_{t-1}, \lambda_{t+1},\dots,\lambda_p)=\text{diag}(\tilde{\lambda}_1, \dots, \tilde{\lambda}_{p-1})$ and $\tilde{\X}=[\X_{:1},\dots,\X_{:t-1},\X_{:t+1},\dots,\X_{:p}]$, respectively.
Let $\Th = \X^\top(\X \X^\top + \lambda \I)^{-1}\y$ be a ridge regression estimator. Suppose for some $k \leq n$, with probability at least $1-\delta$, the condition numbers of $\tilde{\X}_{k+1:p}\SIG_{k+1:p}\tilde{\X}_{k+1:p}^T$ and $\lambda \I + \tilde{\X}_{k+1:p}\tilde{\X}_{k+1:p}^T$ are at most $L'$ and $L$, respectively. Then with probability $1-\delta-5n^{-1}$, we have:

\begin{equation}
    \sqrt{\frac{\tilde{\lambda}_{k+1}\rho_{k}(\tilde{\SIG}^2;0)}{L'^2\lambda_1^2(1+\rho_0(\SIG;\lambda))^2}}\lesssim\mathrm{CN(\Th)} \lesssim \sqrt{(1+\mathrm{SU(\Th)}^2)L^2\left(\frac{k}{n} + \frac{n}{R_k(\tilde{\SIG};\lambda)}\right)\log n}.
\end{equation}
\end{lemma}

\begin{proof}
We begin with the same argument as in Lemma 28 in \cite{muthukumar2020class} to write the CN as a quadratic form of signal vector.
For notation convenience, we denote the columns of $\X$ to be $\X_{:i},~ i\in\{1,2,\dots,p\}$, and define the leave-one-out quantities  $\tilde{\X}:=[\X_{:1},\dots,\X_{:t-1},\X_{:t+1},\dots,\X_{:p}]$, $\tilde{\SIG} = \mathrm{diag}(\lambda_1,\dots,\lambda_{t-1},\lambda_{t+1},\dots,\lambda_p)$, and $\tilde{\A} := \tilde{\X}\tilde{\X}^{\top} + \lambda \I$,. 
Then,
\begin{align*}
    \mathrm{CN(\Th)}^2 \leq 2\y^\top\widetilde{\mathbf{C}}\y + 2\mathrm{SU}^2\z^\top\widetilde{\mathbf{C}}\z,
\end{align*}
where $\z=\lambda_t^{-1/2}\X_{:t}$ and $\widetilde{\mathbf{C}}:=\tilde{\A}^{-1}\tilde{\X}\tilde{\SIG}\tilde{\X}\tilde{\A}^{-1}$. Because of the sparsity assumption and the independence between signal and noise features in Assumption 2, $\y$ and $\z$ are independent of $\widetilde{\mathbf{C}}$. Furthermore, $\y$ and $\z$ are both sub-Gaussian random vector with norm $1$ and independent features.

Now consider an ridge estimator with the observation vector $\varepsilon$ without looking at the $t$-feature:
$$\Th_{-t}(\varepsilon) = (\tilde{\X}\tilde{\X}^\top + \lambda\I)^{-1}\tilde{\X}^\top \varepsilon.$$
The first key observation here is that 
\begin{align}
    \y^\top\widetilde{\mathbf{C}}\y=\|\Th_{-t}(\y)\|_{\tilde{\SIG}}^2,~~\z^\top\widetilde{\mathbf{C}}\z=\|\Th_{-t}(\z)\|_{\tilde{\SIG}}^2,
\end{align}
so we can bound CN as long as we bound the $\|\Th_{-t}(\varepsilon)\|_{\tilde{\SIG}}^2$ for any sub-Gaussian vector $\varepsilon$ independent of $\tilde{\X}$ and has unit norm.
The second key observation is that $\|\Th_{-t}(\varepsilon)\|_{\tilde{\SIG}}^2$ is in fact the variance in the regression analysis.


As shown in Lemma 12 of \cite{tsigler2020benign},
\begin{align}\label{loo_imp}
     \|\Th_{-t}(\varepsilon) \|_{\tilde{\SIG}}^2 \leq \frac{\varepsilon^{\top} \tilde{\A}_{k}^{-1} \tilde{\X}_{0: k} \tilde{\SIG}_{0: k}^{-1} \tilde{\X}_{0: k}^{\top} \tilde{\A}_{k}^{-1} \varepsilon}{\mu_{n}\left(\tilde{\A}_{k}^{-1}\right)^{2} \mu_{k}\left(\tilde{\SIG}_{0: k}^{-1 / 2} \tilde{\X}_{0: k}^{\top} \tilde{\X}_{0: k} \tilde{\SIG}_{0: k}^{-1 / 2}\right)^{2}}+\varepsilon^{\top} \tilde{\A}^{-1} \tilde{\X}_{k: \infty} \tilde{\SIG}_{k: \infty} \tilde{\X}_{k: \infty}^{\top} \tilde{\A}^{-1} \varepsilon, 
\end{align}
where $\tilde{\A}_k=\tilde{\X}_{k+1:p}\tilde{\X}_{k+1:p}^{\top} + \lambda \I$. For self-containment, we sketch the proof on the variance bound.
For the first term, by Lemma \ref{hw}, for some constant $c_1$, with probability $1-2n^{-1}$,
\begin{align*}
    &\varepsilon^{\top} \tilde{\A}_{k}^{-1} \tilde{\X}_{0: k} \tilde{\SIG}_{0: k}^{-1} \tilde{\X}_{0: k}^{\top} \tilde{\A}_{k}^{-1} \varepsilon \lesssim \mathrm{tr}\left(\tilde{\A}_{k}^{-1} \tilde{\X}_{0: k} \tilde{\SIG}_{0: k}^{-1} \tilde{\X}_{0: k}^{\top} \tilde{\A}_{k}^{-1}\right)\log n\\\lesssim~&
    \mu_n(\tilde{\A}_k)^{-2}\mathrm{tr}\left(\tilde{\X}_{0: k} \tilde{\SIG}_{0: k}^{-1} \tilde{\X}_{0: k}^{\top}\right)\log n  \lesssim \mu_n(\tilde{\A}_k)^{-2} \cdot nk\log n,
\end{align*}
where the last follows from the concentration of sum of sub-Gaussian variables.
On the other hand, by Lemma \ref{con_subg}, for some constant $c_2>0$,
\begin{align*}
    \mu_{n}\left(\tilde{\A}_{k}^{-1}\right)^{2} \mu_{k}\left(\tilde{\SIG}_{0: k}^{-1 / 2} \tilde{\X}_{0: k}^{\top} \tilde{\X}_{0: k} \tilde{\SIG}_{0: k}^{-1 / 2}\right)^{2}~=~ &
    \mu_{1}\left(\tilde{\A}_{k}\right)^{-2} \mu_{k}\left(\tilde{\SIG}_{0: k}^{-1 / 2} \tilde{\X}_{0: k}^{\top} \tilde{\X}_{0: k} \tilde{\SIG}_{0: k}^{-1 / 2}\right)^{2}\\
    ~\gtrsim~ & \mu_{1}\left(\tilde{\A}_{k}\right)^{-2} \cdot (n)^2, 
\end{align*}
with probability $1-8\exp(-c_2t)$.

So the first term is, for some constant $c_3>0$, bounded by ${L^2}\frac{k}{n}$
with probability $1-16\exp(-c_3t)$.  
Similarly for the second term, again by Lemma \ref{hw}, Lemma \ref{con_rec}, and Lemma \ref{ssn}, we have for some constant $c_4>0$,
\begin{align*}
    &\varepsilon^{\top} \tilde{\A}^{-1} \tilde{\X}_{k: \infty} \tilde{\SIG}_{k: \infty} \tilde{\X}_{k: \infty}^{\top} \tilde{\A}^{-1} \varepsilon \lesssim \mathrm{tr}\left(\tilde{\A}^{-1} \tilde{\X}_{k: \infty} \tilde{\SIG}_{k: \infty} \tilde{\X}_{k: \infty}^{\top} \tilde{\A}^{-1}\right)\log n\\
    \lesssim~&~\frac{L^2}{n^2}\frac{1}{\tilde{\lambda}^2_{k+1}\rho_k^2(\tilde{\SIG};\lambda)}\cdot n\sum_{i>k}\tilde{\lambda}_i^2\log n\lesssim \frac{L^2n}{R_k(\tilde{\SIG};\lambda)}\log n,
\end{align*}
with probability $1-16\exp(-c_4t)$.

Combining all above, we deduce that 
\begin{align}
    \mathrm{CN(\Th)}^2 \lesssim (1+\mathrm{SU(\Th)}^2)L^2\left(\frac{k}{n} + \frac{n}{R_k(\tilde{\SIG};\lambda)}\right)\log n.
\end{align}
For the lower bound of $\mathrm{CN(\Th)}^2$, as shown in \cite{muthukumar2020class},
\begin{align}
    \mathrm{CN(\Th)}^2 = \y^\top \mathbf{C} \y \geq \mu_n(\mathbf{C})\|\y\|_2^2=n\mu_n(\mathbf{C}),
\end{align}
where $\mathbf{C} = (\X\X^\top + \lambda \I)^{-1}\tilde{\X}\tilde{\SIG}\tilde{\X}^T(\X\X^\top + \lambda \I)^{-1}$.
Now, by Lemma \ref{op-bound}, we have
\begin{align}\label{eq21}
    \mu_1(\X\X^\top+\lambda\I)^{-2} \lesssim \frac{1}{(\lambda_1n + \sum_{j=1}^p\lambda_j+\lambda)^2}\lesssim\frac{1}{\lambda_1^2n^2(1+\rho_0(\SIG;\lambda))^2}.
\end{align}
Also, by the boundness assumption on the condition number of $\tilde{\X}\tilde{\SIG}\tilde{\X}^{\top}$ and Lemma \ref{con_rec} we have
\begin{align}\label{eq22}
    \mu_n(\tilde{\X}\tilde{\SIG}\tilde{\X}^{\top})\gtrsim \frac{n}{L'}\tilde{\lambda}_{k+1}\rho_k(\tilde{\SIG}^2;\lambda),
\end{align}
with probability $1-\delta-n^{-1}$.
Finally, the lower bound in the theorem is established by combining eq. (\ref{eq21}) and (\ref{eq22}):
\begin{align*}
    \mu_n(\mathbf{C})\geq \mu_1(\X\X^\top+\lambda\I)^{-2}\mu_n(\tilde{\X}\tilde{\SIG}\tilde{\X}^{\top})\gtrsim \frac{\tilde{\lambda}_{k+1}\rho_{k}(\tilde{\SIG}^2;0)}{L'^2n\lambda_1^2(1+\rho_0(\SIG;\lambda))^2}.
\end{align*}
\end{proof}

\begin{lemma}[\textbf{Probability of classification error of ridge estimator for dependent features}]\label{ridge_cls_bound}
Consider the classification task under the model and assumption described in Section \ref{class-results} where  $\SIG = \text{diag}(\lambda_1, \dots, \lambda_p)$ and the true signal $\theta^* = \frac{1}{\sqrt{\lambda_t}} \e_t $ is 1-sparse in coordinate $t$. Denote the leave-one-out covariance and data matrix as $\tilde{\SIG} = \text{diag}(\lambda_1, \dots, \lambda_{t-1}, \lambda_{t+1},\dots,\lambda_p)=\text{diag}(\tilde{\lambda}_1, \dots, \tilde{\lambda}_{p-1})$ and $\tilde{\X}=[\X_{:1},\dots,\X_{:t-1},\X_{:t+1},\dots,\X_{:p}]$, respectively.
Let $\Th = \X^\top(\X \X^\top + \lambda \I)^{-1}\y$ be a ridge estimator. Suppose for some $t\leq k \leq n$, with probability at least $1-\delta$, the condition numbers of $\tilde{\X}_{k+1:p}\SIG_{k+1:p}\tilde{\X}_{k+1:p}^T$ and $\lambda \I + \tilde{\X}_{k+1:p}\tilde{\X}_{k+1:p}^T$ are at most $L'$ and $L$, respectively and $\lambda_{k+1}\rho_k(\SIG;\lambda)\geq c$ for some constant $c>0$. Then with probability $1-\delta-5n^{-1}$, we have:
\begin{align}
    \mathrm{POE}(\hat{\theta})~\lesssim~ &  \frac{{\color{red}\mathrm{CN(\Th)}}}{\mathrm{{\color{blue}SU(\Th)}}}\left(1 + \sigma_z \sqrt{\log{\frac{{\color{blue}\mathrm{SU(\Th)}}}{\mathrm{{\color{red}CN(\Th)}}}}}\right),\\
    \frac{\lambda_t(1-2\nu^*)\left(1-\frac{k}{n}\right)}{L\left(\lambda_{k+1}\rho_k(\SIG;\lambda) + \lambda_tL\right)}~\lesssim~\underbrace{{\color{blue}\mathrm{SU(\Th)}}}_{\text{Survival}}~&~\lesssim\frac{L\lambda_t(1-2\nu^*)}{\lambda_{k+1}\rho_k(\SIG;\lambda) + L^{-1}\lambda_t\left(1-\frac{k}{n}\right)},\\
    \sqrt{\frac{\tilde{\lambda}_{k+1}\rho_{k}(\tilde{\SIG}^2;0)}{L'^2\lambda_1^2(1+\rho_0(\SIG;\lambda))^2}}~\lesssim~\underbrace{{\color{red}\mathrm{CN(\Th)}}}_{\text{Contamination}} ~&~\lesssim \sqrt{(1+\mathrm{SU(\Th)}^2)L^2\left(\frac{k}{n} + \frac{n}{R_k(\tilde{\SIG};\lambda)}\right)\log n}.
\end{align}
Furthermore, if the distribution of the covariate $\x$ is Gaussian with independent features, then 
$$\mathrm{POE}(\hat{\theta})=\frac{1}{2} - \frac{1}{\pi}\tan^{-1}{\frac{\mathrm{SU(\Th)}}{ \mathrm{CN(\Th)}}}\leq {\frac{\mathrm{CN(\Th)}}{ \mathrm{SU(\Th)}}}.$$
\end{lemma}

\begin{proof}
This is a direct combination of Lemma \ref{poe_bound}, \ref{su-bound}, and \ref{cn-bound}.
\end{proof}

\begin{lemma}[\textbf{Bounds on the survival-to-contamination ratio between  $\Th_{\text{aug}}$ and $\BT$}]\label{risk-bound}
Consider an estimator $\HT$ that solves the objective (\ref{DAobj}). Denote its averaged approximation $\BT$ as in (\ref{est_2}).
Suppose $\Vert \HT - \BT \Vert_{\SIG} = O(\mathrm{SU(\BT)})$ and $\Vert \HT - \BT \Vert_{\SIG} = O(\mathrm{CN(\BT)})$. Then, the probability of classification error of $\HT$ can be bounded by:

\begin{align}\label{class-bound}
    \frac{1}{\mathrm{EM}}\frac{\mathrm{SU}(\BT)}{\mathrm{CN(\BT)}} \leq \frac{\mathrm{SU}(\HT)}{\mathrm{CN}(\HT)} &\leq {\mathrm{EM}}\frac{\mathrm{SU}(\BT)}{\mathrm{CN(\BT)}},
\end{align}
where EM$:=\exp\left(\left(1 + \frac{\|\HT-\BT\|_\SIG}{\mathrm{CN}(\BT)} \right)\left(1 + \frac{\|\HT-\BT\|_\SIG}{\mathrm{SU}(\BT)}\right)-1\right)\in [1, \infty]$ denotes the error multiplier.

\end{lemma}
\begin{proof}
Without ambiguity, we will denote $\HT$ and $\BT$ as $\Th$ and $\Tb$, respectively.
Define $f(\theta) = \log\frac{\|\V^T\theta\|}{\|\U^T\theta\|}$, where $\V = \mathbf{e}_1$ and $\U=[\mathbf{e}_2,\mathbf{e}_3,\cdots,\mathbf{e}_p]$. Then, for any estimator $\theta$, $\frac{\mathrm{SU}(\theta)}{\mathrm{CN}(\theta)} = \exp(f(\SIG^{1/2}\hat{\theta}))$  By the mean value theorem we have 
\begin{align}
    f(\SIG^{1/2}\hat{\theta}) = f(\SIG^{1/2}\bar{\theta}) + \nabla f(\SIG^{1/2}\eta)\SIG^{1/2}(\hat{\theta}-\bar{\theta}),
\end{align}
where $\eta$ is on the line segment between $\hat{\theta}$ and $\bar{\theta}$. Our goal is to show that $ \Vert \nabla f(\Sigma^{1/2}\eta)\Vert \Vert \hat{\theta}-\bar{\theta}\|_\Sigma$ is small.
To this end, firstly, observe that the norm of $f$'s gradient has a clean expression, 
\begin{align}
    \|\nabla f(\theta)\| &= \frac{1}{\Vert \U^\top \theta \Vert \Vert \V^\top \theta \Vert}\left\Vert\frac{\Vert \U^\top \theta \Vert \V \V^\top \theta}{\Vert \V^\top \theta \Vert} - \frac{\Vert \V^\top \theta \Vert \U \U^\top \theta}{\Vert \U^\top \theta \Vert} \right\Vert\\
    &=\frac{1}{\Vert \U^\top \theta \Vert \Vert \V^\top \theta \Vert}\sqrt{\frac{\Vert \U^\top \theta \Vert^2 }{\Vert \V^\top \theta \Vert^2 }\Vert \V \V^\top \theta \Vert^2 + \frac{\Vert \V^\top \theta \Vert^2}{\Vert \U^\top \theta \Vert^2} \Vert  \U \U^\top \theta\Vert^2}\\ 
    &=\frac{\|\theta\|}{\|\U^T\theta\|\|\V^T\T\|}. 
\end{align}
Hence,
\begin{align}\label{eqq}
    \|\nabla f(\SIG^{1/2}\eta)\|\|\hat{\theta}-\bar{\theta}\|_\Sigma &\leq \frac{(\|\SIG^{1/2}\bar{\theta}\| + t\|\SIG^{1/2}(\hat{\theta}-\bar{\theta})\|)\|\hat{\theta}-\bar{\theta}\|_\SIG}{(\|\U^T\SIG^{1/2}\bar{\theta}\|-t\|\U^T\SIG^{1/2}(\hat{\theta}-\bar{\theta})\|)(\|\V^T\SIG^{1/2}\bar{\theta}\|-t\|\V^T\SIG^{1/2}(\hat{\theta}-\bar{\theta})\|)}
    \nonumber\\
    &\leq \frac{(\|\bar{\theta}\|_{\SIG} + t\|\hat{\theta}-\bar{\theta}\|_\SIG)\|\hat{\theta}-\bar{\theta}\|_\SIG}{(\|\U^T\SIG^{1/2}\bar{\theta}\|-t\|\hat{\theta}-\bar{\theta}\|_\SIG)(\|\V^T\SIG^{1/2}\bar{\theta}\|-t\|\hat{\theta}-\bar{\theta}\|_\SIG)},
\end{align}
for some $t\in[0,1]$.
Secondly, we use the assumption that $\mathrm{CN}(\bar{\theta}) = \|\U^T\SIG^{1/2}\bar{\theta}\| \gg \|\hat{\theta}-\bar{\theta}\|_\SIG$ and $\mathrm{SU}(\bar{\theta}) = \|\V^T\SIG^{1/2}\bar{\theta}\| \gg \|\hat{\theta}-\bar{\theta}\|_\SIG$
for large enough $n$. Then, using the fact that $\Vert \bar{\theta} \Vert_{\SIG} \asymp SU(\bar{\theta})+ CN(\bar{\theta})$, eq. (\ref{eqq}) is bounded by
\begin{align}
    &\lesssim \left(\frac{1}{\mathrm{SU}(\bar{\theta})} + \frac{1}{\mathrm{CN}(\bar{\theta})}+\frac{\|\hat{\theta}-\bar{\theta}\|_\SIG}{\mathrm{CN}(\bar{\theta})\mathrm{SU}(\bar{\theta})}\right)\|\hat{\theta}-\bar{\theta}\|_\SIG\nonumber\\
    &=\left(1 + \frac{\|\hat{\theta}-\bar{\theta}\|_\SIG}{\mathrm{CN}(\bar{\theta})} \right)\left(1 + \frac{\|\hat{\theta}-\bar{\theta}\|_\SIG}{\mathrm{SU}(\bar{\theta})}\right)-1.
\end{align}
Hence, 
\begin{align}
    f(\SIG^{1/2}\hat{\theta})\geq f(\SIG^{1/2}\bar{\theta}) - \left(1 + \frac{\|\hat{\theta}-\bar{\theta}\|_\SIG}{\mathrm{CN}(\bar{\theta})} \right)\left(1 + \frac{\|\hat{\theta}-\bar{\theta}\|_\SIG}{\mathrm{SU}(\bar{\theta})}\right)+1,
\end{align}
and
\begin{align}
 \frac{\mathrm{SU}(\hat{\theta})}{\mathrm{CN}(\hat{\theta})}&\geq  \frac{\mathrm{SU}(\bar{\theta})}{\mathrm{CN}(\bar{\theta})}\exp\left(1-\left(1 + \frac{\|\hat{\theta}-\bar{\theta}\|_\SIG}{\mathrm{CN}(\bar{\theta})} \right)\left(1 + \frac{\|\hat{\theta}-\bar{\theta}\|_\SIG}{\mathrm{SU}(\bar{\theta})}\right)\right):=\frac{\mathrm{SU}(\bar{\theta})}{\mathrm{CN}(\bar{\theta})} \frac{1}{\mathrm{EM}}.
\end{align}

The upper bound follows by an identical argument.
\end{proof}

\begin{lemma}\label{err-class}
Let $\HT$ and $\BT$ be defined as in (\ref{est_2}) for a classification task.
Recall $${\Delta}_G:=\|\mathbb{E}_{\x}[\cg(\x)]^{-\frac{1}{2}}\cg(\X)\mathbb{E}_{\x}[\cg(\x)]^{-\frac{1}{2}}-\I\|,$$ and let $\kappa$ be the condition number of $\SIG_{\text{aug}}$. Assume ${\Delta}_G<c$ for some constant $c<1$.
Then, 
\begin{align}
    \Vert \BT - \HT \Vert_\SIG^2  \leq \kappa \Delta_G^2\left(\mathrm{SU(\BT)}^2 + \mathrm{CN(\BT)}^2\right). 
\end{align}
\end{lemma}
\begin{proof}
    For ease of notation, we denote $\bar{\D}:=\mathbb{E}_{\x}[\cg(\x)]$ and $\D = \operatorname{Cov}_{G}[\X]$. Then,
    \begin{align*}
            \Vert \BT - \HT \Vert_\SIG^2 &= \Delta_G^2\Vert \SIG^{1/2} (\X^\top \X + \D)^{-1}\bar{\D}^{1/2}\boldsymbol{n}\bar{\D}^{1/2}(\X^\top \X + \bar{\D})^{-1} \X^\top \y \Vert_2^2\\
            &= n^2\Delta_G^2\Vert \SIG^{1/2} (\X^\top \X + \D)^{-1}\SIG^{-\frac{1}{2}}\SIG^{\frac{1}{2}}(\X^\top \X + \bar{\D})^{-1} \X^\top \y \Vert_2^2\\
            &= n^2\Delta_G^2\Vert \SIG^{1/2} (\X^\top \X + \D)^{-1}\bar{\D}^{1/2}\bar{\D}^{1/2}\SIG^{-\frac{1}{2}}\SIG^{\frac{1}{2}}\BT \Vert_2^2\\
            &\leq \frac{\kappa \Delta_G^2n^2}{\mu_p(\left(\X^\top \X + \D\right)\bar{\D}^{-1})^2 } \Vert \BT \Vert_
            \SIG^2\leq {\kappa \Delta_G^2}\Vert \BT \Vert_
            \SIG^2,
    \end{align*}
    where, by the assumption ${\Delta}_G<c$, one can prove $\mu_p(\left(\X^\top \X + \D\right)\bar{\D}^{-1})^2  \gtrsim n^2$ similarly as in Lemma \ref{lem:approx}.
    Finally, recalling Definition \ref{su_cn_def}, we observe that
    \begin{align*}
        \Vert \BT \Vert_\SIG^2 = \sum\limits_{i=1}^p\lambda_i (\BT)^2_i= \mathrm{SU(\BT)}^2 + \mathrm{CN(\BT)}^2.
    \end{align*}
\end{proof}
\begin{remark}
    Comparing with Lemma \ref{lem:approx}, we see that the error between $\HT$ and $\BT$ for classification and regression are exactly the same with $SU^2$ and $CN^2$ replaced by Bias and Var.
\end{remark}

\subsection{Proof of Theorem \ref{thm2}}
    \begin{reptheorem}{thm2}[\textbf{Bounds on Probability of Classification Error}] 
    Consider the classification task under the model and assumption described in Section \ref{class-results} where the true signal $\theta^*$ is 1-sparse. Let $\HT$ be the estimator solving the aERM objective in (\ref{DAobj}). Denote ${\Delta}_G := \|\text{Cov}_G(\X)-\mathbb{E}_{\x}[\text{Cov}_g(\x)]\|$, let $t \leq n$ be the index (arranged according to the eigenvalues of ${\SIG}_{\text{aug}}$) of the non-zero coordinate of the true signal, $\tilde{\SIG}_{\text{aug}}$ be the leave-one-out modified spectrum corresponding to index $t$, $\kappa$ be the condition number of $\SIG_{\text{aug}}$, and $\tilde{\X}_{\text{aug}}$ be the leave-one-column-out data matrix corresponding to column $t$. 
    
    Suppose data augmentation is performed independently for $\x_{sig}$ and $\x_{noise}$, and there exists a $t \leq k \leq n$ such that with probability at least $1-\delta$, the condition numbers of $n \I + \tilde{\X}^{\text{aug}}_{k+1:p}(\X^{\text{aug}}_{k+1:p})^\top$ and 
    $n \I + {\X}^{\text{aug}}_{k+1:p}(\X^{\text{aug}}_{k+1:p})^\top$
    are at most $L$, and that of $\tilde{\X}_{k+1:p}\SIG_{k+1:p}\tilde{\X}_{k+1:p}^T$ is at most $L_1$. Then as long as $\Vert \BT - \HT\Vert_{\SIG} = O(SU)$ and $\Vert \BT - \HT\Vert_{\SIG} = O(CN)$, with probability $1-\delta-\exp(-\sqrt{n})-5n^{-1}$, the probability of classification error (POE) can be bounded in terms of the {\color{blue}survival} (SU) and {\color{red}contamination} (CN), as
    \begin{align}
        \mathrm{POE}(\hat{\theta})~\lesssim~ &  \frac{{\color{red}\mathrm{CN}}}{\mathrm{{\color{blue}SU}}}\left(1 + \sigma_z \sqrt{\log{\frac{{\color{blue}\mathrm{SU}}}{\mathrm{{\color{red}CN}}}}}\right),
    \end{align}
    where
    \begin{align}
        \frac{\lambda^{\text{aug}}_t(1-2\nu^*)\left(1-\frac{k}{n}\right)}{L\left(\lambda^{\text{aug}}_{k+1}\rho_k(\SIG_{\text{aug}};n) + \lambda_t^{\text{aug}}L\right)}~\lesssim~\underbrace{{\color{blue}\mathrm{SU}}}_{\text{Survival}}~&\lesssim~\frac{L\lambda^{\text{aug}}_t(1-2\nu^*)}{\lambda^{\text{aug}}_{k+1}\rho_k(\SIG_{\text{aug}};n) + L^{-1}\lambda^{\text{aug}}_t\left(1-\frac{k}{n}\right)},\\
        \sqrt{\frac{\tilde{\lambda}_{k+1}^{aug}\rho_{k}(\tilde{\SIG}_{aug}^2;0)}{L'^2(\lambda_1^{aug})^2(1+\rho_0(\SIG_{aug};\lambda))^2}}~\lesssim
        \underbrace{{\color{red}\mathrm{CN}}}_{\text{Contamination}} &\lesssim~ \sqrt{(1+\mathrm{SU}^2)L^2\left(\frac{k}{n} + \frac{n}{R_k(\tilde{\SIG}_{\text{aug}};n)}\right)\log n}
    \end{align}
    Furthermore, if $\x$ is Gaussian and the augmentation modified spectrum $\SIG_{\text{aug}}$ is diagonal then we have tighter bounds of
    \begin{align}
        \frac{1}{2} - \frac{1}{\pi}\tan^{-1}c\frac{{\color{blue}\mathrm{SU}}}{ {\color{red}\mathrm{CN}}} ~\lesssim~ \mathrm{POE}(\HT) ~\lesssim ~\frac{1}{2} - \frac{1}{\pi}\tan^{-1}\frac{1}{c}\frac{{\color{blue}\mathrm{SU}}}{ {\color{red}\mathrm{CN}}}~\lesssim~ {\frac{{\color{red}\mathrm{CN}}}{{\color{blue} \mathrm{SU}}}},
    \end{align}
where $c$ is a universal constant.
\end{reptheorem}
\begin{proof}
    We can prove the theorem by carefully walking through the proofs of Lemma \ref{poe_bound}, \ref{su-bound}, \ref{cn-bound}, and \ref{risk-bound} and noting that the error multiplier defined in Lemma \ref{risk-bound} is on the order of a constant under the assumptions made in this theorem.
\end{proof}

\subsection{Proof of Theorem \ref{bias_cls}}
\begin{reptheorem}{bias_cls}[\textbf{POE of biased estimators}]\label{pf-thm4}
    Consider the 1-sparse model $\T^*=\mathbf{e}_t$.
    and let $\HT$ be the estimator that solves the aERM in (\ref{DAobj}) with biased augmentation (i.e., $\mu(\x)\neq \x$).
    Assume that Assumption \ref{bias_assump} holds, and the assumptions of Theorem \ref{thm2} 
    are satisfied for data matrix $\mu(\X)$. If the mean augmentation $\mu(\x)$ modifies the $t$-th feature independently 
    of other features
    and
    the sign of the $t$-th feature is preserved under the mean augmentation transformation, i.e.,
    $\sgn\left(\mu(\x)_t\right)=\sgn\left(\x_t\right),$ $\forall \x$, then, the POE($\HT$) 
    is upper bounded by
    \begin{align}
        \mathrm{POE}(\HT)\lesssim \mathrm{POE}^o(\HT),
    \end{align}
    where $\mathrm{POE}^o(\HT)$ is any bound in Theorem \ref{thm2} with $\X$ and $\SIG$ replaced by $\mu(\X)$ and $\bar{\SIG}$, respectively.
\end{reptheorem}
\begin{proof} 
    First, from Lemma \ref{poe_bound}, we know that the POE can be written as a function of the SU and CN of $\Th_{\text{aug}}$. Next, recall that from the analysis in Section~\ref{impli}, the biased estimator is given by
    $$\Th_{\text{aug}}= (\mg(\X)^T\mg(\X) + n\cg(\X))^{-1}\mg(\X)^T\y.$$
    Now, observe that this estimator is almost equivalent to the one with training covariates $$\mu(\x_1),\mu(\x_2),\dots,\mu(\x_n),$$ except that the observation vector $\y$ consists of the signs of $\x_{1,t},\x_{2,t},\dots,\x_{n,t}$ instead of $\tilde{\y}$, the signs of $\mu(\x_{1,t}),\mu(\x_{2,t}),\dots,\mu(\x_{n,t})$. However, $\y$ equals $\tilde{\y}$ by our assumption that the sign of the $t$-th feature is preserved under the mean augmentation transform. So we can bound the SU and CN of $\Th_{\text{aug}}$ by just utilizing the bounds in Theorem \ref{thm2} with $\X$ and $\SIG$ replaced by $\mu(\X)$ and $\bar{\SIG}$, respectively.
\end{proof}


\subsection{Proofs of Corollaries}\label{pf_cls_cor}

\begin{corollary}[\textbf{Classification bounds for uniform random mask augmentation}]\label{rm-class}
Let $\HT$ be the estimator computed by solving the aERM objective on binary labels with mask probability $\beta$, and denote $\psi := \frac{\beta}{1-\beta}$.
Assume $p\ll n^2$. Then, with probability at least $1-\delta - \exp(-\sqrt{n}) - 5n^{-1}$
\begin{align}
    \mathrm{POE} 
        &\lesssim Q^{-1} (1 + \sqrt{\log{Q}})
    ~{\text{where}} \\
     Q &=~ (1-2\nu)\sqrt{\frac{n}{p\log n }}\left(1+\frac{n}{n\psi + p}\right)^{-1}.
\end{align}
In addition, if we assume the input data has independent Gaussian features, then we have tight generalization bounds
\begin{align}
    \mathrm{POE}\asymp~ \frac{1}{2}-\frac{1}{\pi}\tan^{-1}Q
\end{align}
with the same probability.
\end{corollary}
\begin{proof}
We first note the following key quantities:
    \begin{align*}
    &\mathbb{E}_{\x}[\cg(\x)] = \psi \text{diag}(\SIG)=\psi \SIG,~{\T}_\aug^*=\psi^{1/2} \SIG^{1/2}\T^*,~{\SIG}_\aug=\psi^{-1}\I,~{\lambda}^{\aug} = \psi^{-1},
\end{align*}
and the effective ranks of the augmentation modified spectrum are 
\begin{align}
    \rho^\aug_k &= \frac{\psi n + p -k}{n},\\
    R_k^\aug &= \frac{(\psi n + p - k)^2}{p-k}.
\end{align}
Substituting into Theorem $\ref{thm2}$ yields the formulas for the components of POE
\begin{align}
    \mathrm{SU} &\asymp (1-2\nu) \frac{n}{n\psi + n + {p}},\\
    \sqrt{ \frac{np}{(n\psi + p)^2}} \lesssim \mathrm{CN} &\lesssim \sqrt{(1+\mathrm{SU}^2)  \frac{np\log{n}}{(n\psi + p)^2}}\\
\end{align}
It remains to check when the conditions $\Vert \HT - \BT \Vert_{\SIG} = O(SU)$ and $\Vert \HT - \BT \Vert_{\SIG} = O(CN)$ are met. When $p$ grows faster than $n$, we will have $SU \asymp \frac{n}{p}$ and $CU \lesssim \sqrt{\frac{n}{p}}$. Then, using Lemma \ref{err-class}, we have

\begin{align}
    \Vert \HT - \BT \Vert_{\SIG} &\lesssim \kappa^{1/2} \Delta_G  (SU + CN)\\
    &\lesssim \sigma_\z^2\sqrt{\frac{ \log{n}}{n}} \sqrt{\frac{n}{p}}
\end{align}

So, the condition is met for $p \ll n^2$.
\end{proof}

\begin{corollary}[\textbf{Group invariant augmentation}]\label{gv}
An augmentation class $\mathcal{G}$ is said to be group-invariant if $g(\x)\overset{d}{=} \x,~ \forall g \in \mathcal{G}$. For such a class, the augmentation modified spectrum $\SIG_{\text{aug}}$ in Theorem \ref{thm2} is given by $$\mathbf{0}~\preceq ~\SIG_{\text{aug}} = \SIG-\mathbb{E}_\x[\mu_{\mathcal{G}}(\x)\mu_{\mathcal{G}}(\x)]^\top~\preceq ~\SIG.$$
Consider the case where the input covariates satisfy $\x\sim \mathcal{N}(\mathbf{0},\SIG)$. Let $\x'$ be i.i.d. with $\x$ and consider the group-invariant augmentation given by $g(\x) = \frac{1}{\sqrt{2}}\x + \frac{1}{\sqrt{2}}\x'$. Then, under the assumptions of \ref{thm2} and with probability at least $1-\delta-\exp(-\sqrt{n})-5n^{-1}$, this augmented estimator has generalization error
\begin{align}
    &\mathrm{POE}\asymp \frac{1}{2} - \frac{1}{\pi}\tan^{-1}\frac{{\mathrm{SU}}}{ {\mathrm{CN}}} , \text{ where}
    \\
    \mathrm{SU} \asymp (1-2\nu) \frac{n}{2n + {p}},~~
        &\sqrt{\frac{np}{(n+p)^2}}\lesssim \mathrm{CN} \lesssim \sqrt{(1+\mathrm{SU}^2)  \frac{np\log{n}}{(n + p)^2}}.
\end{align}
\end{corollary}
\begin{proof}
By definition and the assumption of group invariance,
\begin{align*}
    \SIG_{aug} &= \mathbb{E}_{\x}[\cg(\x)]=\mathbb{E}_{\x}\mathbb{E}_g[g(\x)g(\x)^\top-\mathbb{E}_g[g(\x)]\mathbb{E}_g[g(\x)]^\top]\\&=\mathbb{E}_{g}\mathbb{E}_\x[g(\x)g(\x)^\top-\mu_{\mathcal{G}}(\x)\mu_{\mathcal{G}}(\x)^\top]=\mathbb{E}_{g}\mathbb{E}_\x[\x\x^\top-\mu_{\mathcal{G}}(\x)\mu_{\mathcal{G}}(\x)^\top]\\
    &= \SIG-\mathbb{E}_\x[\mu_{\mathcal{G}}(\x)\mu_{\mathcal{G}}(\x)]^\top. 
\end{align*}
The change of the expectation order follows from the Tonelli's theorem, while the last inequality is by the group invariance assumption. Now applying Theorem \ref{thm2} completes the proof for $\SIG_{\text{aug}}$.

Now, for the example in this corollary, first note that this is a group-invariant augmentation as $g(\x)$ is Gaussian with the same mean and covariance as $\x$. Direct calculations show that $\mu_{\mathcal{G}}(\x) =\frac{1}{\sqrt{2}}\x$ and $\SIG_{\text{aug}}=\frac{1}{2}\SIG$. Furthermore, $\operatorname{Cov}_{G}(\X) = \frac{1}{2}\SIG$ is a constant matrix so $\Delta_G=0$ and the approximation error is zero. Now applying Theorem \ref{thm2} and \ref{bias_cls} yields the result.
\end{proof}

\begin{repcorollary}{rot}[\textbf{Generalization of random rotation}]
    The estimator induced by the random-rotation augmentation (with angle parameter $\alpha$) can be expressed as $$\hat{\T}_{\text{rot}} = \left(\X^\top\X + \frac{4(1-\cos \alpha)}{p}\left(
    \operatorname{Tr}\left(\X^\top\X\right)\I - \X^\top\X\right)\right)^{-1}\X^\top \y.$$
    An application of Theorem~\ref{gen_bound} yields
    \begin{align*}
        \mathrm{Bias}(\hat{\T}_{\text{rot}}) \asymp \mathrm{Bias}(\hat{\T}_{\text{lse}}),
    \end{align*}
    for sufficiently large $p$ (overparameterized regime), as well as the variance bound
    \begin{align*}
        \mathrm{Var}(\hat{\T}_{\text{rot}})\lesssim \mathrm{Var}(\hat{\T}_{\text{ridge,}\lambda}),
    \end{align*}
    Above, $\hat{\T}_{\text{lse}}$ and $\hat{\T}_{\text{ridge,}\lambda}$ denote the least squared estimator and ridge estimator with ridge intensity $\lambda = {np^{-1}(1-\cos\alpha)\sum_{j}\lambda_j}$.
    The approximation error can also be shown to decay as
    \begin{align*}
        \mathrm{Approx. Error}(\hat{\T}_{\text{rot}})\lesssim \max\left(\frac{1}{n}, \frac{1}{\mathrm{tr} (\SIG)}\right). 
    \end{align*}
\end{repcorollary}

\begin{proof}
    The proof is based on the application of Theorem \ref{gen_bound}, where
    \begin{align*}
        \E_\x\cg(\x) = \frac{4(1-\cos \alpha)}{p}(\mathrm{Tr}(\SIG)\I-\SIG),~\SIG_{\text{aug}}=\frac{p}{4(1-\cos\alpha)}\SIG(\mathrm{Tr}(\SIG)\I-\SIG)^{-1}.
    \end{align*}
    Hence, $\lambda^{\text{aug}}_i  \asymp\frac{p}{4(1-\cos\alpha)} \frac{\lambda_i}{\sum_j\lambda_j}$, and 
    \begin{align*}
        &\mathrm{Bias}(\Th_{\text{rot}}) \\
        &\lesssim~
        \|\T^*_{k+1:\infty}\|_{\SIG_{k+1:\infty}}^2 + \sum_{i=1}^k\frac{(\T^*_i\sum_{j\neq i}\lambda_j)^2}{\lambda_i}\left(1 + \frac{p}{4(1-\cos\alpha )n}\frac{\sum_{j>k}\lambda_j}{\sum_{j}\lambda_j}\right)^2 \left(\frac{4(1-\cos\alpha)}{p}\right)^2\\
        &\lesssim~
        \|\T^*_{k+1:\infty}\|_{\SIG_{k+1:\infty}}^2 + \sum_{i=1}^k\frac{(\T^*_i\sum_{j\neq i}\lambda_j)^2}{\lambda_i}\left( \frac{\sum_{j>k}\lambda_j}{n\sum_{j}\lambda_j}\right)^2~~\text{, for sufficiently large $p$} \\
        &\asymp~
        \|\T^*_{k+1:\infty}\|_{\SIG_{k+1:\infty}}^2 + \|\T^*_{1:k}\|_{\SIG_{1:k}^{-1}}^2\lambda_{k+1}^2\rho_{k}(\SIG;0)^2=\mathrm{Bias}(\Th_{\text{lse}}),
    \end{align*}
    where the last equality is by Corollary \ref{gen_agn} with $\lambda=0$. The variance part can be proved similarly. The approximation error bound is proved in Appendix \ref{non_diag}.
\end{proof}

\begin{corollary}
[\textbf{Classification bounds for Gaussian noise injection}]\label{gn-class}
Consider the independent, additive Gaussian noise augmentation: $g(\x) = \x + \mathbf{n}$, where $\mathbf{n}\sim\mathcal{N}(0,\sigma^2)$. Let $\tilde{\SIG}$ be the leave-one-out spectrum corresponding to index $t$. Then, with probability at least $1-\exp(\sqrt{n})-5n^{-1}$,

\begin{align}
    \mathrm{SU} &\asymp (1-2\nu^*) \frac{\lambda_t}{\lambda_{k+1}\rho_k(\SIG; n\sigma^2) + \lambda_t},\\
    \mathrm{CN} &\lesssim \sqrt{(1+\mathrm{SU}^2)\left( \frac{k}{n} + \frac{n}{R_k(\tilde{\SIG};  n\sigma^2)}\right)\log{n}},\\
\end{align}
and $\mathrm{EM} = 1$.

\begin{proof}
    As in the regression analysis, we note that in this case, the key quantities are given by
    \begin{align*}
        &\mathbb{E}_{\x}[\cg(\x)] = \sigma^2 \I,~{\T}_\aug^*=\sigma\T^*,~{\SIG}_\aug=\sigma^{-2} \SIG,~{\lambda}^{\aug} = \sigma^{-2}\lambda,
    \end{align*}
    and the effective ranks are given by
    \begin{align*}
 \rho_k(\SIG_\aug;n) =\rho_k(\SIG;n\sigma^2),\\
         R_k(\SIG_\aug;n) =R_k(\SIG;n\sigma^2).
    \end{align*}
    
    Finally, $\log(EM)$ is zero because $\Delta_G=0$. Substituting the above quantities into the Theorem \ref{thm2} yields the result.
\end{proof}
\end{corollary}

\begin{corollary}[\textbf{Classification bounds for non-uniform random mask}]\label{hrm-class}
Consider the case where the dropout parameter $\psi_j = \frac{\beta_j}{1-\beta_j}$ is applied to the $j$-th feature, and assume the conditions of Theorem \ref{thm2} are met. For simplicity, we consider the bi-level case where $\psi_j = \psi$ for $j \neq t$. Then, with probability at least $1-\delta - \exp(\sqrt{n}) - 5n^{-1}$,
\begin{align}
    \mathrm{SU} &\asymp  \frac{1}{\psi_1 + \frac{p\psi_t}{n\psi} + 1}\\
    \mathrm{CN} &\lesssim \sqrt{(1+\mathrm{SU}^2) \frac{np\log{n}}{(n\psi + p)^2}}
\end{align}
\end{corollary}
\begin{proof}

    Let $\Psi$ denote the diagonal matrix with $\Psi_{i,i}=\psi$ if $i \neq t$ and $\Psi_{t,t}=\psi_t$.

We can then compute the following key quantities:
    \begin{align*}
    &\mathbb{E}_{\x}[\cg(\x)] = {\Psi} \SIG,~{\T}_\aug^*= {\Psi}^{1/2} \SIG^{1/2}\T^*,~{\SIG}_\aug=\Psi^{-1},
\end{align*}
and the effective ranks of the augmentation modified spectrum are 
\begin{align}
    \rho^\aug_k &= \frac{\psi n + p -k}{n},\\
    R_k^\aug &= \frac{(\psi n + p - k)^2}{p-k}.
\end{align}
The approximation error bound proceeds as in the uniform random mask case. Substituting the above quantities into Theorem \ref{thm2} completes the proof.
\end{proof}

\section{Comparisons between Regression and Classification}
\subsection{Proof of Proposition \ref{gen_compare}}
\label{app:prop2}
\begin{proposition}[\textbf{DA is easier to tune in classification than regression}]\label{gen_compare}
Consider the 1-sparse model $\T^*=\sqrt{\frac{1}{\lambda_t}}\mathbf{e}_t$ for Gaussian covariate with independent components and an independent feature augmentation. Suppose that  the approximation error is not dominant in the bounds of Theorem \ref{gen_bound} (simple sufficient conditions can be found in Lemma \ref{cond_tight1} in Appendix \ref{appen:lemma}) and the assumptions in the two theorems hold, then,
\begin{align*}
    \mathrm{POE}(\HT) &\lesssim \sqrt{(\lambda_{k+1}^{\text{aug}}\rho_k(\SIG_{\text{aug}};n))^2 \cdot \left(\frac{n}{R_k(\SIG_{\text{aug}};n)} +\frac{k}{n}\right)\log n},\\
    \mathrm{MSE}(\HT) &\gtrsim (\lambda_{k+1}^{\text{aug}}\rho_k(\SIG_{\text{aug}};n))^2 +\left( \frac{n}{R_k(\SIG_{\text{aug}};n)} +\frac{k}{n}\right).
\end{align*}
As a consequence, the regression risk serves as a surrogate for the classification risk up to a $log$-factor:
\begin{align}
    \mathrm{POE}(\HT) \lesssim \mathrm{MSE}(\HT)\sqrt{\log n}.
\end{align}

As concrete examples of the regression risk being a surrogate of classification risk, consider Gaussian noise injection augmentation with noise standard deviation $\sigma$ and random mask with dropout probability $\beta$ to train the 1-sparse model in the decaying data spectrum $\SIG_{ii}=\gamma^i,~\forall i\in\{1,2,\dots,p\}$, where $\gamma$ is some constant satisfying $0<\gamma<1$. Let $\hat{\T}_{\text{gn}}$ and $\hat{\T}_{\text{rm}}$ be the corresponding estimators, then
\begin{align}
    \lim_{n\rightarrow \infty}\lim_{\sigma\rightarrow \infty}\mathrm{POE}(\hat{\T}_{\text{gn}})= 0~~\text{while}~~\lim_{n\rightarrow \infty}\lim_{\sigma\rightarrow \infty}\mathrm{MSE}(\hat{\T}_{\text{gn}})= 1.
\end{align}
Also, when $p\log n\ll n$,
\begin{align}
    \lim_{n\rightarrow \infty}\lim_{\beta\rightarrow 1}\mathrm{POE}(\hat{\T}_{\text{rm}})= 0~~\text{while}~~\lim_{n\rightarrow \infty}\lim_{\beta\rightarrow 1}\mathrm{MSE}(\hat{\T}_{\text{rm}})= 1.
\end{align}
Furthermore, the augmentation of Gaussian injection has gone through significant distributional shift where
\begin{align}
    \frac{W^2_2(g(\x),\x)}{p} \overset{n,\sigma}{\longrightarrow} \infty,
\end{align}
in which $W_2$ denotes the 2-Wasserstein distance between the pre- and post-augmented distribution of the data by the Gaussian noise injection.
\end{proposition}
\begin{proof}
    We begin with proving the first statement.
    By our assumption that the approximation error and error multiplier are not dominant terms in generalization errors, we can only consider bias/variance and survival/contamination.
    By Proposition \ref{ind_cor}, the regression testing risk is bounded by
    $$\mathrm{MSE}(\HT) \lesssim 
    (\lambda_{k+1}^{\text{aug}}\rho_k(\SIG_{\text{aug}};n))^2 +\left( \frac{n}{R_k(\SIG_{\text{aug}};n)} +\frac{k}{n}\right).$$
    However, by the independence of the original data feature components and their augmentations and the boundness assumption on $\rho_k$, Lemma $2$, Lemma $3$ and Theorem $5$ in \cite{tsigler2020benign} shows that there is a matching lower bound such that 
    \begin{align}\label{mse_low}
        \mathrm{MSE}(\HT) \gtrsim (\lambda_{k+1}^{\text{aug}}\rho_k(\SIG_{\text{aug}};n))^2 +\left( \frac{n}{R_k(\SIG_{\text{aug}};n)} +\frac{k}{n}\right),
    \end{align}
    for some $k$.
    On the other hand, by Theorem \ref{thm2}, we know that
    \begin{align}\label{poe_upp}
        \mathrm{POE}(\HT) \lesssim \sqrt{(\lambda_{k+1}^{\text{aug}}\rho_k(\SIG_{\text{aug}};n))^2 \cdot \left(\frac{n}{R_k(\SIG_{\text{aug}};n)} +\frac{k}{n}\right)\log n}, 
    \end{align}
    for any k. Now combining E. q. (\ref{mse_low}) and (\ref{poe_upp}) along with the inequality $x + y \geq 2\sqrt{xy}$ for any $x,y\geq 0$ proves the first statement.
    
    To prove the second statement about $\hat{\theta}_{\text{gn}}$, note that $\hat{\theta}_{\text{gn}} = (\X^\top \X+ \sigma^2n\I)^{-1} \X^\top \y \rightarrow 0$ almost surely as $\sigma\rightarrow \infty$, so $$\mathrm{MSE}(\hat{\theta}_{\text{gn}})=\|\T^*-\hat{\theta}_{\text{gn}}\|_{\SIG}\overset{a.s.}{\longrightarrow} \|\T^*\|_\SIG=1.$$
    On the other hand, by Theorem \ref{thm2}, choose $k = 0$, then
    \begin{align}
       &\mathrm{SU}(\hat{\theta}_{\text{gn}})\gtrsim \frac{n\frac{\lambda_t}{\sigma^2}}{{n+\frac{\sum{\lambda_j}}{\sigma^2}} + \frac{n\lambda_t}{\sigma^2}},~~
       \mathrm{CN}(\hat{\theta}_{\text{gn}})\lesssim \frac{1}{\sigma^2}\sqrt{\frac{(\sum\lambda_j^2)n\log n}{(n +\sum\frac{\lambda_j}{\sigma^2})^2}},\nonumber
    \end{align}
    So,
    \begin{align}
        &\mathrm{POE}(\hat{\theta}_{\text{gn}}) \leq \frac{\mathrm{CN}(\hat{\theta}_{\text{gn}})}{\mathrm{SU}(\hat{\theta}_{\text{gn}})} \asymp \frac{1}{\lambda_t}\sqrt{\frac{(\sum \lambda_j^2)\log n}{n}}\times\frac{n +\sum\frac{\lambda_j}{\sigma^2}+\frac{n\lambda_t}{\sigma^2}}{n +\sum\frac{\lambda_j}{\sigma^2}},\\
        &\lim_{n\rightarrow \infty}\lim_{\sigma\rightarrow \infty}\mathrm{POE}(\hat{\theta}_{\text{gn}})=
        \lim_{n\rightarrow \infty}\frac{1}{\lambda_t}\sqrt{\frac{\log n}{n(1-\gamma^2)}}=0.
    \end{align}
    We can prove the statement for $\hat{\theta}_{\text{rm}}$ similarly. When $\beta\rightarrow 1$, $\hat{\theta}_{\text{rm}}=(\X^\top \X+ \frac{\beta}{1-\beta}\operatorname{diag}[\X^\top\X])^{-1} \X^\top \y\rightarrow 0$ almost surely. So MSE approaches $1$ almost surely. But by Corollary \ref{rm-class}, we have 
    \begin{align}
        \lim_{n\rightarrow \infty}\lim_{\beta\rightarrow 1}\mathrm{POE}(\hat{\theta}_{\text{rm}}) =\lim_{n\rightarrow \infty} \sqrt{\frac{p\log n}{n}}= 0.
    \end{align}
    Finally, by the closed-form formula of Wasserstein distance between Gaussian distributions,
    \begin{align}
        W_2(g(\x),\x) = \|(\SIG + \sigma^2\I)^{\frac{1}{2}}-\SIG^{\frac{1}{2}}\|_F^2=\Omega(p\sigma^2).
    \end{align}
\end{proof}

\subsection{Classification/regression separation for non-uniform random mask}
\label{app:prop3}

\begin{proposition}\label{prop:non-uniformmask}[\textbf{Non-uniform random mask is easier to tune in classification}]
Consider the 1-sparse model $\T^*=\sqrt{\frac{1}{\lambda_t}}\mathbf{e}_t$. Suppose the approximation error is not dominant in the bounds of Theorem \ref{gen_bound} (simple sufficient conditions can be found in Lemma \ref{cond_tight1} in Appendix \ref{appen:lemma}) and the assumptions in the two theorems hold. Suppose we apply the non-uniform random mask augmentation and recall the definitions of $\psi$ and $\psi_t$ as in Corollary \ref{hrm-class}. Then, if 
    $\sqrt{\frac{p}{n}} \ll \frac{\psi}{\psi_t} \ll \frac{p}{n},$ we have 

\begin{align}
    \mathrm{POE}(\hat{\T}_{\text{rm}})\overset{n}{\longrightarrow} 0~~\text{while}~~\mathrm{MSE}(\hat{\T}_{\text{rm}})\overset{n}{\longrightarrow} 1.
\end{align}

\end{proposition}
\begin{proof}
From Corollary \ref{het_mask}, we have that the bias scales as 
\begin{align*}
    \text{Bias} &\lesssim \frac{(\psi_t n + \frac{\psi_t p}{\psi})^2}{n^2 + (\psi_t n + \frac{\psi_t p}{\psi})^2} \asymp \frac{(\psi_t n + \frac{\psi_t p}{\psi})^2}{(\psi_t n + \frac{\psi_t p}{\psi})^2} = 1,
\end{align*}
where the second asymptotic equality uses the assumption that $\frac{\psi_t p}{\psi} \gg n$. Hence the MSE approaches a constant (here we use the fact that the MSE bound is tight when the approximation error is non-dominant, as per \cite{tsigler2020benign}). Next we use the bounds in Corollary \ref{hrm-class} to find that
\begin{align*}
    \text{SU} \asymp \frac{1}{\psi_t + \frac{p \psi_t}{n \psi} + 1} \asymp \frac{1}{\psi_t + \frac{p \psi_t}{n \psi}},~~~~ \mathrm{CN} &\asymp \sqrt{\frac{np}{(n\psi + p)^2}}.
\end{align*}
So, if $p \gg n\psi$, we have
\begin{align*}
    \frac{SU}{CN} &\asymp \frac{1/\psi_t}{(1/\psi)\sqrt{p/n}} = \frac{\psi/\psi_t}{\sqrt{p/n}} \to \infty,
\end{align*}
and if $p \ll n\psi$, we have
\begin{align*}
    \frac{SU}{CN} &\asymp \frac{\frac{n\psi}{p\psi_t}}{\sqrt{\frac{n}{p}}} = \frac{\psi/\psi_t}{\sqrt{p/n}} \to \infty.
\end{align*}
Since we assume we are operating in a regime where the approximation error and error multiplier do not dominate, we can conclude that $\text{POE} \to 0$.
\end{proof}

\section{Derivations of Common Augmented Estimators}\label{common_est}

\begin{proposition}[\textbf{Common augmentation estimators}]
Below are closed-form expression of estimators that solves (\ref{DAobj}) with common data augmentation.
\begin{itemize}
    \item {\em Gaussian noise injection with zero-mean noise of covariance $\mathbf{W}$:}
    $$\hat{\T}_{\text{aug}} = (\X^\top\X + n\mathbf{W})^{-1}\X^\top \y$$
    \item {\em Unbiased random mask with mask probability $\beta$:}
    $$\hat{\T}_{\text{aug}} = \left(\X^\top\X + \frac{\beta}{1-\beta}~\rm{diag}(\X^T\X)\right)^{-1}\X^\top \y$$
    \item {\em Unbiased random cutout with number of cutout features $k$:}
    $$\left(\X^\top\X + \frac{p}{p-k}\M\odot{\X^\top\X}\right)^{-1}\X^\top \y,$$ where $\M_{i,j}=\frac{k}{p}-\frac{|j-i|\mathbf{1}_{|j-i|<k-1}+k\mathbf{1}_{|j-i|\geq k-1}}{p-k}$.
    \item {\em Salt and Pepper $(\beta, \mu, \sigma^2)$:}
    $$\hat{\T}_{\text{aug}} = \left(\X^\top\X + \frac{\beta}{1-\beta}~\rm{diag}\left(\X^\top\X\right)  + \frac{\beta\sigma^2n}{(1-\beta)^2}\I\right)^{-1}\X^\top \y$$
    \item {\em Unbiased random rotation with angle $\alpha$:}
    $$\hat{\T}_{\text{aug}} = \left(\X^\top\X + \frac{4(1-\cos \alpha)}{p^2}\left(
    \operatorname{Tr}\left(\X\X^\top\right)\I - \X\X^\top\right)\right)^{-1}\X^\top \y$$
\end{itemize}
\end{proposition}
\begin{proof}
    To prove all the unbiased augmented estimator formulas, it suffices to derive $\cg(\X)$. Then, $$\HT=(\X^\top\X+n\cg(\X))^\top\X\y.$$
    \paragraph{\textbf{Gaussian noise injection}} $g(\x) = \x + \mathbf{n}$, where $\mathbf{n}\sim \mathcal{N}(0,\mathbf{W})$. Therefore,
    \begin{align*}
        \cg(\X)=n^{-1}\sum_i\cg(\x_i)=n^{-1}\sum_i \mathbb{E}_{\mathbf{n}_i}[(\x_i+\mathbf{n}_i)(\x+\mathbf{n}_i)^\top-\x_i\x_i^\top]= \mathbf{W}.
    \end{align*}
    \paragraph{\textbf{Unbiased random mask}}
    $g(\x) = (1-\beta)^{-1}\mathbf{b}\odot\x$, where $\mathbf{b}$ has i.i.d. Bernoulli random variable with dropout probability $\beta$ in its component. The factor $(1-\beta)^{-1}$ is to rescale the estimator to be unbiased. Hence,
    \begin{align*}
        \cg(\X)&=(1-\beta)^{-2}n^{-1}\sum_i \mathbb{E}_{\bb_i}[\bb_i\bb_i^\top\odot\x_i\x_i^\top-\x_i\x_i^
        \top]\\
        &= n^{-1}\sum_i \left(\frac{\beta}{1-\beta}\I + \mathbf{1}\mathbf{1}^\top-\mathbf{1}\mathbf{1}^\top\right)\odot\x_i\x_i^\top=n^{-1}\frac{\beta}{1-\beta}\operatorname{diag}\left(\X^\top\X\right)
    \end{align*}
    \paragraph{Random cutout} Define $h(\x)$ to be the random cutout of $k$ consecutive features, then the unbiased cutout can be written as $g(\x)=\frac{p}{p-k}h(
    \x)$ as $\mathbb{E}_hh(\x)=\frac{p-k}{k}\x$. Now,
    \begin{align*}
        \operatorname{Cov}_h(\x)=\mathbb{E}_h[h(\x)h(\x)^\top]-\left(\frac{p-k}{p}\right)^2\x\x^\top.
    \end{align*}
    Note that $\mathbb{E}_hh(\x)h(\x)^\top =\HH\odot \x\x^\top$, where 
    \begin{align*}
        \HH_{ij}&=\operatorname{P}[\x_i\text{ is not cutout and }\x_j\text{ is not cutout}] \\&= \operatorname{P}[\text{a random }k\text{ consecutive features does not cover }i\text{ nor }j]\\
        &=\frac{p-k-|j-i|\mathbf{1}_{|j-i|<k-1}-k\mathbf{1}_{|j-i|\geq k-1}}{p}. 
    \end{align*}
    Hence,
    \begin{align*}
        \operatorname{Cov}_h(\x) &=\left(\HH-\left(\frac{p-k}{p}\right)^2\mathbf{1}\mathbf{1}^\top\right)\odot \x\x^\top,\\
        \left(\HH-\left(\frac{p-k}{p}\right)^2\mathbf{1}\mathbf{1}^\top\right)_{ij}&=\frac{p-k}{p}\frac{k}{p} - \frac{|j-i|\mathbf{1}_{|j-i|<k-1}+k\mathbf{1}_{|j-i|\geq k-1}}{p},
    \end{align*}
    and 
    \begin{align*}
        \cg(\x) &= \left(\frac{p}{p-k}\right)^2\operatorname{Cov}_h(\x)\\
        &= \frac{p}{p-k}\left(\frac{k}{p}-\frac{|j-i|\mathbf{1}_{|j-i|<k-1}+k\mathbf{1}_{|j-i|\geq k-1}}{p-k}\right)\odot \x\x^\top\\
        &=\frac{p}{p-k}\M\odot \x\x^\top.
    \end{align*}
    \paragraph{Salt and pepper}
    This estimator can be derived similarly by combining the derivations of the random mask and the injection of Gaussian noise by writing the augmentation as
    $$g(\x) = (1-\beta)^{-1}\left(\mathbf{b}\odot \x + (\mathbf{1}-\mathbf{b})\odot \mathbf{n}\right),$$
    where $\mathbf{b}$ has i.i.d. components of Bernoulli random variables with parameter $\beta$ and $\mathbf{n}\sim \mathcal{N}(0,\I)$.
    \paragraph{\textbf{Random rotation}}
    Given a training example $\x$, we will consider rotating $\x$ by a angle $\alpha$ in $\frac{p}{2}$ random plane spanned by two randomly generated orthonormal vectors $\mathbf{u}$ and $\mathbf{v}$. For rotation in each one of the plan, the data transformation can be written by
    \begin{align}
        h(\x) = (\I + \sin\alpha (\vv\uu^\top-\uu\vv^\top) + (\cos\alpha - 1)(\uu\uu^\top + \vv\vv^\top) )\x.
    \end{align}
    The bias of $h$ is $\Delta =\mathbb{E}_{\uu,\vv}[h(\x)]-\x$.
    We consider the unbiased transform $g$ by subtracting the bias from $h$ where $g(\x):=h(\x) - \Delta $. Since we consider random $\uu$ and $\vv$, they are distributed uniformly on the sphere of $\mathbf{R}^p$ but orthogonal to each other. The exact joint distribution of $\uu$ and $\vv$ is intractable, but fortunately when $p$ is large, we know from high dimensional statistics that they are approximately independent vector of $\mathcal{N}(0, {\frac{1}{p}}\I)$. We will thus use this approximation to facilitate our derivation.
    
    Firstly, $$\mathbb{E}_{\uu,\vv}[h(\x)] = \x + \mathbb{E}_\uu2(\cos\alpha - 1)\uu\uu^\top \x = \x + \frac{2}{p}\x,$$
    so the bias $\Delta = \frac{2}{p}\x$ which is small in high dimensional space. Secondly, subtracting $\Delta$ from $h$, we proceed to calculate the $\cg(\X) = \frac{\sum_{i=1}^n\mathrm{Cov_{g_i}}(\x_i)}{n}$ according to Definition \ref{cov_def}. After simplification, we have
    \begin{align*}
        \mathrm{Cov_{g_i}}(\x_i)&=\mathbb{E}_g{g(\x_i)g(\x_i)^\top}  \\
        & = \mathbb{E}_{\uu,\vv}\Bigg[\sin^2\alpha\left(\vv\uu^\top - \uu\vv^\top\right)\x\x^\top(\vv\uu^\top -\uu\vv^\top)\\
        &+ (\cos \alpha - 1)^2\left(\uu\uu^\top + \vv\vv^\top -\frac{2}{p}\I\right)\x\x^\top \left(\uu\uu^\top + \vv\vv^\top -\frac{2}{p}\I\right)\Bigg]\\
        & = 2\sin^2\alpha\left(\mathbb{E}_{\uu,\vv}\left[\langle \vv, \x\rangle\langle \uu, \x\rangle\uu\vv^\top-\langle \uu,\x\rangle^2\vv\vv^\top\right]\right) \\
        & + 2(\cos \alpha - 1)^2\left(
        \mathbb{E}_{\uu,\vv}\left[\langle \uu,\x\rangle^2\vv\vv^\top + \langle \vv, \x\rangle\langle \uu, \x\rangle\uu\vv^\top
        -\frac{4}{p}\langle \uu,\x\rangle\uu\x^\top\right] + \frac{2}{p^2}\x\x^\top\right).
    \end{align*}
    By direct calculations, we also have,
    \begin{align*}
        &\mathbb{E}_{\uu,\vv}\left[\langle \uu,\x\rangle^2\vv\vv^\top\right]= \mathbb{E}_{\uu,\vv}\left[\langle \uu,\x\rangle^2\right]\mathbb{E}_{\uu,\vv}[\vv\vv^\top] = \frac{\|\x\|_2^2}{p^2},\\
        &\mathbb{E}_{\uu,\vv}\left[\langle \vv, \x\rangle\langle \uu, \x\rangle\uu\vv^\top\right]=\frac{\x\x^\top}{p^2}.
    \end{align*}
    Now, plugging in the terms into $\cg(\X)$ and
    multiplying the result by $\frac{p}{2}$ as there are $\frac{p}{2}$ rotations completes the proof.
\end{proof}

\paragraph{PatchShuffle regularization (\cite{kang2017patchshuffle})} 

This augmentation is an example of a patch-based method where the original feature vector $\x$ is partitioned into sub-vectors $\Tilde{\x}$ each with dimension $b$. The augmentation function to each sub-vector is given by
$g(\Tilde{\x}) = (1-r)\Tilde{\x} + r\Pi(\Tilde{\x})$ where $r \sim \text{Bernoulli}(1 - \beta)$ (chosen independently for each patch), and $\Pi$ is a uniform random permutation to the features in $\Tilde{\x}$. 

Given a patch vector feature vector $\x\in \mathbb{R}^b$, we will show that $\text{{Cov}}_{\mathcal{G}_k}(\x)$  and $\E_x\text{{Cov}}_{\mathcal{G}_k}(\x)$ are given as 
\begin{align*}
    \text{{Cov}}_{\mathcal{G}_k}(\x)~=~&\beta(1-\beta)\x\x^\top + \left[\frac{1-\beta}{b(b-1)}\left(\left(\mathbf{1}^\top\x\right)^2-\mathbf{1}^\top\x^{\odot 2}\right)-\left(\frac{1-\beta}{b}\right)^2\left(\mathbf{1}^\top\x\right)^2\right]\mathbf{1}\mathbf{1}^\top\\
    &+\left(\frac{1-\beta}{1-b}\mathbf{1}^\top\x^{\odot 2}-\frac{1-\beta}{b(1-b)}\left(\mathbf{1}^\top\x\right)^2\right)\I_b-\frac{\beta(1-\beta)}{b}\mathbf{1}^\top\x\left(\x\mathbf{1}^\top + \mathbf{1}\x^\top\right),\\
\E_{\x}[\text{{Cov}}_{\mathcal{G}_k}(\x)]~=~&\beta(1-\beta)\SIG - \frac{1-\beta}{b}\left(\mathbf{1}^\top\mathbf{\lambda}\right)\I_b -\left(\frac{1-\beta}{b}\right)^2\mathbf{1}^\top\mathbf{\lambda}\mathbf{1}\mathbf{1}^\top \\
&-\frac{\beta(1-\beta)}{b}\left(\lambda\mathbf{1}^\top +\mathbf{1}\lambda^\top\right),
\end{align*}
where $\x^{\odot 2}$ denotes the entrywise (Hadamard) product of $\x$ with itself.

First, note that by definition, $\text{{Cov}}_{\mathcal{G}_k}(\x) = \E[g(\x)g(\x)^\top]-\E[g(\x)]\E[g(\x)]^\top$. The first term is
    \begin{align*}
        & \E[g(\x)g(\x)^\top]\\
        = ~&~\E[(1-r)^2]\x\x^\top + \E[(1-r)r]\E[\x\Pi(\x)^\top] + \E[(1-r)r]\E[\Pi(\x)\x^\top] + \E[r^2]\E[\Pi(\x)\Pi(\x)^\top]\\
        = ~&~\beta\x\x^\top + (1-\beta) \E[\Pi(\x)\Pi(\x)^\top]\\
        = ~&~\beta\x\x^\top+ (1-\beta)\left[\left(\frac{(\mathbf{1}^\top \x)^2-\mathbf{1}^\top\x^{\odot 2}}{b(b-1)}\right)\mathbf{1}\mathbf{1}^\top + \left(\frac{1}{b-1}\mathbf{1}^\top\x^{\odot 2}-\frac{(\mathbf{1}^\top\x)^2}{b(b-1)}\right)\I\right].
    \end{align*}
    The last equation follows since the diagonal elements of $\E[\Pi(\x)\Pi(\x)^\top]$ are all equal to $\frac{1}{b}\sum_i \x_i^2$, while the $(i, j)$ off-diagonal element is $\frac{\sum_i\sum_{j\neq i}\x_i\x_j}{b(b-1)}=\frac{(\sum_i\x_i)^2-\sum_i\x_i^2}{b(b-1)}$. 
    For the second term, $\E[g(\x)] = \beta \x + (1-\beta)\frac{\mathbf{1}^\top\x}{b}\mathbf{1}$. Combining the above expressions gives the result of $\text{{Cov}}_{\mathcal{G}_k}(\x)$.
    Finally, notice that $\E[\left(\mathbf{1}^\top\x\right)^2]=\E[\mathbf{1}^\top\x^{\odot 2}]=\mathbf{1}^\top\boldsymbol{\lambda}$ and $\E[\mathbf{1}^\top\x\x\mathbf{1}^\top] = \boldsymbol{\lambda}\mathbf{1}^\top$, where $\boldsymbol{\lambda}$ denotes the spectrum of $\boldsymbol{\Sigma}$. This completes the calculation of $\E_{\x}[\text{{Cov}}_{\mathcal{G}_k}(\x)]$.

\section{Approximation Error for Dependent Feature Augmentation}\label{non_diag}
In this section, we demonstrate how to bound the approximation error for the augmentation of dependent features which satisfy neither the independent-augmentation nor regionally-correlated augmentation assumptions. We use rotation in a random plane (Section~\ref{sec:rot} and cutout~\cite{devries2017improved} as our two examples.
We will build on Proposition~\ref{prop:dependent_bound}, which we restate here for convenience.
\begin{repproposition}{prop:dependent_bound}
Consider the decomposition $\cg(\X)=\D + \Q$, where $\D$ is a diagonal matrix representing the \textit{independent} feature augmentation part. 
Then, we have
\begin{align*}
    {\Delta}_{G}\lesssim \frac{\|\D-\E\D\| + \|\Q-\E\Q\|}{\mu_p(\E_\x\cg(\x))}. 
\end{align*}
\end{repproposition}

\subsection{Approximation error of random rotations}
In this section, we will walk through the steps to bound the approximation error for the random rotation estimator.
Specifically, we will prove that $$\cg(\X)=\frac{4(1-\cos \alpha)}{np}\left(
    \operatorname{Tr}\left(\X^\top\X\right)\I - \X^\top\X\right),~~{\Delta}_{G}\lesssim  \frac{\lambda_1n +\sum_j\lambda_j}{n\sum_{j>1}\lambda_j}.$$

We follow the bound in E.q. (\ref{dependent_bound}) from the main text: $$
    {\Delta}_{G}\lesssim \frac{\|\D-\E\D\| + \|\Q-\E\Q\|}{\mu_p(\E_\x\cg(\x))},
$$ where we decompose $\cg(\X)$ into diagonal and off-diagonal parts as $\cg({\X})=\D + \Q$,  $\D=a\left(\operatorname{Tr}\left(\X^\top\X\right)\I +\mathrm{Diag}(\X^\top\X)\right)$, $\Q=a\left(\X^\top\X -\mathrm{Diag}(\X^\top\X)\right)$, and $a = \frac{4(1-\cos \alpha)}{np} =\Theta(\frac{1}{np})$. 
Using similar arguments in the proof of Proposition \ref{ind_cor} for the independent feature augmentation, the error of the diagonal part can be expressed as a sum of $n$ independent subexponential variables divided by $\Theta(np)$. Then, by the concentration bound in Lemma \ref{ber_exp} we have,
$$\|\D-\E\D\| \lesssim \frac{1}{p}\sqrt{\frac{\log n}{n}},$$
with probability $1-n^{-1}$.

On the other hand, by invoking Lemma \ref{op-bound}, we also have,
$$\|\Q-\E\Q\|=\|\Q\| \lesssim \frac{\lambda_1 n + \sum_j\lambda_j}{np},$$
with probability at least $1-\frac{1}{n}$, using the fact that $\mathbb{E}\Q=0$.
Finally, $$\mu_p(\E_\x\cg(\x)) = 4(1-\cos \alpha)\frac{\mu_p(\mathrm{Tr}(\SIG)\I-\SIG)}{p} = 4(1-\cos \alpha)\frac{\sum_{j>1}\lambda_j}{p},$$
so $${\Delta}_{G}\lesssim \frac{\lambda_1n +\sum_j\lambda_j}{n\sum_{j>1}\lambda_j},$$
with probability $1-2n^{-1}$. Note that $\Delta_G$ is $o(1)$ for $\sum_{j>1}\lambda_j \gg \lambda_1$.

\subsection{Approximation error of random cutout}\label{sec:approx_cut}
In this section, we turn our attention to the bound of the approximation error for random cutout, where $k$ consecutive features are cut out randomly by the augmentation. As the features are dropout dependently, the random cutout belongs to the class of dependent feature augmentation.
For simplicity, we consider the unbiased random cutout, where the augmented estimator is rescaled by the factor $\frac{p}{p-k}$ (so $\mg(\x) = \x$). The calculations in Section \ref{common_est} show that 
\begin{align}\label{cutexp}\mathbb{E}_{\x}[\cg(\x)]=\frac{k}{p-k}\operatorname{diag}(\SIG),~~\cg(\X)=\frac{p}{p-k}\M\odot\frac{\X^\top\X}{n},\end{align}
where $\M$ is a circulant matrix in which $\M_{i,j}=\frac{k}{p}-\frac{|j-i|\mathbf{1}_{|j-i|<k-1}+k\mathbf{1}_{|j-i|\geq k-1}}{p-k}$ and $\odot$ denotes the element-wise matrix  product (Hadamard product). Because $\SIG$ is diagonal we have,
\begin{align*}
    \Delta_G=\frac{p}{k}\|\M\odot{(n^{-1}\Z^\top\Z}-\I_p)\|,
\end{align*}
where $\Z$ is a $n$ by $p$ matrix with i.i.d. subgaussian rows that has identity covariance $\I$. Then
\begin{align}\label{delg}
    \Delta_G
    =\frac{p}{k}\cdot\left(\left\|\widetilde{\M}\odot \D + \frac{k^2}{p(p-k)} n^{-1}\Z^\top\Z\right\|\right)
    \leq \frac{p}{k}\cdot\left(\underbrace{\left\|\widetilde{\M}\odot \D\right\|}_{L_1} + \underbrace{\left\|\frac{k^2}{p(p-k)} n^{-1}\Z^\top\Z\right\|}_{L_2}\right),   
\end{align}
where $\D$ is an almost diagonal circular matrix with $\D_{ij} = \sum_{l=1}^n\frac{\Z_{li}\Z_{lj}}{n}-\delta_{ij}$ if $|i-j|\leq k$ and $0$ otherwise, while $\widetilde{\M}_{i,j}=\M_{i,j} + \frac{k^2}{p(p-k)}$.  Our decomposition strategy here is consistent with our idea in the previous subsection, where we partition the matrix of interest into strong diagonal components and weak off-diagonal components. However, in the random cutout case, approximately $O(k)$ near the diagonal components have a strong covariance with intensity of the order $O(\frac{k}{p})$ while the rest of the order $O(\frac{k^2}{p^2})$; hence, we gather all elements with strong covariance into the ``diagonal'' part. Now we will bound $L_2$ and $L_1$ in a sequence. 

Like in the previous section, $L_2$ can be bounded by invoking the lemma \ref{op-bound} which gives
$$L_2\lesssim \frac{k^2}{p(p-k)}\frac{n+p}{n},$$ with probability $1-\frac{c}{n}$ for some constant $c>0$.
For the bounds of $L_1$, we first bound the elements of $\D$. For $i\neq j$, since $\Z_{ki}^2$ is sub-exponential we have
\begin{align*}
    \D_{i,j}\leq \sum_{k=1}^n\frac{\Z_{ki}\Z_{kj}}{n}\leq n^{-1}\sqrt{\sum_{k=1}^n\Z_{ki}^2}\sqrt{\sum_{k=1}^n\Z_{kj}^2}\leq \varepsilon,
\end{align*}
with probability $\exp(-nC\varepsilon^2)$ for some constant $C$ by Lemma \ref{ber_exp}, where we have used Cauchy-Schwartz inequality and $\varepsilon$ will be determined below. The case where $i=j$ is similar. As there are $O(pk)$ nonzero terms in $\D$, we choose $\varepsilon=\sqrt{\frac{5\log pk}{n}}$. Then, by union bounds over $pk$ terms, we obtain
$$\D_{i,j}\leq \sqrt{\frac{5\log pk}{n}}, ~~\forall i,j$$
with probability at least $1-\frac{1}{p^3k^3}$. Next, denote $\A:=\widetilde{\M}\odot \D$. Note that $|\A_{ij}|\lesssim \frac{k}{p}\varepsilon$ for all $|i-j|\leq k$ and $0$ otherwise. We will bound the operator norm of $\A$.
Consider any $\vv$ with $\|\vv\|_2=1$, 
\begin{align*}
    \|\A\vv\|_2&= \sqrt{\sum_{i=1}^k(\sum_{j=1}^k\A_{ij}\vv_j)^2}
    =\sqrt{\sum_{i=1}^k(\sum_{j\in i-k:i+k}\A_{i,j}\vv_j)^2}\\
    & \leq \sqrt{\sum_{i=1}^k(\sum_{{j\in i-k:i+k}}\A_{i,j}^2)(\sum_{j\in i-k:i+k}\vv_j^2)}\leq \frac{k}{p}\sqrt{2k\varepsilon^2\sum_{i=1}^k\sum_{j\in i-k:i+k}\vv_j^2}\\
    &=O\left(\frac{k^2}{p}\varepsilon\right),
\end{align*}
where we have used the sparsity property that $\A_{ij}=0$ if $|j-i|>k$.
Therefore, $L_1=\|\A\|\lesssim O(\frac{k^2}{p}\varepsilon)=O\left(\frac{k^2}{p}\sqrt{\frac{5\log pk}{n}}\right)$.
Now combining the bounds on $L_1$ and $L_2$ and \eqref{delg} we arrive at the result:
$$\Delta_G\lesssim k\sqrt{\frac{\log pk}{n}},$$
with probability at least $1-\frac{c}{n}-\frac{1}{p^3k^3}$.

\begin{remark}This approximation bound, together with Corollary \ref{cond_tight1}, show that the approximation error is dominated by the bias-variance (survival-contamination) if \textbf{1.} \textit{over-parameterized regime} ($p\gg n$): $p$ is upper bounded by some polynomial of $n$ and $k\ll \sqrt{\frac{n}{\log p}}$, or \textbf{2.} \textit{under-parameterized regime} ($p\ll n$): $n$ is upper bounded by some polynomial of $p$ and $k \ll \frac{p}{\sqrt{n}}$. 
\end{remark}

\section{Additional experiments}

\subsection{The implicit bias of minimal or ``weak'' DA}\label{weakDA}

It is well-known that Gaussian noise injection approximates the LSE when the variance of the added noise approaches zero. Surprisingly, however, this does not imply that all kinds of DA approach the LSE in the limit of decreasing augmentation intensity. 
Suppose that the augmentation $g$ is characterized by some hyperparameter $\xi$ that reflects the intensity of the  augmentation (for e.g., mask probability $\beta$ in the case of randomized mask, or Gaussian noise standard deviation $\sigma$ in the case of Gaussian noise injection), and that $\operatorname{Cov}_G(\X)/\xi {\longrightarrow }\operatorname{Cov}_{\infty}$ as $\xi \rightarrow 0$ for some positive semidefinite matrix $\operatorname{Cov}_{\infty}$ that does not depend on $\xi$. Then, the limiting estimator when the augmentation intensity $\xi$ approaches zero is given by
\begin{align}\label{limit_aug}
    \hat{\theta}_{aug}\overset{\xi\rightarrow 0}{\longrightarrow}\operatorname{Cov}_{\infty}^{-1}\X^\top\left(\X \operatorname{Cov}_{\infty}^{-1}\X^\top\right)^{\dagger}\y.
\end{align}
It can be easily checked that this estimator is the minimum-Mahalanobis-norm interpolant of the training data where the positive semi-definite matrix used for the Mahalanobis norm is given by $\operatorname{Cov}_{\infty}$. Formally, the estimator solves the optimization problem
\begin{align}
    \min_{\T}\|\T\|_{\operatorname{Cov}_{\infty}}~~{\text{s.t. } \X\T = \y}
\end{align}
 Thus, the choice of augmentation impacts the specific interpolator  that we obtain in the limit of minimally applied DA. For example, the above formula can be applied to random mask with $$\operatorname{Cov}_{\infty} = n^{-1}\text{diag}(\X^T\X)\approx \SIG.$$
 \begin{wrapfigure}{r}{0.5\textwidth}
    \centering 
    \includegraphics[width=0.5\textwidth]{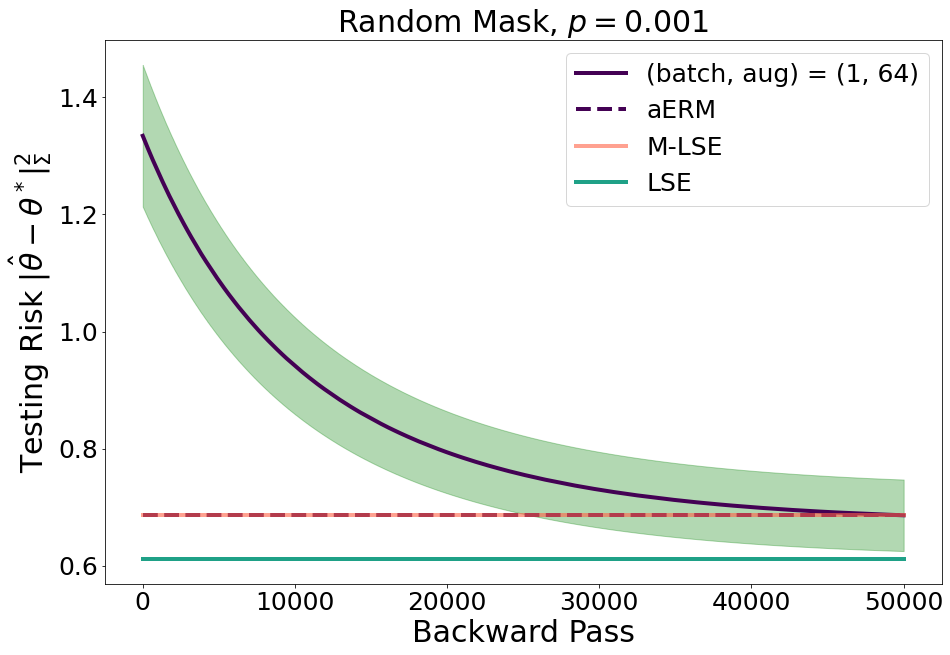}
    \caption{\footnotesize{\textbf{aSGD convergence to aERM for small random mask.} We simulate the convergence of aSGD for random mask with dropout probability $0.01$. We compare its converging estimator with the aERM limit~\eqref{limit_aug}). }\label{conv_rm}}
\end{wrapfigure}
 Fig. \ref{conv_rm} demonstrates that the MSE of the random mask does not converge to that of the LSE. Instead, it converges to the light green curve which we abbreviate as M-LSE (for the \emph{masked least squared estimator}).
 To test whether these limits appear only in an aERM solution,
 we plot the convergence path of aSGD with the random mask augmentation with masking probability $\beta = 0.01$. We set the ambient dimension $p$, noise standard deviation $\sigma_{\epsilon}$, number of training examples $n$, and learning rate $\eta$ to be $128$, $0.5$, $64$ and $10^{-5}$ respectively. We choose a decaying covariate spectrum of the form $\SIG_{ii} \propto \gamma^{i}$, where $\gamma$ is chosen such that $\mu_p(\SIG) = 0.2 \mu_1(\SIG)$. It is clear from the plot that both aSGD and aERM converges to the M-LSE solution of~\eqref{limit_aug}). 
 The curves and the shaded area denote the averaged result and the $90\%$ confidence interval for $50$ experiments.
 A caveat to this result is that the convergence rate turns out to be relatively slow and highly sensitive to the learning rate. 
 A theoretical investigation of this behavior (and the optimization convergence of aSGD to aERM more generally) is beyond the scope of this work and would be interesting to explore in the future.

\newpage


\end{document}